\newif\ifarxiv 
\arxivtrue

\newif\ifreview 
\reviewtrue

\newif\iffinal 
\finaltrue

\ifarxiv
\documentclass{article}
\else
\ifreview
\documentclass[3p,review,authoryear]{elsarticle}
\else
\documentclass[3p,final,authoryear]{elsarticle}
\fi
\journal{Artificial Intelligence}
\fi

\usepackage[utf8]{inputenc}
\usepackage[vlined,linesnumbered,ruled,titlenotnumbered]
{algorithm2e}

\ifarxiv
\usepackage[a4paper,left=2cm,right=2cm,top=2cm,bottom=2cm]{geometry}
\usepackage[onehalfspacing]{setspace}
\usepackage[round,compress]{natbib}
\fi

\usepackage{amssymb,amsfonts,amsmath,amsthm}
\usepackage{dsfont}
\usepackage{bbm}
\usepackage{makecell}

\usepackage{mathtools}
\usepackage{booktabs,multirow}
\usepackage{xspace}
\usepackage[dvipsnames]{xcolor}
\usepackage{graphicx}
\usepackage{enumitem}
\usepackage{subcaption}
\usepackage{natbib}

\ifarxiv
\usepackage[pdfusetitle,hidelinks]{hyperref} 
\else
\usepackage{hyperref}
\fi

\ifarxiv
\usepackage{authblk}
\fi

\usepackage{tikz}
\usetikzlibrary{shapes,arrows}
\usetikzlibrary{decorations}
\usetikzlibrary{plotmarks}
\usetikzlibrary{calc}
\usetikzlibrary{mindmap}
\usetikzlibrary{shadows}
\usetikzlibrary{backgrounds}
\usetikzlibrary{patterns}
\usetikzlibrary{shapes.symbols}

\allowdisplaybreaks

\newtheorem{theorem}             {Theorem}
\newtheorem{lemma}      [theorem]{Lemma}

\newtheorem{definition} [theorem]{Definition}

\newcommand{\Natural}{\mathbb{N}}

\newcommand{\mLOTZ}{\ensuremath{d}\text{-}\textsc{LOTZ}\xspace}
\newcommand{\mOMM}{\ensuremath{d}\text{-}\textsc{OMM}\xspace}
\newcommand{\mCOCZ}{\ensuremath{d}\text{-}\textsc{COCZ}\xspace}
\newcommand{\mOJZJ}{\ensuremath{d}\text{-}\textsc{OJZJ}\xspace}
\newcommand{\mRRMO}{\ensuremath{d}\text{-}$\textsc{RR}_{\mathrm{MO}}$\xspace}

\usepackage{xcolor}

\newcommand{\LZ}{\textsc{LZ}\xspace}                      
\newcommand{\TZ}{\textsc{TZ}\xspace}
\newcommand{\LO}{\textsc{LO}\xspace}                      
\newcommand{\TOs}{\textsc{TO}\xspace}                     
\newcommand{\ones}[1]{|#1|_1}                             

\newcommand{\COCZ}{\textsc{COCZ}\xspace}                           

\newcommand{\OMM}{\textsc{OMM}\xspace}                             

\newcommand{\LOTZ}{\textsc{LOTZ}\xspace}                           

\newcommand{\OJZJ}{\textsc{OJZJ}\xspace}                           
\newcommand{\OJZJfull}{\textsc{OneJumpZeroJump}\xspace}            

\newcommand{\expect}[1]{\mathrm{E}\left[#1\right]}        

\newcommand{\RRRMO}{\ensuremath{\textsc{RR}_{\mathrm{MO}}}\xspace}  

\newcommand{\rp}{\mathrm{rp}}
\newcommand{\refer}{\mathcal{R}_p}

\newcommand{\nsga}{NSGA\nobreakdash-II\xspace}
\newcommand{\nsgaIII}{NSGA\nobreakdash-III\xspace}
\newcommand{\vecone}{\vec{1}}

\newcommand{\zeros}[1]{|#1|_0}  

\iffinal
\newcommand{\todo}[1]{}
\newcommand{\newedit}[1]{\textcolor{black}{#1}} 

\else
\newcommand{\todo}[1]{\textcolor{red}{[TODO: #1]}}
\newcommand{\newedit}[1]{\textcolor{blue}{#1}}

\fi

\ifreview
\usepackage[switch,mathlines]{lineno}

\usepackage{etoolbox} 
\newcommand*\linenomathpatch[1]{
  \cspreto{#1}{\linenomath}%
  \cspreto{#1*}{\linenomath}%
  \csappto{end#1}{\endlinenomath}%
  \csappto{end#1*}{\endlinenomath}%
}
\linenomathpatch{equation}
\linenomathpatch{gather}
\linenomathpatch{multline}
\linenomathpatch{align}
\linenomathpatch{alignat}
\linenomathpatch{flalign}
\fi

\ifarxiv
\title{Many-Objective Problems Where Crossover is Provably Essential}
\fi 

\ifarxiv
\author{Andre Opris\\
  University of Passau\\
  Passau, Germany
  }
\fi 

\begin{document}

\ifarxiv
\maketitle
\else
\begin{frontmatter}

\title{Runtime Analyses of NSGA-III on Many-Objective Problems: Provable Exponential Speedup via Stochastic Population Update}

\author[1]{Andre~Opris}\ead{andre.opris@uni-passau.de}
\cortext[cor1]{Corresponding author}
\affiliation[1]{organization={Chair of Algorithms for Intelligent Systems, University of Passau},
	city={Passau},
	country={Germany}}
\fi

\begin{abstract}
NSGA-III is a prominent algorithm in evolutionary many-objective optimization. It is particularly well suited for optimizing problems with more than three objectives, distinguishing it from the classical NSGA-II. However, theoretical understanding of when and why NSGA-III performs well is still at an early stage. In this paper, we contribute to closing this gap by conducting rigorous runtime analyses on the classical many-objective benchmark problems $d$-\textsc{LeadingOnesTrailingZeros} ($d$-LOTZ), $d$-\textsc{CountingOnesCountingZeros} ($d$-COCZ), $d$-\textsc{OneMinMax} ($d$-OMM), and $d$-\textsc{OneJumpZeroJump} ($d$-OJZJ) for arbitrary numbers of objectives $d$. In particular, we improve upon previous results from~\citep{OprisNSGAIII, DoerrNearTight} when the population size is asymptotically larger than the size of the Pareto front. Notably, in the bi-objective case, the derived upper runtime bounds are asymptotically tighter than those known for NSGA-II. For the problems $2$-OMM and $2$-OJZJ, NSGA-III even outperforms NSGA-II in terms of expected runtime for suitable population sizes $\mu$. Further, we show that a stochastic population update mechanism proposed in~\citep{UpBian} provably yields an exponential speedup in the expected runtime on many-objective multimodal problems such as \mOJZJ, as well as on the function \mRRMO introduced in~\citep{OPRIS2026}, for certain parameter regimes. To complement our analysis, we also establish tight runtime bounds for NSGA-III on $2$-\textsc{OJZJ} and $4$-\textsc{OJZJ}. In particular, the result for $4$-OJZJ provides, to the best of our knowledge, the first lower bound for NSGA-III on a classical benchmark problem with more than two objectives. Deriving these bounds requires a substantially deeper analysis of the population dynamics of NSGA-III than has been achieved in previous work.
\end{abstract}

\ifarxiv
\else

\end{frontmatter}

\ifreview
\linenumbers
\fi

\fi

\section{Introduction}
Evolutionary multi-objective algorithms (EMOAs) use principles of nature to optimize functions with multiple conflicting objectives by finding a Pareto optimal set~\citep{TutEmMOEA2018}. These have been commonly applied to a wide range of optimization problems in practice~\citep{Swarm2023,Keller2017,Gunantara2018} which also include neural networks~\citep{LIU2020}, bioinformatics~\citep{Handl2008}, engineering~\citep{MULTIENGINEERING} or various fields of artificial intelligence~\citep{LUUKKONEN2023102537,ZHANG2021447,MONTEIROEURO2023}. As it is typical for real world problems~\citep{Stewart2008}, such problems often involve four or more objectives. Hence, the study of EMOAs on many-objective problems has quickly gained huge importance in many research fields. However, when the number of objectives increases, the Pareto front grows exponentially in the number of objectives, and hence, problems typically become more challenging. Additionally, it becomes more difficult to identify dependencies between single objectives. \nsga~\citep{Deb2002}, the most prominent EMOA ($\sim$60,000 citations), is able to optimize bi-objective problems efficiently (see~\citep{ZhengLuiDoerrAAAI22} for a first rigorous analysis and for example \citep{DANG2024104098,Qu2022PPSN} for further rigorous ones or~\citep{Deb2002} for empirical results), but loses performance if the number of objectives grows (see~\citep{Zheng2023Inefficiency,InefficiencyLOTZ} for rigorously proven results, where there may be a huge difference already between two and three objectives, or~\citep{NSGAIIINEFF2003,NSGAIIINEFF2007} for empirical studies). This behavior comes from the \emph{crowding distance}, the tie breaker in \nsga, which is based on sorting search points in each objective. In case of two objectives, a sorting of non-dominated individuals with respect to the first objective induces a sorting with respect to the second one (in reverse order) and hence the crowding distance is a good measure for the closeness of an individual to its neighbors in the objective space. However, for problems with three or more objectives, such a correlation between different objectives does not necessarily exist and hence individuals may have crowding distance zero even if they are not close to others in the objective space.  
To overcome this problem, the \nsgaIII algorithm was designed in~\citep{DebJain2014} which uses a set of predefined reference points instead of the crowding distance. It can optimize a broad class of different benchmarks with at least four objectives efficiently \citep{NSGAIIIpractice2022,GU2022117738}, demonstrating its success in practice ($\sim$6,000 citations). However, theoretical understanding of its success is still in its early development and lags far behind its practical impact. To the best of our knowledge, only a few papers provide rigorous runtime analyses of this algorithm (see~\citep{WiethegerD23,DoerrNearTight,OprisNSGAIII,OPRIS2026,OprisMultimodal,Opris26PopDyn} for main contributions and breakthroughts), where the works~\citep{OprisMultimodal,Opris26PopDyn} demonstrate how the selection procedure of \nsgaIII can lead to improved runtime guarantees for large population sizes, asymptotically exceeding the size of the Pareto front. However, their analyses are restricted to settings in which all search points in the population are already Pareto optimal. To the best of our knowledge, there are no results demonstrating how this mechanism can yield improved runtime guarantees for large population sizes on problems where the Pareto front must first be identified. This is despite empirical evidence showing that \nsgaIII distributes solutions very effectively across the search space~\citep{CHAUDHARI20221509}. In this paper, we investigate such problems and demonstrate that this selection mechanism makes \nsgaIII quite robust with respect to the choice of the population size. This constitutes a notable advantage, as it reduces the amount of problem-specific knowledge required to select an appropriate population size.

This robustness can be beneficial from both a theoretical and a practical perspective. First, population size is a critical parameter that can significantly influence convergence and runtime, as shown in theoretical works on simple evolutionary algorithms, also from a multi-objective point of view~\citep{6747993,DoerrN20}. Such robustness indicates that the algorithm’s dynamics are stable and reliable across diverse problem instances, allowing theoretical guarantees to hold under broader parameter regimes. Finally, population-size insensitivity enables flexible scaling with available computational resources, improving both efficiency and practical applicability in many-objective optimization where the Pareto front is typically unknown~\citep{4492360}. This property can be highly valuable in practical scenarios, for example in optimizing expensive engineering simulations where each evaluation is costly~\citep{ALBERTIN2023113433}, tuning hyperparameters in machine learning models with evaluation times that grow with dataset size~\citep{10.1145/3610536}, and solving industrial scheduling or design problems where the size of the Pareto front cannot be predicted in advance~\citep{ZHU2023119707}. In such cases, the ability to select a population size based on available resources, rather than delicate performance tuning, greatly simplifies algorithm deployment and reduces computational waste. 

Notably, this robustness also applies to functions with local optima. For this class of problems in particular, it is crucial to understand the mechanisms of \nsgaIII and its variants in order to gain insights into when and why the algorithm performs well, as many-objective problems with local optima frequently arise in real-world applications~\citep{NSGAIIILocalOptima}. An example of such an improved variant is \nsgaIII with stochastic population update, which has been successfully applied in many-objective settings~\citep{OprisMultimodal}. This mechanism has also been successfully incorporated into other EMOAs, such as SMS-EMOA~\citep{Zheng_Doerr_2024} and NSGA-II~\citep{UpBian}. Instead of forming the next generation solely by selecting the first ranked solutions resulting from non dominated sorting, that is, in a greedy, deterministic manner, some individuals are also selected uniformly at random independently of their fitness. This update scheme may help preserve diversity, prevents premature convergence to local optima, and enhance the exploration of Pareto optimal solutions. As demonstrated on two example problems below, it may significantly facilitate crossing fitness valleys. We anticipate that this or similar approaches could be applied to multi-objective traveling salesperson problems or engineering design tasks. However, their successful application will require a thorough understanding of the population dynamics of \nsgaIII.

\textbf{Our contribution}: This article significantly extends the two conference papers~\citep{OprisNSGAIII,OprisMultimodal} and provides runtime analyses of \nsgaIII for any number $d$ of objectives on the classical benchmark problems $d$-\textsc{OneMinMax} (\mOMM for short), $d$-\textsc{LeadingOnesTrailingZeros} (\mLOTZ for short), and $d$-\textsc{CountingOnesCountingZeros} (\mCOCZ for short), and finally on $d$-\textsc{OneJumpZeroJump} (\mOJZJ for short) with and without stochastic population update. All these results improve the corresponding results from~\citep{DoerrNearTight} for certain parameter regimes of the population size asymptotically larger than the size of the Pareto front. In Table~\ref{tab:overview-runtime-results} $k$ denotes the gap size, a problem specific parameter for \mOJZJ. We also show how stochastic population update may accelerate the opimization process of \nsgaIII on a $d$-objective version of the RealRoyalRoad function (\mRRMO for short) for any number of objectives $d$, which is a further example for a many-objective multimodal problem next to \mOJZJ. This function builds on its single-objective version introduced in~\citep{Jansen2005c} and its bi-objective extension in~\citep{DANG2024104098}. It was originally designed to demonstrate that crossover can drastically accelerate the optimization process. Finally, we provide also tight runtime bounds of \nsgaIII on \mOJZJ for $d \in \{2,4\}$ to complement our analysis. The two conference papers~\citet{OprisNSGAIII,OprisMultimodal} restrict themselves to a constant number of objectives. The first analyzes the classic benchmark problems \mLOTZ, \mOMM and \mCOCZ, but does not take any results about the population dynamics of \nsgaIII into account. Although the second does so on \mOJZJ, these general results were only restricted to the case where the population consists only of Pareto optimal solutions, and stochastic population update is turned off. Further, it does not analyze the benefits of stochastic population update for \mRRMO, and only provides a lower runtime bound for \nsgaIII on \mOJZJ for the case $d=2$. All our main contributions are detailed as follows.

Firstly, to be able to derive these improved runtime bounds, we have to investigate the population dynamics of \nsgaIII more carefully than previous papers did. To this end, we use the \emph{cover number} $c_t(v)$ of an objective vector $v$, which is the number of individuals in the current population $P_t$ with fitness vector $v$. First, we prove a very general result applicable to any fitness function and to both scenarios with and without stochastic population update, showing that the cover number of a fitness vector cannot decrease as long as $c_t(v) \le \lfloor \mu / ((1+a)|S_d|) \rfloor$, where $S_d$ is a maximum set of mutually incomparable solutions, and $v$ is non-dominated. Further, $a \in \{0,1\}$ is a parameter that takes the value $1$ if stochastic population update is enabled, and zero otherwise. This means that \nsgaIII maintains a distributed set of solutions, as long as no individual is generated that dominates another one. As we will see, this behavior may support hill-climbing. Consequently, if the entire population is on the Pareto front and remains there in the next generation, any vector covered by $\ell \leq \lfloor{\mu/((1-a)|S_d|)}\rfloor$ solutions will be covered by $\ell$ solutions in the next generation. For a large class of problems, we can additionally bound the cover number from above if stochastic population update is disabled. In particular, it can be shown that with probability $1-e^{-\Omega(n)}$, after $O(n)$ generations in which all individuals in $P_t$ are Pareto optimal, the cover number of every Pareto optimal vector is at most $2^{d/2+1} \mu/|S_d|$ if $\mu \in 2^{O(n)}$. This also holds in all future iterations as long as all individuals are Pareto optimal. Hence, all solutions are spread out quite evenly on a large fraction of the Pareto front. For example, if $\mu = \Theta(|S_d|)$, the cover number of every $v$ is then at most $2^{d/2}$ which is just a constant if $d= O(1)$. Such insights are very useful for rigorously proving lower bounds, as we will do for \nsgaIII on \mOJZJ for $d \in \{2,4\}$ when stochastic population update is turned off. 

\begin{table}[t]
\begin{center}
\caption{Overview of all runtime bounds for the \nsgaIII with standard bit mutation on all discussed benchmark problems in terms of generations, where $d$ denotes the number of objectives, $n$ the problem size, and $\mu$ the population size. For \nsgaIII, we always use $\varepsilon_{\text{nad}} \geq f_{\max}$, and a set $\refer = \left\{\left(\frac{a_1}{p}, \ldots ,\frac{a_d}{p} \right) 
    \text{ } \Big| \text{ }  
    (a_1,\dots,a_d) \in \mathbb{N}_0^d, 
    \sum_{i=1}^d a_i = p
\right\}$ of reference points for $p \in \mathbb{N}$ as defined below with $p \geq 2d^{3/2}f_{\max}$, where $f_{\max}$ denotes the maximum value in each objective of the underlying $d$-objective function $f$. Compare with the preliminaries section below for explanations regarding $\varepsilon_{\text{nad}}$. We also require that $|S_d| \leq \mu$, where $|S_d|$ is a maximum number of mutually incomparable solutions, when stochastic population update is turned off, and otherwise, $2|S_d| \leq \mu$. In the column "$|S_d|$" the expression $"\leq b"$ means that the maximum number of pairwise incomparable solutions is bounded by the number $b$. In the column "update", "no" means that stochastic population update is turned off, and "yes" means that it is turned on. The column "known bound" always refers to scenarios where stochastic population update is turned off and presents known results for $d=O(1)$ for simplicity. We also require $\mu \leq \min\{n^{k-d/2}, 2^{O(n)}\}$ in order to obtain tight runtime bounds for \nsgaIII on \mOJZJ of $\Theta(n^{k+d}/\mu)$ for $d \in \{2,4\}$. For $d=4$ the case $k \in \{2,3\}$ is excluded.}
\label{tab:overview-runtime-results}
\begin{tabular}{llccll}
   problem & \;\;\;\;$|S_d|$ & $d$ & update & our established bound & known bound\\
   \toprule
   \mLOTZ & $\leq (\frac{2n}{d}+1)^{d-1}$ & any & yes/no & $O(n \ln(\min\{\frac{\mu}{|S_d|},n\})+ \frac{|S_d|n^2}{\mu})$ & $O(n^2)$ \\ 
   \mOMM & \;\;$(\frac{2n}{d}+1)^{d/2}$ & $\omega(\frac{\sqrt{n}}{\ln(n)})$ & yes/no & $O(n\ln(\min\{\frac{\mu}{|S_d|},n\}) + \frac{d|S_d| n \ln(n)}{\mu})$ & \\
   \mOMM & \;\;$(\frac{2n}{d}+1)^{d/2}$ & $O(\frac{\sqrt{n}}{\ln(n)})$ & yes/no & $O(n + \frac{d|S_d| n \ln(n)}{\mu})$ & $O(n\ln(n))$ \\
   \mCOCZ & \;\;$(\frac{n}{d}+1)^{d/2}$ & $\omega(\frac{\sqrt{n}}{\ln(n)})$ & yes/no & $O(n\ln(\min\{\frac{\mu}{|S_d|},n\}) + \frac{d|S_d| n \ln(n)}{\mu})$ & \\
   \mCOCZ & \;\;$(\frac{n}{d}+1)^{d/2}$ & $O(\frac{\sqrt{n}}{\ln(n)})$ & yes/no & $O(n + \frac{d|S_d| n \ln(n)}{\mu})$ & $O(n\ln(n))$ \\
   \noalign{\vskip 4pt}
   \mOJZJ & $\leq (\frac{2n}{d}+1)^{d-1}$ & $d \leq n/6$ & no & \makecell[l]{$O(n\ln(\min\{\frac{\mu}{|S_d|},d\}) + \frac{d^2 |S_d| n \ln(n)}{\mu}$ \\ $+ \frac{d n^k |S_d|}{\mu} + d\ln(\min\{\frac{\mu}{|S_d|},n^k\}))$} & $O(n^k)$ \\
   \noalign{\vskip 4pt}
   \mOJZJ & $\leq (\frac{2n}{d}+1)^{d-1}$ & $d \leq n/6$ & yes & \makecell[l]{$O(n \ln(\min\{\frac{\mu}{|S_d|},d\}) + \frac{d^2 |S_d| n \ln(n)}{\mu}$ \\ $+ \frac{d(12en)^k}{k^{k-1}})$} & - \\
   \noalign{\vskip 4pt}
   \mOJZJ & $\leq (\frac{2n}{d}+1)^{d-1}$ & $\in \{2,4\}$ & no & $\Theta(n^{k+d}/\mu)$ & $O(n^k)$ \\
   \mRRMO & $\leq (\frac{4n}{5d}+1)^{d-1}$ & any & no & - & $n^{\Omega(n)}$ \\
   \mRRMO & $\leq (\frac{4n}{5d}+1)^{d-1}$ & any & yes & $O(n^3 + \frac{n (12n)^{2n/(5d)}}{(2n/(5d))!})$ & - \\
   \bottomrule
\end{tabular}
\end{center}
\end{table}

Secondly, we establish expected runtime bounds of \nsgaIII on several benchmark problems in terms of generations which are all depicted in Table~\ref{tab:overview-runtime-results}. The improved runtime bounds are all established in this paper, except for \mRRMO without stochastic population update, since the bound is already established in~\citep{OPRIS2026}. 
In case of \mLOTZ for instance, if $\mu \neq |S_d|$, the improved upper bound is by a factor of $\min\{\frac{\mu}{|S_d|},\frac{n}{\ln(\min\{\frac{\mu}{|S_d|},n\}})\}$ smaller than the known bounds from~\citep{OprisNSGAIII,DoerrNearTight} if the number of objectives $d$ is constant. The former is $\omega(1)$ for $\mu \in \omega(|S_d|) \cap 2^{o(n)}$. The reason for this improvement is the additional consideration of the population dynamics mentioned in the first point above. We also obtain similar improvements for other classical benchmark functions such as \mOMM and \mCOCZ, although the improvements there are not that large, particularly at most by a factor of $O(\ln(n))$. However, in case of \mOMM and \mCOCZ, if the number of objectives is small, we can show that the factor of improvement is always $\Omega(\ln(n))$ for $\mu = \Omega(n \ln(n) |S_d|)$, which does not depend on $\mu$ for $\mu = \Omega(d \ln(n) |S_d|)$. For \mOJZJ, the factor of improvement is $\Omega(\min\{\frac{n^{k-1}}{\ln(\min\{\mu/|S_d|,d\})},\frac{\mu n^{k-1}}{ |S_d| \ln(n)},\frac{\mu}{|S_d|}\})$ for $\mu > |S_d|$ and constant $d$ which is $\omega(1)$ for $\mu = \omega(|S_d|)$. This is an even larger parameter regime for $\mu = \omega(|S_d| \ln(n))$ compared to \mLOTZ. The reason is, that the runtime is dominated by the time to cross the fitness valleys via mutation, and accumulating many individuals in the local optima helps to escape from it more efficiently, since this increases the chance of choosing such an individual during parent selection.  Remarkably, for $d=2$ and $\mu \in \omega(n)$, our bound is $o(n^k\mu)$, and hence, \nsgaIII outperforms \nsga if also $\mu \in o(n^2/k^2)$ (compare with~\citep{DoerrQu2023a} for the tight runtime bound of $\Theta(\mu n^k)$ of \nsga on $2$-\OJZJ if $n+1 \leq \mu \in o(n^2/k^2)$). On $2$-\OMM, we also observe that \nsgaIII outperforms NSGA-II by a factor of $\min\{1/\ln(n), \mu/n\}$ under the assumption $4(n+1) \le \mu \le o(n^b)(n+1)$ for $b < 1$ (see~\citep{DoerrQu2023a} for a lower bound of $\Omega(n \ln n)$ generations for NSGA-II to optimize $2$-\OMM in this setting). In this context, we also establish some tight runtime bounds of \nsgaIII on \mOJZJ in case of $d \in \{2,4\}$ and population sizes $|S_d| \leq \mu \in O(n^{k-1})$ for $d=2$ and $|S_d| \leq \mu \in O(n^{k-2})$ for $d=4$, where also $k \geq 4$ in the latter. Results for $k<4$ in case of $d=4$ or alternative results for $d>4$ would require a much deeper understanding of the population dynamics of \nsgaIII, including the behavior of its solutions below the Pareto front of \mOJZJ, which lies beyond the scope of the present work. We refer to Section~6 for the details. This is another novel approach, since, as far as we know, the only rigorously proven lower runtime bounds in the many-objective case beyond the bi-objective one are provided in~\citep{Opris2025} on \mRRMO, which is quite different to \mOJZJ, and in~\citep{Opris2025PAES} on \mLOTZ for an EMOA called PAES-25, which has a much different working principle compared to \nsgaIII. It stores best-so-far solutions in an archive and always mutates a current solution instead of choosing a solution from a population uniformly at random.

Thirdly, we investigate \nsgaIII with stochastic population update where the next generation is formed not only by selecting the first ranked solutions resulting from non-dominated sorting (i.e. in a greedy, deterministic manner), but also by including some individuals chosen uniformly at random~\citep{UpBian}. This allows low-ranked but promising solutions to survive with a certain probability, which can lead to an exponential speedup in runtime: Using this variant with $\mu=|S_d|$ and $k>2$, \nsgaIII can optimize \mOJZJ in $O(k(bn/k)^k)$ generations, where $b>1$ is a suitable constant. This extends the results from~\citep{UpBian} for NSGA-II with $d=2$ to \nsgaIII and many objectives (see also~\citep{Zheng_Doerr_2024} for similar results regarding the SMS-EMOA). This is by at least a factor of $\Omega((k/b)^{k-1}n^{d/2}/\mu)$ smaller which is exponentially large if $k$ is linear in $n$ and $\mu=\Theta(n^{d/2})$. We also obtain a similar speedup on the \mRRMO function via stochastic population update. To achieve all these results, we also have to adapt the arguments from~\citep{OprisNSGAIII} about the protection of good solutions also to the case when stochastic population update is considered. For the sake of completeness, we allow stochastic population update when analyzing the standard benchmark problems, but we do not examine the resulting dynamics. Instead, we consider only those solutions that survive through non-dominated sorting and the reference point selection procedure.  

\textbf{Related Work}: There are several theoretical runtime analyses showcasing the efficiency of \nsga on bi-objective problems. The first was conducted by~\citep{ZhengLuiDoerrAAAI22} on $2$-\LOTZ and $2$-\OMM, followed by results on the multimodal problem $2$-\OJZJ~\citep{Qu2022PPSN}, about the usefulness of crossover~\citep{DANG2024104098,DoerrQ23b}, noisy environments~\citep{DaOp2023}, approximations of covering the Pareto front~\citep{Zheng2022}, lower bounds~\citep{DoerrQu2023a}, trap functions~\citep{DangEfficient2024} and stochastic population update~\citep{UpBian}. It is also shown that \nsga outperforms GSEMO exponentially which only relies on dominance relations~\citep{Lessons}. There are also results on combinatorial optimization problems like the minimum spanning tree problem~\citep{Cerf2023} or the subset selection problem~\citep{MOEASubset}. However, rigorous runtime results in many-objective optimization on simple benchmark functions appeared only in the last few years for the SMS-EMOA~\citep{Zheng_Doerr_2024}, the SPEA2~\citep{RenSPEA2}, variants of the \nsga~\citep{Krejca2025b} and the \nsgaIII~\citep{WiethegerD23,OprisNSGAIII,MOEAsLevels}. For the \nsgaIII, in~\citep{WiethegerD23} the first runtime analysis on the $3$-\OMM problem with $p \geq 21n$ divisions along each objective was conducted for defining the set of reference points. In \citep{OprisNSGAIII} this result was generalized to more than three objectives and runtime analyses for the classical $d$-\COCZ and $d$-\LOTZ benchmarks~\citep{Laumanns2004} for any constant number $d$ of objectives were provided where it is also necessary to reach the Pareto front. They could also reduce the number of required divisions by more than half. In~\citep{DoerrNearTight} these analyses were extended on the \mOJZJ function, and an arbitrary number of objectives, but without taking the population dynamics of \nsgaIII into account, and in~\citep{OPRIS2026}, new benchmarks were designed which showcase that \nsgaIII with crossover may lead to an exponential speedup compared to the case without. 
There are also several results for approximating the Pareto front if the size of the Pareto front is larger than the population size in case of NSGA-II~\citep{Zheng2022}, \nsgaIII~\citep{ApproximationNSGAIII} and SPEA-2~\citep{ApproximationSPEA2}, which is able to distribute the solutions more evenly on the Pareto front of $2$-\OMM than in case of NSGA-II. However, the analysis of population dynamics of \nsgaIII together with lower runtime bounds of other prominent MOEAs is still hardly understood. Until~\citep{OprisMultimodal,Opris26PopDyn,DoerrQu2023a}, to the best of our knowledge there are no tight runtime bounds for NSGA-II or \nsgaIII on classical benchmarks, despite extensive research on EMOA limitations and population dynamics. Notably, all these results are restricted to the case $d=2$, and currently we are only aware of one result of an EMOA called PAES-25 for the many-objective case beyond two objectives. Particularly, in~\citep{Opris2025PAES} it was shown that PAES-25 with one-bit mutation on the \mLOTZ problem, has a tight runtime of $\Theta(n^3)$ for $d=2$, $\Theta(n^3 \ln n)$ for $d=4$, and $\Theta(n (2n/d)^{d/2} \ln(n/d))$ for $d>4$. Additional tight bounds for GSEMO on several classical bi-objective benchmarks like $2$-\OMM, $2$-\OJZJ, and $2$-\COCZ are given in~\citep{doerr2025tightruntimeGSEMO}. To the best of our knowledge, there are no results on the population dynamics of \nsgaIII apart from investigations on \mOJZJ in~\citep{OprisMultimodal} and $2$-\OMM in~\citep{Opris26PopDyn}. A further restriction of these works is that they consider only populations that are Pareto optimal in the current and all subsequent generations. However, \citep{doerr2026improvedruntimeguaranteesspea2} provides closely related results on the cover number of individuals for SPEA2, which are shown to facilitate hill climbing and help cover the Pareto front, similar to the effects we exploit in this paper. Nevertheless, the resulting runtime analysis is restricted to the case $d=2$, and neither stochastic population update nor lower runtime bounds are considered there.

\section{Preliminaries}

\textbf{Notation:} For a finite set $A$ we denote by $\lvert{A}\rvert$ its cardinality and by $\ln$ the logarithm to base $e$. For $n \in \mathbb{N}$ let $[n]:=\{1, \ldots , n\}$ and we say that $\mu \in O(\text{poly}(n))$ if $\mu$ does not grow asymptotically faster than a polynomial in $n$. The number of ones in a bit string $x$ is denoted by $\ones{x}$ and the number of zeros by $\zeros{x}$, respectively. The number of leading zeros in $x$, denoted by $\LZ(x)$, is the length of the longest prefix of $x$ which contains only zeros, and the number of trailing zeros in $x$, denoted by $\TZ(x)$, the length of the longest suffix of $x$ containing only zeros, respectively. For example, if $x=00110110110000$, then $\LZ(x)=2$ and $\TZ(x)=4$. The number of leading ones in $x$, denoted by $\LO(x)$ and trailing ones in $x$, denoted by $\TOs(x)$, is the length of the longest prefix and suffix of $x$, which contains only ones, respectively. For two random variables $Y$ and $Z$ on $\mathbb{N}_0$ we say that $Z$ \emph{stochastically dominates} $Y$ if $\Pr(Z \leq c) \leq \Pr(Y \leq c)$ for every $c \geq 0$. For two vectors $v,w \in \mathbb{R}^d$ we denote by $v \circ w = \sum_{i=1}^d v_iw_i$ their inner product. For two $x,y \in \{0,1\}^n$ we denote by $H(x,y):=\sum_{i=1}^n |x_i-y_i|$ the \emph{Hamming distance} of $x$ and $y$. For a $d$-objective function $f:\{0,1\}^n \to \mathbb{N}_0^d, x \mapsto (f_1(x), \ldots , f_d(x))$, let $f_{\max}:=\max\{f_j(x) \mid x \in \{0,1\}^n, j \in [d]\}$ be the maximum possible objective value. When $d=2$, $f$ is also called \emph{bi-objective}. 
For two search points $x, y \in \{0, 1\}^n$, $x$ \emph{weakly dominates} $y$, written as $x \succeq y$,
if $f_{i}(x) \geq f_{i}(y)$ for all $i \in [d]$ and $x$ \emph{(strictly) dominates} $y$, written as $x \succ y$, if one inequality is strict. We call $x$ and $y$ \emph{incomparable} if neither $x \succeq y$ nor $y \succeq x$.
A set $S \subseteq \{0,1\}^n$ is a \emph{set of mutually incomparable solutions} if all search points in $S$ are incomparable. Each solution $x$ not dominated by any other in $\{0,1\}^n$ is called \emph{Pareto optimal} and we call $f(x)$ \emph{non-dominated fitness value}. The set of all non-dominated fitness values is called \emph{Pareto front}. For a population $P_t$ the \emph{cover number} $c_t(v)$ of $v \in \mathbb{N}_0^d$ is the number of individuals $x \in P_t$ with $f(x)=v$ and we say that $v$ is \emph{covered} if its cover number is at least $1$. \\

\textbf{Mathematical Tools:}
The following lemma describes some useful monotonicity properties. 

\begin{lemma}
\label{lem:mon-property}
The function $g: \text{}]0,n/4[\text{} \to \mathbb{R}, x \mapsto  (2n/(3x)+1/2)^{x/2},$ is strictly monotone increasing.
\end{lemma}

\begin{proof}
We have 
$$\ln(g)' = \frac{\ln(2n/(3x)+1/2)}{2} -  \frac{x/2 \cdot 2n/(3x^2)}{2n/(3x)+1/2} = \frac{\ln(2n/(3x)+1/2)}{2} - \frac{2n}{4n+3x}.$$
Note that $\ln(g)'>0$ since $\ln(2n/(3x)+1/2) \geq \ln(3) \geq 1$ for $x \in \text{}]0,n/4[$, concluding the proof.
\end{proof}

The following inequalities are also useful for our analyzes

\begin{lemma}
	\label{lem:inequalities}
	For every $n \in \mathbb{N}$, we have $\sum_{j=n+1}^{\infty} 1/j^2 \leq 1/n$.
\end{lemma}

\begin{proof}
	For $n \in \mathbb{N}$, consider the function $f: \mathbb{R}_{>0} \to \mathbb{R}, x \mapsto 1/x^2$, which is positive and monotonically decreasing, which implies (see for example Inequality~(A.12) in~\citep{CormenLRS01IntroductionToAlgorithms})
	\[
	\sum_{j=n+1}^{\infty} \frac{1}{j^2}
	\;\leq\;
	\int_{n}^{\infty} \frac{1}{x^2}\,dx = \left[ -\frac{1}{x} \right]_{n}^{\infty} = \frac{1}{n},
	\]
	proving the claim.
\end{proof}

The following bounds are useful for estimating probabilities when stochastic amplifications occur.

\begin{lemma}[Lemma~10 in \citet{Badkobeh2015}]\label{lem:Badkobeh}
	For every $p\in[0,1]$ and every $\lambda\in\Natural$,
	\begin{align*}
		1-(1-p)^\lambda \in \left[\frac{p\lambda}{1+p\lambda},\frac{2p\lambda}{1+p\lambda}\right].
	\end{align*}
\end{lemma}

We also require a tail bound for the sum of independent geometric random variables, which is a valuable tool for deriving runtime bounds that hold not only in expectation, but also with high probability.

\begin{lemma}[Theorem 15 in~\citet{DOERR2019115}]\label{lem:Doerr-dominance}
	Let $X_1, \ldots , X_n$ be independent geometric random variables with success probabilities $p_1, \ldots , p_n$. Let $X=\sum_{i=1}^n X_i$, $s=\sum_{i=1}^n (1/p_i)^2$ and $p_{\min}:=\min\{p_i \mid i \in [n]\}$. Then for all $\lambda \geq 0$
	\begin{enumerate}
		\item[(1)]
		$\Pr(X \geq \expect{X} + \lambda) \leq \exp \left(-\frac{1}{4} \min \left\{\frac{\lambda^2}{s}, \lambda p_{\min}\right\}\right)$,
		\item[(2)]
		$\Pr(X \leq \expect{X} - \lambda) \leq \exp \left(-\frac{\lambda^2}{2s}\right)$.
	\end{enumerate} 
\end{lemma}

We frequently consider sums of independent geometrically distributed random variables $X_1, \ldots, X_n$, where the success probability of $X_i$ is proportional to $i$.

\begin{lemma}[Theorem 16 in~\citet{DOERR2019115}]\label{lem:Doerr-dominance2}
Let $X_1, \ldots, X_n$ be independent geometric random variables with success probabilities $p_1, \ldots, p_n$. Assume that there exists a value $C \le 1$ such that $p_i \ge C \frac{i}{n}$ for all $i \in \{1, \ldots ,n]$. Let $X = \sum_{i=1}^n X_i$. Then the following holds. 
\begin{enumerate}
    \item[(1)] $\expect{X} \le \frac{1}{C} n H_n \le \frac{1}{C} n (1 + \ln n)$,
    \item[(2)] $\Pr(X \ge (1+\delta)\tfrac{1}{C} n \ln n) \le n^{-\delta} \text{ for all } \delta \ge 0$.
\end{enumerate}
\end{lemma}

\begin{algorithm}[t]
	Initialize $P_0 \sim \text{Unif}( (\{0,1\}^n)^{\mu})$\\
	\For{$t:= 0$ to $\infty$}{
		Initialize $Q_t:=\emptyset$\\
		\For{$i=1$ to $\mu$}{
                Sample $s$ from $P_t$ uniformly at random\\
                Create $r$ by standard bit mutation on $s$ with mutation probability $1/n$\\
                Update $Q_t:=Q_t \cup \{r\}$\\
                }
            Set $R_t := P_t \cup Q_t$\\
            \If{$a=1$}{
		Update $R_t$ by choosing $\lceil{3\mu/2}\rceil$ solutions from $R_t$ uniformly at random without replacement~\label{line:update}\\
            }
		Partition $R_t$ into layers $F^1_t,F^2_t,\ldots ,F^k_t$ of non-dominated solutions \label{line:non-dominated}\\
            \If{$a=0$}{
            Find $i^* \geq 1$ such that $\sum_{i=1}^{i^*-1} \lvert{F_t^i}\rvert < \mu$ and $\sum_{i=1}^{i^*} \lvert{F_t^i}\rvert \geq \mu$
            }
            \Else{
            Find $i^* \geq 1$ such that $\sum_{i=1}^{i^*-1} \lvert{F_t^i}\rvert < \lceil{\mu/2}\rceil$ and $\sum_{i=1}^{i^*} \lvert{F_t^i}\rvert \geq \lceil{\mu/2\rceil}$
            }
            Compute $Y_t = \bigcup_{i=1}^{i^*-1} F_t^i$\\
            Choose $\tilde{F}_t^{i^*} \subset F_t^{i^*}$ such that $\lvert{Y_t \cup \tilde{F}_t^{i^*}}\rvert = \mu$ if $a=0$ and $\lvert{Y_t \cup \tilde{F}_t^{i^*}}\rvert = \lceil{\mu/2}\rceil$ if $a=1$ with Algorithm~\ref{alg:Survival-Selection}\\
		Create the next population $P_{t+1} := Y_t \cup \tilde{F}^{i^*}_t$ if $a=0$ and $P_{t+1} := Y_t \cup \tilde{F}^{i^*}_t \cup ((P_t \cup Q_t) \setminus R_t)$ if $a=1$ \label{line:survival}\\
	}
	\caption{NSGA-III~(\protect\citep{DebJain2014}) with population size $\mu$, stochastic population update ($a=1$) and without ($a=0$) on a $d$-objective function $f$ for $d \in \mathbb{N}$}
	\label{alg:nsga-iii}
\end{algorithm}

\textbf{Algorithm:} The \nsgaIII algorithm, originated in~\citep{DebJain2014}, with or without \emph{stochastic population update} controlled by the parameter $a \in \{0,1\}$, is shown in Algorithm~\ref{alg:nsga-iii}. If $a=0$, stochastic population update is turned off, while it is turned on, if $a=1$. Initially, a population of size $\mu$ is created by choosing $\mu$ individuals from $\{0,1\}^n$ uniformly at random. Then in each iteration $t$, a multiset $Q_t$ of $\mu$ new offspring is created by $\mu$ times choosing an individual $s \in P_t$ uniformly at random and applying \emph{standard bit mutation} on $s$, i.e. each bit is flipped independently with probability $1/n$ which we call \emph{trial}. During the survival selection, the parent and offspring populations $P_t$ and $Q_t$ are merged into $R_t$. When stochastic population update is turned on then $R_t$ is updated by choosing $\lceil{3\mu/2}\rceil$ individuals from $R_t$ uniformly at random without replacement. Then $R_t$ is partitioned into layers $F^1_t,F^2_t,\dots$ using the \emph{non-dominated sorting algorithm}~\citep{Deb2002} where $F^1_t$ consists of all non-dominated individuals, and $F^i_t$ for $i>1$ of individuals only dominated by those from $F^1_t,\dots,F^{i-1}_t$. Then if $a=0$ the critical rank $i^*$ with $\sum_{i=1}^{i^*-1} \lvert{F_t^i}\rvert < \mu$ and $\sum_{i=1}^{i^*} \lvert{F_t^i}\rvert \geq \mu$ is determined (i.e. there are fewer than $\mu$ search points in $R_t$ with a lower rank than $i^*$, but at least $\mu$ search points with rank at most $i^*$). If $a=1$, then every individual not selected for $R_t$ survives. The number of such individuals is $\mu-\lceil{3\mu/2}\rceil = \lfloor{\mu/2}\rfloor$. Hence, in this case, the critical index $i^*$ refers only on $\mu-\lfloor{\mu/2}\rfloor = \lceil{\mu/2}\rceil$ individuals. All individuals with a lower rank than $i^*$ are included in $P_{t+1}$, while the remaining individuals are selected from $F_t^{i^*}$ using Algorithm~\ref{alg:Survival-Selection}. 

\begin{algorithm}[t]
     $E_{-1} = \emptyset$, $E_t=\emptyset$\\
	\For{$j = 1$ to $d$}{
	$y^{\min}_j = \min\{y^{\min}_j, \min_{y \in R_t} (f_j(y))\}$, $y^{\max}_j = \max\{y^{\max}_j, \max_{y \in F_t^1} (f_j(y))$\}\label{line:max}\\
    Determine an extreme point $e^{(j)}$ in the $j$-th objective from $f(Y_t \cup F_t^{i^*}) \cup E_{t-1}$ using an achievement scalarization function; $E_t = E_t \cup \{e^{(j)}\}$~\label{line:Extreme-points}}
        valid = \textit{false}\\
            \If{$e^{(1)}, \ldots , e^{(d)}$ are linearly independent~\label{line:lin-independent}}{
             valid = \textit{true}; let $H$ be the hyperplane spanned by $e^{(1)}, \ldots , e^{(d)}$\\
            \For{$j=1$ to $d$}{
            Determine the intercept $I_j$ of $H$ with the $j$-th objective axis\\
            \lIf{$I_j \geq \varepsilon_{\text{nad}}$ and $I_j \leq y^{\max}_j$~\label{line:Second-If}}{
            $y_j^{\text{nad}} = I_j$\label{line:intercept}}
            \Else{
            valid = \textit{false}; \textbf{break}\label{line:False}
                    }
                }
            }
		\If{valid = \textit{false}}{
            \For{$j=1$ to $d$}{
            $y_j^{\text{nad}} = \max_{y \in F_t^1}f_j(y)$~\label{line:valid-false}
                }
            }
            \For{$j=1$ to $d$}{
            \If{$y_j^{\text{nad}}<y^{\min}_j + \varepsilon_{\text{nad}}$}{
            $y_j^{\text{nad}} = \max_{y \in F_t^1 \cup \ldots \cup F_t^k}f_j(y)$~\label{line:normalization-value}
                }
            }
         Let $f_j^n(x) = (f_j(x)-y^{\min}_j)/(y_j^{\text{nad}}-y^{\min}_j)$ for each $x \in \{0,1\}^n$ and $j \in \{1, \ldots , d\}$ \label{line:output}
	\caption{Normalization procedure \textsc{Normalize} \citep{DebJain2014,WiethegerD23}}
	\label{alg:normalization}
\end{algorithm}

Hereby, a normalized objective function $f^n$ is computed according to Algorithm~\ref{alg:normalization} and then each individual with rank at most $i^*$ is associated with reference points. For Algorithm~\ref{alg:normalization}, we use the normalization procedure from~\citep{WiethegerD23} which can be also used for maximization problems as shown in~\citep{OprisNSGAIII} if the given positive threshold $\varepsilon_{\text{nad}}$ is sufficiently large. This will be explained in the following.

\begin{algorithm}[t]
        Compute the normalization $f^n$ of $f$ using Algorithm~\ref{alg:normalization}\\
        Associate each $x\in Y_t \cup F_t^{i^*}$ with its reference point $\rp(x)$ such that the distance between $f^n(x)$ and the line through the origin and $\rp(x)$ is minimized\\ 
        For each $r \in \refer$, set $\rho_r:=|\{x\in Y_t \mid \mathrm{rp}(x)=r\}|$\\
        Initialize $\tilde{F}_t^{i^*}=\emptyset$ and $R':=\refer$\\
        \While{true}{
        Determine $r_{\min} \in R'$ such that $\rho_{r_{\min}}$ is minimal (where ties are broken randomly)~\label{line:minimal}\\
        Determine $x_{r_{\min}} \in F_t^{i^*} \setminus \tilde{F}_t^{i^*}$ which is associated with $r_{\min}$ and minimizes the distance between the vectors $f^n(x_{r_{\min}})$ and $r_{\min}$ (where ties are broken randomly)\label{line:association}\\
        \If{$x_{r_{\min}}$ exists}{
        $\tilde{F}_t^{i^*} = \tilde{F}_t^{i^*} \cup \{x_{r_{\min}}\}$\\
        $\rho_{r_{\min}} = \rho_{r_{\min}} + 1$\\
            \If{$a=0$ \emph{and} $\lvert{Y_t}\rvert + \lvert{\tilde{F}_t^{i^*}}\rvert = \mu$}
                {\Return{$\tilde{F}_t^{i^*}$}}
            \If{$a=1$ \emph{and} $\lvert{Y_t}\rvert + \lvert{\tilde{F}_t^{i^*}}\rvert = \lceil{\mu/2}\rceil$}
                {\Return{$\tilde{F}_t^{i^*}$}}
                }
            \lElse{$R'=R' \setminus \{r_{\min}\}$}
        }
	\caption{Selection procedure utilizing a set $\refer$ of reference points to maximize a function}
	\label{alg:Survival-Selection}
\end{algorithm}

Let $t$ be any generation. Let $y_j^{\min}$ and $y_j^{\max}$ be the minimum and maximum value in the $j$-th objective over all joined populations from generation $t$ and previous ones (i.e. from $R_0, \ldots R_t$), respectively. 
The point $y^{\min}:=(y^{\min}_1, \dots, y^{\min}_m)$ is called \emph{the ideal point} in \citep{DebJain2014}. Then for each objective $j$ we compute an extreme point $e^{(j)}$ in that objective by using an achievement scalarization function (see Section 4.2 of~\citep{Blank2019} for the exact definition), and the hyperplane spanned by these $e^{(j)}$ is denoted $H$. The intercept $I_j$ of $H$ 
with the $j$-th objective axis gives the so-called \emph{Nadir point estimate} $y_j^{\text{nad}}$. If $H$ has smaller dimension than $d$ or if an intercept is either larger than a given positive threshold $\varepsilon_{\text{nad}}$ or smaller than $y_j^{\max}$, $y_j^{\text{nad}}$ is set to $\max_{y\in F_t^1} f_j(y)$, the maximum of the $j$-th objective values which occur in the current first layer $F_t^1$ (which may differ from $y_j^{\max}$). We also replace $y_j^{\text{nad}}$ by $\max_{y\in F_t^1} f_j(y)$ if the Nadir point estimate is smaller than $y_j^{\min} + \varepsilon_{\text{nad}}$. Then the normalization $f_j^n(x)$ is computed as in Line~\ref{line:normalization-value} in Algorithm~\ref{alg:normalization} and the normalized objective function $f^n$ is defined as $(f_1^n , \ldots , f_d^n)$. The goal of normalization is to have the normalized fitness vectors in $[0,1]^{d}$ while preserving or properly emphasizing the dominance relations \citep{Blank2019}. Note that an individual $x$ (weakly) dominates another individual $y$ with respect to $f$ if and only if $x$ (weakly) dominates $y$ with respect to $f^n$.

After computing the normalization, each individual $x$ is associated with the reference point $\text{rp}(x)$ such that the distance between $f^n(x)$ and the line through the origin and $\text{rp}(x)$ is minimal (see also Figure~\ref{fig:reference-points}). Ties are broken deterministically, which means that two individuals with the same normalized objective vector are associated to the same reference point. We use the same set of reference points $\refer$ as
proposed in~\citep{DebJain2014}. The points are defined as
\[
\left\{\left(\frac{a_1}{p}, \ldots ,\frac{a_d}{p} \right) 
    \text{ } \Big| \text{ }  
    (a_1,\dots,a_d) \in \mathbb{N}_0^d, 
    \sum_{i=1}^d a_i = p
\right\}
\]
where $p \in \mathbb{N}$ is a parameter one can choose according 
to the fitness function $f$. These are uniformly distributed on the simplex determined by the unit vectors $(1,0,\dots,0)^{\intercal},(0,1,\dots,0)^{\intercal},\dots,(0,0,\dots,1)^{\intercal}$. 


Then, one iterates through all the reference points where the reference point with the fewest associated individuals that are already selected for the next generation $P_{t+1}$ is chosen. A reference point is omitted if it only has associated individuals that are already selected for $P_{t+1}$ and ties are broken uniformly at random. Next, from the individuals associated to that reference point who have not yet been selected, the one closest to the chosen reference point is selected for the next generation, where ties are again broken uniformly at random. Once the required number of individuals is reached (i.e. if $\lvert{Y_t}\rvert+\lvert{\tilde{F}_t^{i^*}}\rvert = \mu$ when $a=0$ and $\lvert{Y_t}\rvert+\lvert{\tilde{F}_t^{i^*}}\rvert = \lceil{\mu/2}\rceil$ when $a=1$) the selection ends. 

In this paper, we evaluate the performance of \nsgaIII on a problem by the expected number of generations. Another common measure is the expected number of fitness evaluations. However, this can be obtained directly by multiplying the expected number of generations by $\mu$.  

\begin{figure}
    \centering
    \includegraphics[scale=0.7]{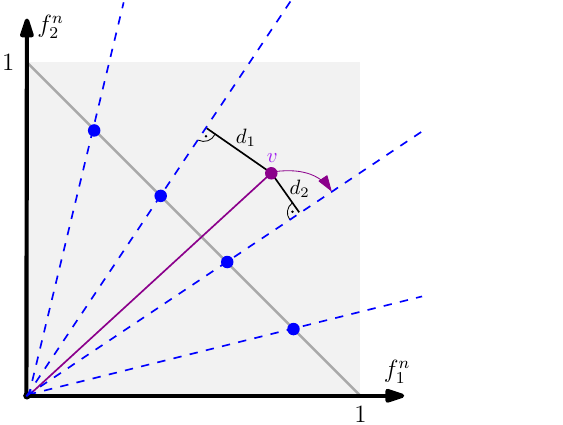}
    \caption{Illustrating how search points with fitness vector $v$ are associated with reference points (dots on line through $(1,0)$ and $(0,1)$ connected by dashed lines through the origin) for $d=2$ objectives in the normalized objective space. The vector $v$ is associated to the nearest reference point to its right.}
    \label{fig:reference-points}
\end{figure}

\section{Structural Results}

We recall the definitions and results from~\citep{OprisNSGAIII}. The major goal is to show that, by employing sufficiently many reference points, once the population covers a fitness vector with a solution $x \in F_t^1$, it is covered for all future generations as long as $x \in F_t^1$. Such a statement was already shown in the groundbreaking paper \citep{WiethegerD23}, but this was limited to the $3$-OMM problem and some precision was lost
when they applied the inequality $\arccos(1-1/24)>2\arccos(1-18/42^2)$. We give a general proof that \nsgaIII maintains good solutions using different analytical arguments and require only $p=2 \cdot 3^{3/2}n \thickapprox 10.39n$ divisions in case of $3$-\OMM, which is approximately by a factor of $2$ smaller. In contrast to~\citep{WiethegerD23} we compute distances from search points to their nearest reference point via the sine function of the angle between the lines running through both points and use a general trigonometrical identity for $\arcsin(x)$ for any $0<x\leq 1$. We start with a lemma about the normalized objectives of a multiobjective function.

\begin{lemma}
\label{lem:Normalization}
Consider \nsgaIII optimizing a many-objective function $f:\{0,1\}^n \to \mathbb{N}_0^d$ with ($a=1$), or without stochastic population update ($a=0$), and suppose that $\varepsilon_{\mathrm{nad}} \geq f_{\max}$. In every generation, after running Algorithm~\ref{alg:normalization}, for every $x \in R_t$ and every objective~$j$ we have $0 \leq f_j^n(x) \leq 1$ and $y_j^{\mathrm{nad}}-y_j^{\min} \leq f_{\max}$.
\end{lemma}

\begin{proof}
Until Line~\ref{line:normalization-value} of the algorithm, we note that $y_j^{\mathrm{nad}}$ is either set to (i) $\max_{y \in F_t^1} f_j(y)$ or (ii) the intercept $I_j$. For (i), since $f_{\max} \geq \max_{y \in F_t^1} f_j(y)$, $y_j^{\min} \geq 0$ and $f_j(x) \leq \max_{y \in F_t^1}f_j(y)$ we get 
    \begin{align}
    (y_j^{\mathrm{nad}}-y_j^{\min} \leq f_{\max})  
    \wedge (f_j(x) \leq y_j^{\mathrm{nad}})
    \label{eq:proper-normalization}
    \end{align}
and the results follow. In the case of (ii), the setting is triggered by Line~\ref{line:Second-If} of the algorithm, thus the condition of the if statement must be true. Specifically $I_j \geq \varepsilon_{\mathrm{nad}} \geq f_{\max}$ and $I_j \leq y_j^{\max} \leq f_{\max}$ and this implies $y_j^{\mathrm{nad}}=I_j = f_{\max}$. If the condition in the next for loop is also triggered then we fall back to case (i), otherwise note that \eqref{eq:proper-normalization} still holds because $f_j(x)\leq f_{\max}$.   
\end{proof}

The aim for the rest of this section is to generalize Lemma~2 in~\citep{WiethegerD23} to arbitrary objective functions with an arbitrary number of objectives. We start with the following technical result about general vectors $v,w \in [0,1]^d$.

\begin{lemma}
\label{lem:angle-helper}
    Let $v,w \in [0,1]^d$, and suppose that there are indices $i, j$ with $w_i=v_i+\ell_1$ and $w_j=v_j-\ell_2$ where $\ell_1,\ell_2 \geq 1/f_{\max}$. 
    Denote by $\varphi$ the angle between $v$ and $w$. Then $\sin(\varphi) \geq 1/(\sqrt{d} f_{\max})$.
\end{lemma}

\begin{proof}
    Without loss of generality we may assume that $i=1$ and $j=2$. Since $v$ and $w$ have non-negative components we have that $v \circ w \geq 0$ and hence $\cos(\varphi) \in [0,1]$ which implies $\varphi \in [0, \pi/2]$. Thus, $\sin(\varphi) = q/\lvert{v}\rvert$ where $q$ is the distance of the point $v$ to the line $g$ through the origin and the point $w$. 
    Note that $g$ is defined as $g=\beta \cdot w$ for $\beta \in \mathbb{R}$. We estimate $q$: let $h$ be the line through $v$, intersecting $g$, and which is orthogonal to $g$, i.e. $h=v+\lambda \cdot b$ for $\lambda \in \mathbb{R}$ and a suitable $b \in \mathbb{R}^{d}$. The intersection point of $g$ and $h$ is $\mu \cdot w$ for a suitable $\mu \in \mathbb{R}$ and thus the distance $q$ of $v$ to this intersection point is $q=\lvert{v-\mu w}\rvert$. We obtain
    \begin{align*}
    q&=\lvert{(v_1, \ldots , v_d) - \mu \cdot (v_1 + \ell_1, v_2 - \ell_2,w_3, \ldots ,w_d)}\rvert \\
    &=\lvert{((1-\mu)v_1 - \mu\ell_1, (1-\mu)v_2 + \mu\ell_2,v_3 - \mu w_3, \ldots ,v_d - \mu w_d)}\rvert.
    \end{align*}
    If $\mu \geq 1$ then $q \geq 1/f_{\max}$ as $(1-\mu)v_1 - \mu\ell_1 \leq - \ell_1 \leq -1/(f_{\max})$ and if $\mu \leq 1$ then $q \geq (1-\mu)v_2 + \mu \ell_2 \geq \ell_2 \geq 1/f_{\max}$ as $v_2-\ell_2 \geq 0$ implies $v_2 \geq \ell_2$. This implies $\sin(\varphi) \geq 1/(\lvert{v}\rvert f_{\max}) \geq 1/(\sqrt{d} f_{\max})$ since $\lvert{v}\rvert \leq \sqrt{d}$ due to $\lvert{v_i}\rvert \leq 1$ for all $i \in [d]$.
\end{proof}

We also need the following trigonometrical identity.

\begin{lemma}
\label{lem:Trigonometry}
    Let $x \in (0,1]$. Then $2\arcsin(x/2) < \arcsin(x)$.
\end{lemma}
\begin{proof}
For $a,b \in \mathbb{R}$ we use the well known trigonometric identity $\sin(a+b)=\sin(a)\cos(b)+\cos(a)\sin(b)$.
For $y,z \in [-1,1]$ we obtain by using $\cos(y)=\sqrt{1-\sin^2(y)}$
\begin{align*}
& \sin(\arcsin(y)+\arcsin(z)) \\
&=\sin(\arcsin(y))\cos(\arcsin(z))+\cos(\arcsin(y))\sin(\arcsin(z))\\
&=y\sqrt{1-z^2}+z\sqrt{1-y^2}.
\end{align*}
And hence if additionally $y,z \in [-0.5,0.5]$ (since then $-\pi/2 \leq \arcsin(y)+\arcsin(z) \leq \pi/2$) 
\[
\arcsin(y) + \arcsin(z) = \arcsin(y \cdot \sqrt{1-z^2} + z \cdot \sqrt{1-y^2})
\]
which implies for $x \in (0,1]$ (where $y=z=x/2$)
\begin{align*}
2\arcsin(x/2) &= \arcsin((x/2) \sqrt{1-x^2/4}+(x/2) \sqrt{1-x^2/4})\\ 
&= \arcsin(x \sqrt{1-x^2/4}) < \arcsin(x)
\end{align*}
where the inequality holds since $\arcsin$ strictly increases on $(0,1]$. 
\end{proof}

Now we are able to prove the main lemma for this section.

\begin{lemma}
\label{lem:Same-Reference}
    Consider \nsgaIII with or without stochastic population update optimizing a $d$-objective function~$f$ with $\varepsilon_{\text{nad}} \geq f_{\max}$. 
    If the set of reference points $\refer$ is created using $p$ divisions along each objective and $p \geq 2d^{3/2} f_{\max}$, 
    all search points in $F_t^1$ that are associated with the same reference point from $\refer$ have the same fitness vector.
\end{lemma}

\begin{proof}
An individual $x$ is associated with the reference point $r$ that minimizes the distance between $f^n(x)$, and the line through the origin and $r$. Particularly, the angle between this point and $f^n(x)$ is minimized (see Figure~\ref{fig:reference-points}). To show the statement, we first 
    \begin{itemize}
        \item[(i)] lower bound the angle between two individuals from $F_t^1$ with different fitness vectors by 
        
        $\arcsin(1/(\sqrt{d} f_{\max}))$ and 
        \item[(ii)] upper bound the angle between a normalized fitness vector of $f$ and its nearest reference point by $\arcsin(1/(2\sqrt{d}f_{\max}))$.
    \end{itemize}    
    Hence, the lower bound in (i) is more than twice the upper bound in (ii) as by Lemma~\ref{lem:Trigonometry} we have that $\arcsin(1/(\sqrt{d} f_{\max}))/\arcsin(1/(2\sqrt{d}f_{\max})) > 2$ for $x=1/(\sqrt{d}f_{\max})$.  
Then search points with different fitness vectors are never associated with the same reference point, because if they were then the angle between their normalized fitness vectors would be at most the sum of their angles to that reference point.

    (i): Let $x,z \in F_t^1$ be two search points with distinct fitness. 
As $x$ and $z$ are incomparable and the domain of $f$ is $\mathbb{N}_0^d$, the fitness vectors $f(x)$ and $f(z)$ differ by at least $1$ in two objectives before normalization. After normalization, the difference of $f^n(x)$ and $f^n(z)$ correspond to at least $1/f_{\max}$ 
in those objectives since $y_i^{\text{nad}}-y_i^{\text{min}} \leq f_{\max}$ by Lemma~\ref{lem:Normalization} for $i \in [d]$. Thus there are $i,j \in [d]$ such that $f^n(z)_i=f^n(x)_i+\ell_1 \geq 0$ and $f^n(z)_j = f^n(x)_j - \ell_2 \geq 0$ for $\ell_1,\ell_2 \geq 1/f_{\max}$. By Lemma~\ref{lem:angle-helper} we have that $\sin(\varphi) \geq 1/(\sqrt{d}f_{\max})$ where $\varphi$ is the angle between $f^n(x)$ and $f^n(z)$ (since $f^n(x),f^n(z) \in [0,1]^d$ by Lemma~\ref{lem:Normalization}). Since $\arcsin$ is increasing and $\varphi \in [0,\pi/2]$ (owing to $f^n(x) \circ f^n(z) \geq 0$ we have that $\cos(\varphi) \in [0,1]$), we obtain $\varphi \geq \arcsin(1/(\sqrt{d}f_{\max}))$.
    
(ii): We scale $f^n(x)$ to $b = a \cdot f^n(x)$ for $a>0$ so that $b_1 + \ldots + b_d = 1$. We claim that there is $r \in \refer$ with $\lvert{b_i-r_i}\rvert \leq 1/p$ for every $i \in [d]$ which implies $\lvert{b-r}\rvert \leq \sqrt{d}/p$.
At first choose $c_1, \ldots , c_d \in \mathbb{N}_0$ with $b_i \in [c_i/p,(c_i+1)/p]$. Since $b_1 + \ldots + b_d = 1$, we see $\sum_{i=1}^d c_i/p \leq 1$ and $\sum_{i=1}^d (c_i+1)/p \geq 1$. Hence, $\sum_{i=1}^d c_i \leq p$ and $\sum_{i=1}^d (c_i+1) \geq p$. Thus, there is $\ell \in [d]$ with $\left(\sum_{i=1}^\ell c_i\right) + \sum_{i=\ell+1}^d (c_i+1) = p$. So for $i \in [d]$ choose $r_i=c_i/p$ if $i\leq \ell$ and $r_i=(c_i+1)/p$ if $i > \ell$. Then $\lvert{b_i-r_i}\rvert \leq 1/p$ and $(r_1, \ldots , r_d) \in \refer$ since $r_1 + \ldots + r_m = 1$. Now let $\vartheta$ be the angle between $b$ and $r$. Then $\sin(\vartheta) = q/\lvert{b}\rvert$ where $q$ is the distance between the point $b$ and the line through the origin and $r$. Note that $q \leq \lvert{b-r}\rvert \leq \sqrt{d}/p$ and $\lvert{b}\rvert \geq 1/\sqrt{d}$ (since $b_1 + \ldots + b_d=1$). Thus $\sin(\vartheta) \leq \sqrt{d} \cdot \sqrt{d}/p = d/p \leq 1/(2\sqrt{d}f_{\max})$. 
Since $\arcsin$ is increasing, this yields that the angle between $b$ and $r$ (and hence also the angle between $f^n(x)$ and $r$) is at most $\arcsin(d/p) \leq \arcsin(1/(2\sqrt{d}f_{\max}))$ (because $r \circ t \geq 0$, we have that $\cos(\vartheta) \in [0,1]$, i.e. $\vartheta \in [0,\pi/2]$).
\end{proof}


Arguing as in~\citep{WiethegerD23}, the fact that search points assigned to different reference points have different fitness vectors implies that for every first ranked individual $x$ there is an individual $y$ with the same fitness vector which is taken into the next generation.
The number of reference points is $|\refer|=\binom{p+d-1}{d-1}$. So for example, if $d=O(1)$ and we set $p = 2d^{3/2} f_{\max} = \Theta(f_{\max})$, then $|\refer|=\Theta(f_{\max}^{d-1})$.\\

The following result, originally formulated for the vanilla \nsgaIII (where $a=0$), can be extended to the case where stochastic population update is enabled (where $a=1$). It shows that if a population covers a fitness vector $v$ by means of a first-ranked individual $x$, then this coverage is preserved in all subsequent generations as long as $f(x)$ remains non-dominated. In other words, the best solutions are protected.

\begin{lemma}
\label{lem:Reference-Points}
    Consider \nsgaIII with or without stochastic population update on a $d$-objective function $f$ with $\varepsilon_{\text{nad}} \geq f_{\max}$ and a set $\refer$ of reference points for $p \in \mathbb{N}$ with $p \geq 2d^{3/2}f_{\max}$. Let $P_t$ be its current population and $\mu$ its population size. Assume $\mu \geq (1+a)\lvert{S_d}\rvert$, where $S_d$ is a maximum set of mutually incomparable solutions, meaning that $\mu \geq |S_d|$ if stochastic population update is turned off, while $\mu \geq 2|S_d|$ if it is turned on. Let $\tilde{F}_t^1,\tilde{F}_t^2, \ldots$ be the layers of non-dominated fitness vectors of $P_t \cup Q_t$. Then for every $x \in \tilde{F}_t^1$ there is an $x' \in P_{t+1}$ with $f(x')=f(x)$.
\end{lemma}

\begin{proof}
We show the lemma separately for both possible cases $a=0$ and $a=1$.

Suppose that $a=0$. If $\lvert{\tilde{F}_t^1}\rvert \leq \mu$ all individuals in $\tilde{F}_t^1$ including $x$ survive. Otherwise, the objective functions with respect to $\tilde{F}_t^1$ are normalized. 
Let $\ell \coloneqq \lvert{\{f(x) \mid x \in \tilde{F}_t^1\}\rvert}$
be the number of different fitness vectors of individuals from $\tilde{F}_t^1$. 
Given the condition on $p$, Lemma~\ref{lem:Same-Reference} implies that individuals with distinct fitness vectors in $\tilde{F}^1_t$ are associated with different reference points. 
Thus, there are as many as $\ell$ reference points, each associated with at least one individual. Since $\mu\geq \ell$, at least 
one individual $x' \in R_t$ associated with the same reference point as $x$ survives and we have $f(x)=f(x')$.

So suppose that $a=1$. Let $x \in \tilde{F}_t^1$. If $x \notin R_t$ then the claim holds (by Line~\ref{line:survival} in Algorithm~\ref{alg:nsga-iii} we have $(P_t \cup Q_t) \setminus R_t \subset P_{t+1}$). Suppose that $x \in R_t$. Then $x \in F_t^1$ where $F_t^1$ is the first layer of non-dominated solutions from Line~\ref{line:non-dominated} in Algorithm~\ref{alg:nsga-iii} (since if $x$ is non-dominated with respect to $P_t \cup Q_t$ then it is also non-dominated with respect to the updated $R_t$ from Line~\ref{line:update} in Algorithm~\ref{alg:nsga-iii}). If $\lvert{F_t^1}\rvert \leq \lceil{\mu/2}\rceil$ all individuals in $F_t^1$ (including $x$) survive. Hence, suppose that $\lvert{F_t^1}\rvert > \lceil{\mu/2}\rceil$. Then by Lemma~\ref{lem:Same-Reference} there are $\ell:=\lvert\{f(x) \mid x \in F_t^1\}\rvert$ reference points, each associated with at least one individual. Since $\lceil{\mu/2}\rceil \geq \lfloor{\mu/2}\rfloor \geq |S_d| \geq \ell$, at least one individual $x' \in R_t$ associated with the same reference point as $x$ survives and we have $f(x)=f(x')$.
\end{proof}

We formally describe the population dynamics of \nsgaIII using the cover number $c_t(v)$ of a fitness vector $v$, defined as the number of individuals in $P_t$ with fitness value $v$. These dynamics differ significantly from those of NSGA-II. In NSGA-II, which uses the crowding distance as a secondary tie-breaker, all rank-one individuals with crowding distance zero are treated equally during survival selection, regardless of their distribution in the objective space. In particular, for sufficiently large population sizes $\mu$, if $c_t(v) \ge 5$ for some fitness vector $v$, then there is a non-zero probability that one of these individuals is deleted~\citep{ZhengLuiDoerrAAAI22}. In contrast, \nsgaIII with or without stochastic population update can maintain up to $\lfloor{\mu / ((1+a)|S_d|)}\rfloor$ solutions with the same fitness vector, as long as they are non-dominated. This fact may lead to a very uniform distributions of Pareto optimal individuals across the objective space. We formalize this behavior in the following lemma, which generalizes the corresponding Lemma~2 from~\citep{OprisMultimodal}, as it also applies to scenarios where stochastic population update is turned on and the Pareto front must first be reached. Note that statement~(1) in Lemma~\ref{lem:even-cover} is a generalization of Lemma~\ref{lem:Reference-Points}, and statements (2) and (3) are false when stochastic population update is turned on, since the cover number of a Pareto-optimal fitness vector $v$ can be increased by replicating individuals, which may survive when they are chosen according to Line~\ref{line:update} in Algorithm~\ref{alg:nsga-iii} uniformly at random when stochastic population update is enabled. 

\begin{lemma}
\label{lem:even-cover}
    Assume the same conditions as in Lemma~\ref{lem:Reference-Points}. In Cases~(2) and~(3), stochastic population update is disabled ($a=0$), while in Case~(1) it may be also enabled ($a=1$). Under these assumptions, the following properties hold.
    \begin{enumerate}
        \item[(1)] Let $v \in \mathbb{N}_0^d$ be a fitness vector which is non-dominated in iteration $t$ and $t+1$, and let $0 \leq b \leq \mu/((1+a)|S_d|)$. If $c_t(v) \geq b$ then also $c_{t+1}(v) \geq b$.
        \item[(2)] Let $v \in \mathbb{N}_0^d$ be a fitness vector which is non-dominated in iteration $t$ and $t+1$, and suppose that $c_{t+1}(v) < c_t(v)$. Then $c_{t+1}(w) \leq c_t(v)$ for every $w \in \mathbb{N}_0^d$, which is non-dominated in iteration $t+1$.
        \item[(3)] Suppose that every $x\in P_t$ is Pareto optimal. Then $m_t:=\max\{c_t(v) \mid v \in \mathbb{N}_0^d \text{ is non-dominated}\}$ does not increase.
    \end{enumerate}
\end{lemma}

\begin{proof}
    (1): Note that all individuals $x \in P_t$ with $f(x)=v$ are either in $(P_t \cup Q_t) \setminus R_t$ (if stochastic population update is enabled), or in the first layer $F_t^1$. Suppose that there are $\ell \in \{0\} \cup [b]$ such individuals $x$ in $(P_t \cup Q_t) \setminus R_t$. Note that $\ell>0$ is only possible if stochastic population update is enabled. We show that $b-\ell$ such individuals will be selected from $F_t^1$ for $P_{t+1}$. The \nsgaIII algorithm iterates through all reference points, always preferring a reference point $r$ with the fewest associated individuals chosen for $P_{t+1}$ so far (see Line~\ref{line:minimal} in Algorithm~\ref{alg:Survival-Selection}), and selecting an individual for $P_{t+1}$ associated to $r$. By Lemma~\ref{lem:Same-Reference}, two non-dominated individuals with distinct fitness values are associated with different reference points, and those with the same fitness to the same.
    Note that in case of stochastic population update, \nsgaIII selects at least $\lceil{\mu/2}\rceil$ individuals for $P_t$ via non-dominated sorting and Algorithm~\ref{alg:Survival-Selection} and hence, iterates at least $\lfloor{\mu/(2|S_d|)}\rfloor$ times through all reference points. When it is turned off, it iterates at least $\lfloor{\mu/|S_d|}\rfloor$ times through all reference points since it chooses $\mu$ individuals for $P_t$. Hence, the cover number $c_{t+1}(v)$ of $v$ with respect to $P_{t+1}$ is still at least $b \leq \mu/((1+a)|S_d|)$.
    
    (2): Note that $|F_t^1| > \mu$. Otherwise, $c_t(v)$ cannot decrease. There is a reference point $r$ to that all $x$ with $f(x)=v$ are associated. Hence, \nsgaIII iterates at most $c_t(v)-1$ times through $r$, since $c_{t+1}(v)<c_t(v)$, and consequently, through all reference points at most $c_t(v)$ times (see Line~\ref{line:minimal} in Algorithm~\ref{alg:Survival-Selection}). Thus, for every non-dominated $w \in \mathbb{N}_0^d$, at most $c_t(v)$ individuals $x$ with $f(x)=w$ survive.
    
    (3): By iterating through all reference points at most $m_t$ times, the algorithm finds $\mu$ individuals in $F_t^1$ (since all Pareto optimal individuals are in $F_t^1$) and the survival selection ends. Then $c_{t+1}(w) \leq m_t$ for every vector $w$ on the Pareto front.
\end{proof}

The next lemma builds on Lemma~\ref{lem:even-cover}(1) and provides two very general bounds on the number of generations required to create $\beta \leq \lfloor \mu/((1+a)|S_d|) \rfloor$ individuals with a certain fitness value, given that there are already $\gamma < \beta$ individuals with that value in a generation $t_0$. The first bound is obtained by repeatedly replicating individuals with that value, whereas the second allows for higher replication rates using Chernoff bounds, but requires that a certain number of individuals with that value is already present in the population. We formulate these results not only for a single generation $t_0$, but for multiple generations $t_1, \ldots, t_m$. Specifically, we consider the time required to generate $\beta$ individuals $x_i$ with fitness vector $v_i$, starting from generation $t_i$ for each $i \in [m]$, and consider the sum of all these times. Note that these times may overlap, but they do not necessarily do so. We will use this result later in our analyses, where we consider suitably replicating intermediate non-nominated individuals to facilitate the creation of dominating solutions or the discovery of the Pareto front.

\begin{lemma}
\label{lem:clone-general}
Assume the same conditions as in Lemma~\ref{lem:Reference-Points}. Let $v \in \mathbb{N}_0^d$ and suppose there is $m \in \mathbb{N}$, $t_1, \ldots, t_m$ and $x_i \in P_{t_i}$ with $f(x)=v_i$ for $i \in [m]$. Let $\beta$ and $1 \leq \gamma < \beta$ be natural numbers, where $1 < \beta \leq \lfloor{\mu / ((1+a)|S_d|)}\rfloor$. Suppose that in generation $t_i$ there are at least $\gamma$ individuals with fitness vector $v_i$. For $i \in [m]$ denote by $X_i$ the number of generations from $t_i$ on until at least $\beta$ search points with the same fitness vector $v_i$ are generated, or an individual $y$ dominating $x_i$ is created. Let $X:=\sum_{i=1}^m X_i$. Then the following holds for any parameter $\delta>0$.
\begin{itemize}
\item[(1)] We have $\Pr(X \geq \lceil{20 (m+\delta)(\beta-\gamma)/\gamma}\rceil) = e^{-4(m+\delta)(\beta-\gamma)}$.

\item[(2)] We have $\Pr(X \geq m\ln(\beta)) = m\ln(\beta)e^{-\gamma/72}$. 

\end{itemize}
\end{lemma}

\begin{proof}
(1): For $i \in [m]$ and $j \in \{\gamma, \ldots ,\beta-1\}$ let $Y_{i,j}$ be a random variable that counts the number of trials with $c_t(v_i)=j$. Then, the number of trials required until the cover number of $v$ is at least $\beta$, or an individual $y$ that dominates $x$ is created, is at most $Y:=\sum_{j=\gamma}^{\beta-1} Y_j$, since $c_t(v)$ cannot decrease by Lemma~\ref{lem:even-cover} as long as no individual dominating $x$ is created. In one trial, $c_t(v_i)$ can be increased by choosing an individual $y$ with $f(x)=f(y)$ as parent and flipping no bits (prob. $\gamma/\mu \cdot (1-1/n)^n \geq \gamma/(4 \mu)=:p_\gamma$ due to $(1-1/n)^n \geq 1/4$).
Hence, $Y_i$ is stochastically dominated by an independent sum $Z_i:=\sum_{j=\gamma}^{\beta-1}Z_{i,j}$ of geometrically distributed random variables $Z_{i,j}$ with success probability $p_\gamma$. Note that $\expect{Y} \leq \expect{Z} \leq m(\beta-\gamma)/p_\gamma = 4 \mu m (\beta - \gamma) / \gamma$ and hence, by Lemma~\ref{lem:Doerr-dominance}, we obtain for $s:= m(\beta-\gamma)/p_\gamma^2 = 16 m \mu^2(\beta-\gamma)/\gamma^2$, and $\lambda \geq 0$
\[
\Pr(Z \geq \expect{Z} + \lambda) \leq \exp\left(-\frac{1}{4} \min\left\{\frac{\lambda^2}{s}, \lambda p_\gamma \right\}\right)
\]
and for $\lambda=4(m+\delta)(\beta-\gamma)/p_\gamma = 16\mu (m+\delta)(\beta-\gamma)/\gamma$ we obtain $\Pr(Y \geq 20\mu (m+\delta)(\beta-\gamma)/\gamma) \leq  \Pr(Z \geq \expect{Z} + 16\mu (m+\delta)(\beta-\gamma)/\gamma) = e^{-4(m+\delta)(\beta-\gamma)}$. Hence, $\Pr(Y \geq 20\mu (m+\delta)(\beta-\gamma)/\gamma) \leq  \Pr(Z \geq \expect{Z} + 16\mu (m+\delta)(\beta-\gamma)/\gamma) = e^{-4(m+\delta)(\beta-\gamma)}$, and therefore
\[
 \Pr(X \geq \lceil{20 (m+\delta)(\beta-\gamma)/\gamma}\rceil) \leq \Pr(Y \geq 20\mu (m+\delta)(\beta-\gamma)/\gamma) \leq e^{-4(m+\delta)(\beta-\gamma)}
\]
since a generation consists of $\mu$ trials. 

(2): For $t \geq t_i$ let $A_{i,t}$ be the number of individuals $x$ with $f(x)=v_i$. Denote by $N_{i,t}$ the number of new created individuals of this form. Then $\expect{N_{i,t}} \geq A_{i,t}/4$ since in one trial such an individual is cloned with probability at least $A_{i,t}/(4 \mu)$ (with prob. at least $A_{i,t}/\mu$ one such individual is selected as parent and no bit is changed with prob. $(1-1/n)^n \geq 1/4$ during mutation) and a generation consists of $\mu$ trials. By a classical Chernoff bound $\Pr(N_{i,t} \leq 2\expect{N_{i,t}}/3) \leq e^{-A_{i,t}/72} \leq e^{-\gamma/72}$. Hence, with probability $1-e^{-\gamma/72}$ we have that $A_{i,t+1} \geq \min\{A_{i,t}+A_{i,t}/6,\beta\} = \min\{7A_{i,t}/6,\beta\}$ or, in other words, $A_{i,t}$ increases by a factor of at least $7/6$ if the value $\beta$ is not reached. Since $\gamma \cdot (7/6)^{\ln(\beta/\gamma)/\ln(7/6)} = \beta$, and $\ln(\beta/\gamma)/\ln(7/6) \leq \ln(\beta)$, we see that $\ln(\beta)$ such generations in a row are sufficient to obtain a cover number of $v_i$ of at least $\beta$ (since $c_t(v_i)$ cannot decrease due to $\beta \leq \lceil{\mu/((1+a)|S_d|)\rceil}$) or an individual $y$ is created dominating $x_i$. Hence, $\Pr(X_i \geq \ln(\beta)) = \ln(\beta)e^{-\gamma/72}$ and $\Pr(X \geq m\ln(\beta)) = m\ln(\beta)e^{-\gamma/72}$ by a union bound.
\end{proof}

We can also derive tail bounds for the time required until, for every fitness vector $v \in V$ of a subset $V$ of the Pareto front that is already covered, at least $\beta \leq \lfloor{\mu/((1+a)|S_d|)}\rfloor$ search points with fitness $v$ have been generated. This differs substantially from the previous result, where each vector $v_i$ may first need to be discovered before $\beta$ individuals with fitness $v_i$ can be generated and hence, the single times may not overlap. Here, we can increase the cover number of all vectors $v \in V$ in parallel. Hence, the first result above is more suitable for describing hill-climbing processes or the gradual coverage of the Pareto front, whereas the second result below captures the spread of solutions over an already covered subset of the Pareto front. 

In the latter setting, when already a set $V$ on the Pareto front is covered, the cover number of all $v \in V$ can also be bounded from above, which is particularly useful for deriving lower bounds on the runtime of \nsgaIII on \mOJZJ for $d \in \{2,4\}$ (see Section~6 below) or for $2$-\OMM (compare with~\citep{Opris26PopDyn}). Notably, when $|V|$ coincides with the size of the Pareto front, this corresponds to the optimal possible distribution of $P_t$ across the entire Pareto front. We also show that, once such a distribution is achieved, it is maintained in all subsequent generations. However, this cannot be transferred to the case where stochastic population update is enabled, since randomly chosen solutions also survive. We formulate this result for arbitrary fitness functions. Importantly, it holds for subsets $V$ of arbitrary size, but only under the small restriction $\mu \in 2^{O(n)}$.

\begin{lemma}
\label{lem:sparsity}
    Assume the same conditions as in Lemma~\ref{lem:Reference-Points}. Let $P_t$ be its current population and denote by $S_d$ a maximum set of mutually incomparable solutions. Assume $\lvert{S_d}\rvert \leq \mu \in 2^{O(n)}$. Let $V$ be a subset of the Pareto front covered by $P_t$. The following properties are satisfied after $O(n)$ generations and all future ones with probability at least $1-2e^{-n/4}$ and in expectation.
    \begin{itemize}
        \item[(1)] 
        For all $v \in V$ we have $c_t(v) \geq \lfloor{\mu/((1+a)|S_d|)}\rfloor$.
        \item[(2)]
        Suppose that stochastic population update is disabled. If all populations $P_0,P_1, \ldots$ consist only of Pareto optimal individuals, then $c_t(v) \leq \lceil{\mu / |V|}\rceil$ for all $v \in V$. 
    \end{itemize}
\end{lemma}

\begin{proof}
    (1): Fix a Pareto optimal $x$ with $v:=f(x) \in V$. Let $\kappa:=\lfloor{\mu/((1+a)|S_d|)}\rfloor$. We show that $c_t(v) \geq \kappa$ with probability at least $1-2e^{-n/4}$ after $O(n)$ generations. By Lemma~\ref{lem:even-cover}(1), $c_t(v) \leq \kappa$ cannot decrease. We define two phases where the second phase only applies if $\kappa>144n$.
    
    \textbf{Phase 1.} We have $c_t(v) \geq \ell:=\min\{\kappa,144n\}$.
    
    \textbf{Phase 2.} We have $c_t(v) \geq \kappa$.
    
    We show that both phases are finished in $O(n)$ generations with probability at least $1-e^{-n}$ and in expectation. The result then follows by a union bound on both phases and every $v \in V$ since $|V| \leq \lvert{\{0,1\}^n}\rvert = 2^n$ and thus $2 \cdot 2^n e^{-n} = 2e^{-n + n\ln(2)} \leq 2e^{-n/4}$ for $n$ sufficiently large.
    
    \emph{Phase 1}: For $j \in [\ell-1]$ let $X_j$ be a random variable that counts the number of generations with $c_t=j$. Then the number of generations until the cover number of $v$ is at least $\ell$ is at most $X:=\sum_{j=1}^{\ell-1} X_j$. Note that $c_t$ can be increased by choosing an individual $y$ with $f(x)=f(y)$ as parent and flipping no bits (prob. $1/\mu \cdot (1-1/n)^n \geq 1/(4 \mu)=:\sigma_t$). Hence, the probability of increasing $c_t$ in one generation is at least 
    $$1-(1-\sigma_t)^\mu \geq \frac{\sigma_t \mu}{1+\sigma_t \mu} = \frac{1/4}{1+1/4} = \frac{1}{5}.$$   
    Hence, $X$ is stochastically dominated by an independent sum $Z:=\sum_{j=1}^{\ell-1}Z_j$ of geometrically distributed random variables $Z_j$ with success probability $1/5$. 
    Then $\expect{X} \leq \expect{Z} \leq 5\ell \leq 720n$ and hence, by Lemma~\ref{lem:Doerr-dominance}, we obtain for $s:= 25\ell$, and $\lambda \geq 0$
    \[
    \Pr(Z \geq \expect{Z} + \lambda) \leq \exp\left(-\frac{1}{4} \min\left\{\frac{\lambda^2}{s}, \frac{\lambda}{5} \right\}\right)
    \]
    and for $\lambda=120n$ we obtain $\Pr(X \geq 840n) \leq \Pr(X \geq  5\ell + 120n) \leq  \Pr(Z \geq \expect{Z} + 120n) \leq e^{-n}$.\\
    \emph{Phase 2:} We can assume that $\kappa > 144n$. Let $Y$ be the number of individuals $x$ with $f(x)=v$. Then, by Lemma~\ref{lem:clone-general}(2) applied on $\beta=\kappa$ and $m=1$, we obtain $\Pr(Y \geq \ln(\kappa)) = \ln(\kappa)e^{-2n} \leq e^{-n}$.
    
    The bound on the expected number of generations in each phase follows by applying the same arguments for an additional period of $O(n)$ generations and by the fact that $(1+o(1))O(n) = O(n)$ such periods are sufficient.
    
    (2): With the same argument as in (1) we obtain with probability $1-e^{-\Omega(n)}$ that, after $O(n)$ generations, the cover number of all $v \in V$ is at least $\lceil{\mu/|V|}\rceil$ or one of these cover numbers has decreased at least one time when it was at most $\lceil{\mu/|V|}\rceil$. Suppose that the former happens. If $\lceil{\mu/|V|}\rceil>\mu/|V|$ we see that $\lceil{\mu/|V|}\rceil \cdot |V| > |V| \cdot \mu/|V| = \mu$, a contradiction. Hence, $\lceil{\mu/|V|}\rceil=\mu/|V|$ and all $\mu$ individuals are completely evenly distributed on $V$. If the latter happens, we see with Lemma~\ref{lem:even-cover}(2) that the cover number of all Pareto optimal vectors $v$ is at most $\lceil{\mu/|V|}\rceil$. 
    By Lemma~\ref{lem:even-cover}(3) the maximum cover number cannot increase, proving Lemma~\ref{lem:sparsity}.
\end{proof}

\section{Improved Runtime Guarantees of NSGA-III on Classical Benchmark Problems}

In this section, we generalize the runtime results from~\citep{OprisNSGAIII} to any number of objectives, and also improve them for population sizes asymptotically larger than the size of the Pareto front by considering the structural results from Lemma~\ref{lem:even-cover}. The important difference to previous works is that we consider the population dynamics of \nsgaIII, particularly that the cover number $c_t(v) \leq \lfloor{\mu/((1+a)|S_d|)}\rfloor$ of a non-dominated fitness vector $v$ cannot decrease. For the bi-objective case, particularly for SPEA2, there are already first results in~\citep{doerr2026improvedruntimeguaranteesspea2} considering such mechanisms.

\subsection{The Many-Objective LeadingOnesTrailingZeros Problem}
In this section, we derive the first improved result on the runtime of \nsgaIII on the $\mLOTZ$ benchmark. In $\mLOTZ(x)$, $x$ is divided in $d/2$ many blocks and in each block we count the $\LO$-value (the length of the longest prefix containing only ones in this block) and the $\TZ$-value (the length of the longest suffix containing only zeros in this block). This can be formalized as follows.
\begin{definition}[\citet{Laumanns2004}]
\label{def:mLOTZ}
Let $d$ be divisible by $2$ and let the problem size be a multiple of $d/2$. Then the $d$-objective function \mLOTZ is defined by
$\mLOTZ: \{0,1\}^n \to \mathbb{N}_0^d$ as 
\[
\mLOTZ(x) = (f_1(x), f_2(x), \ldots ,f_d(x))
\]
with 
\[
f_j(x)=
\begin{cases}
    \sum_{i=1}^{2n/d} \prod_{\ell=1}^i x_{\ell+n(j-1)/d}, & \text{ if $j$ is odd,} \\
    \sum_{i=1}^{2n/d} \prod_{\ell=i}^{2n/d} (1-x_{\ell+n(j-2)/d}), & \text{ else,}
\end{cases}
\]
for all $x=(x_1, \ldots ,x_n) \in \{0,1\}^n$.
\end{definition}

A Pareto optimal set of \mLOTZ is
\[
\{1^{j_1}0^{2n/d-j_1} \ldots 1^{j_{d/2}}0^{2n/d-j_{d/2}} \mid j_1, \ldots , j_{d/2} \in \{0\} \cup [2n/d]\}
\]
which coincides with the set of Pareto optimal search points of \mLOTZ. The cardinality of this set is $(2n/d+1)^{d/2}$. With Lemma~\ref{lem:Reference-Points} we see that mutually incomparable solutions from the first rank are not being lost in future generations if the normalized objective vector is in the unit cube, the number of reference points is sufficiently high and the population size is at least the cardinality of a set $S_d$ of mutually incomparable solutions of maximum size if stochastic population update is disabled. But even if it is enabled, the same effect shows up if the population is at least twice the cardinality of $|S_d|$. 

This is similar to the NSGA-II algorithm, for which previous work showed that the population size must be chosen large enough to guarantee the survival of all mutually incomparable solutions~\citep{DANG2024104098,DaOp2023,DANG2024104098,Doerr2023,ZhengLuiDoerrAAAI22,ZhengD2023}. 

To bound the expected runtime on \mLOTZ, the cardinality of $S_d$ can be bounded as follows. We remark that this cardinality can be much larger than the size of the Pareto front, which is $(2n/d+1)^{d/2}$. Lemma~3.2 from~\citep{Opris2025PAES} provides the following.

\begin{lemma}
\label{lem:fitnessvectors-non-dom-LOTZ}
    Let $S_d$ be a maximum set of mutually incomparable solutions for $f:=\mLOTZ$. Then $|S_d| = n+1$ if $d=2$ and 
    $$\frac{(2n/d + 1)^{d-1}}{4(d-2)^{d/2-1}} \leq \lvert{S_d}\rvert \leq (2n/d+1)^{d-1}$$
    if $d \geq 4$.
\end{lemma}

Note that for a constant number $d$ of objectives, the inequality in Lemma~\ref{lem:fitnessvectors-non-dom-LOTZ} is asymptotically tight. We are now in a position to show that \nsgaIII efficiently finds the Pareto front of $\mLOTZ$, also incorporating our results on the cover number. For rather of completeness, we state the result below for the case where stochastic population update is enabled. The main tools in our analysis are Lemma~\ref{lem:Reference-Points} and Lemma~\ref{lem:even-cover}. Recall that these require $\mu \geq (1+a)|S_d|$, where $a=0$ if stochastic population update is disabled and $a=1$ otherwise.

\begin{theorem}
\label{thm:Runtime-Analysis-NSGA-III-mLOTZ}
Consider \nsgaIII with or without stochastic population update optimizing $f:=\mLOTZ$ with $\varepsilon_{\text{nad}} \geq 2n/d$ and a set $\refer$ of reference points as defined above for $p \in \mathbb{N}$ with $p \geq 4n\sqrt{d}$, and a population size $\mu \geq (1+a)|S_d|$. For every initial population, the whole Pareto front of \mLOTZ is found in $O(n \ln(\min\{\mu/|S_d|,n\}) + n^2|S_d|/\mu)$ generations with probability $1-e^{-\Omega(n)}$ and in expectation. 
\end{theorem}

\begin{proof}
    By Lemma~\ref{lem:fitnessvectors-non-dom-LOTZ} the condition on the population size $\mu$ in Lemma~\ref{lem:Reference-Points} is always met. Along with $f_{\max}=2n/d$ and $p \geq 4n\sqrt{d} = 2d^{3/2}f_{\max}$, Lemma~\ref{lem:Reference-Points} and Lemma~\ref{lem:even-cover} are applicable in every generation $t$. We use the method of typical runs and divide the optimization procedure into two phases. We  show that with probability $1-e^{-\Omega(n)}$ each phase is completed in $O(n \ln(\min\{\mu/|S_d|,n\}) + n^2 |S_d|/\mu)$ generations with probability $1-e^{-\Omega(n)}$ and in expectation. The final bound then follows by applying a union bound over both phases. To bound the expectation, suppose that the Pareto front is not covered after this number of generations. In that case, we restart the argument and apply the same reasoning to a subsequent period of the same length, since the analyses below hold for any initialization of search points. The expected number of periods required is $1 + e^{-\Omega(n)} = 1+o(1)$.
    
    \textbf{Phase $1$:} Create a Pareto optimal search point.\\
    For $x \in P_t$ define the vector $w(x) \in (\{0\} \cup [2n/d])^{d/2}$ by $w(x)_j:=f_{2j-1}(x)+f_{2j}(x)$. Note that $w(x)_j = 2n/d$ for every $j \in \{1, \ldots , n\}$ if and only if $x$ is Pareto optimal. 
    Set $g_t:=\max_{x \in P_t}\sum_{j=1}^{d/2} w(x)_j \in \{0\} \cup [n]$ which is the maximum possible sum of all leading ones and trailing zeros across all blocks $j \in [d/2]$. Then there is a Pareto optimal solution if and only if $g_t=n$. Since a search point $x$ can only be dominated by a search point $y$ if $\sum_{j=1}^{d/2} w(x)_j<\sum_{j=1}^{d/2} w(y)_j$, $g_t$ cannot decrease by Lemma~\ref{lem:Reference-Points}. 
    For $\ell \in \{0\} \cup [n-1]$ we define the random variable $X_\ell^{\text{climb}}$ as the number of generations $t$ with $g_t=\ell$. Then the number of generations $X$ until there is a Pareto optimal solution is at most $X^{\text{climb}}:=\sum_{\ell=0}^{n-1} X_\ell^{\text{climb}}$. For $X_\ell^{\text{climb}}$ we first estimate the number of generations $X_\ell^{\text{clone}}$ to create at least $\alpha:=\min\{\mu/((1+a)|S_d|),n\}$ individuals $x$ with $g(x)=\ell$, or to increase $g_t$ (which also includes the event that an individual $y$ is created which dominates $x$). Note that these number cannot fall below $\alpha$ by Lemma~\ref{lem:even-cover}(1). Second, we estimate the number of generations $X_\ell^{\text{mut}}$ to increase $g_t$ by flipping one specific bit in one of those $\min\{\mu/((1+a)|S_d|),n\}$ individuals, while leaving all other bits the same, to increase the $\LO(x^j)$ or $\TZ(x^j)$ value in a block $j \in [d/2]$ with $\LO(x^j) + \TZ(x^j) < 2n/d$. Note that $X^{\text{climb}}$ is stochastically dominated by $\sum_{\ell=0}^{n-1} (X_\ell^{\text{mut}}+X_\ell^{\text{clone}})$. For both random variables $X^{\text{clone}}:=\sum_{\ell=0}^{\alpha-1} X_\ell^{\text{clone}}$ and $X^{\text{mut}}:=\sum_{\ell=0}^{n-1} X_\ell^{\text{mut}}$, we derive desired tail bounds separately.
    
    We start with $X_\ell^{\text{clone}}$. For $\gamma \in \{0, \ldots , \alpha-1\} \setminus \{0\}$ denote by $X_{\ell,\gamma}^{\text{clone}}$ the random variable for the number of generations until there are exactly $\gamma$ distinct individuals $x$ with $g_t(x)=\ell$, or $g_t$ increases. Then $X_\ell^{\text{clone}}=\sum_{\gamma=1}^{\alpha-1} X_{\ell,\gamma}^{\text{clone}}$, and $X^{\text{clone}}=\sum_{\ell=0}^{n-1}\sum_{\gamma=1}^{\alpha-1} X_{\ell,\gamma}^{\text{clone}} = \sum_{\gamma=1}^{\alpha-1}\sum_{\ell=0}^{n-1} X_{\ell,\gamma}^{\text{clone}}$. By Lemma~\ref{lem:clone-general}(1) we have $\Pr(\sum_{\ell=0}^{n-1} X_{\ell,\gamma}^{\text{clone}} \geq 21n/\gamma) = e^{-4n}$ (by using $ \beta=\gamma+1$ and $\delta=0$) and hence, by a union bound over $\gamma \in [\alpha-1]$, we obtain $\Pr(\sum_{\gamma=1}^{\alpha-1} \sum_{\ell=0}^{n-1} X_{\ell,\gamma}^{\text{clone}} \geq \sum_{\gamma=1}^{\alpha-1} 21n/\gamma) \leq \alpha e^{-4n} \leq e^{-3n}$. The latter implies that $\Pr(X \geq 42n \ln(\alpha))  \leq e^{-3n}$ (due to $\sum_{\gamma=1}^{\alpha-1} 21n/\gamma \leq 42n \ln(\alpha)$).
    
    Now we estimate $X^{\text{mut}}$. To increase $g_t$ in one trial (no matter which value it currently attains), it suffices to choose an individual $y$ with $\sum_{j=1}^{d/2} w(y)_j=g_t$, and then to flip one specific bit in $y$. Note that there are at least $\min\{\mu/((1+a)|S_d|),n\}$ such $y$. Hence, this happens with probability at least $1/\mu \cdot \min\{n,\lfloor{\mu/((1+a)|S_d|)}\rfloor\} \cdot 1/n \cdot (1-1/n)^{n-1} \geq \min\{1/(e\mu),1/(4en|S_d|)\}=:\eta$ (due to $\lfloor{x}\rfloor \geq x/2$ for $x \geq 1$). Since a generation consists of $\mu$ trials, the probability for decreasing $g_t$ in one generation is at least 
    \[
    1-(1-\eta)^\mu \geq \frac{\eta \mu}{\eta \mu+1}.
    \]
    For $\mu \leq 4n |S_d|$ we obtain $\frac{\eta \mu}{\eta \mu+1} = \frac{\mu/(4en|S_d|)}{1+\mu/(4en|S_d|)} = \frac{1}{4en|S_d|/\mu+1} \geq \frac{\mu}{8en|S_d|},$ and for $\mu > 4n |S_d|$ we obtain $\frac{\eta \mu}{\eta \mu+1} \geq 1/(2e+1) \geq 1/7$. Hence, $X^{\text{mut}}$ is stochastically dominated by an independent sum $Y:=\sum_{k=0}^{n-1}Y_k$ of geometrically distributed random variables $Y_k$ with success probability $p^*$ where $p^*=\frac{\mu}{8en|S_d|}$ if $\mu \leq 4n |S_d|$, and $p^*=1/7$ otherwise. Note that $\expect{Y} = n/p^*$. By Lemma~\ref{lem:Doerr-dominance} we obtain for $s:=n/(p^*)^2$ and $\lambda \geq 0$
    \[
     \Pr(X^{\text{mut}} \geq \expect{Y} + \lambda) \leq \Pr(Y \geq \expect{Y} + \lambda) \leq \exp\left(-\frac{1}{4} \min\left\{\frac{\lambda^2}{s}, \lambda p^* \right\}\right).
    \]
    For $\lambda:= 12n/p^*$ we obtain $\Pr(X^{\text{mut}} \geq 13n/p^*) = \Pr(X^{\text{mut}} \geq n/p^* + 12n/p^*) = \Pr(X^{\text{mut}} \geq \expect{Y} + \lambda) \leq e^{-3n}$.  In total, we obtain by a union bound 
    \[
    \Pr\left(X^{\text{clone}}+X^{\text{mut}} \geq 42n \ln(\min\{\lfloor{\mu/((1+a)|S_d|)}\rfloor,n\}) + 13n/p^*\right) \leq 2e^{-3n}.
    \]
    Since $13n/p^* \leq 91n + 104en^2 |S_d|/\mu$, we obtain that after 
    $$42n \ln(\min\{\lfloor{\mu/((1+a)|S_d|)}\rfloor,n\}) +  91n + 104en^2 |S_d|/\mu = O(n \ln(\min\{\mu/|S_d|,n\}) + n^2 |S_d|/\mu)$$
    generations, a Pareto optimal point is found with probability at least $1-2e^{-3n}$.
   
    \textbf{Phase $2$:} Cover the whole Pareto front.
    
    Denote by $\mathcal{F}_d$ a set of Pareto optimal solutions, and let $y \in \mathcal{F}_d$ with $f(y)=v$. We consider the distance of~$y$ to the closest Pareto optimal search point in the population, which is $\delta_t:=\min_{x \in P_t \cap \mathcal{F}_d} \sum_{j=1}^{d/2} |f_{2j-1}(x)-f_{2j-1}(y)| =\min_{x \in P_t \cap \mathcal{F}_d} \sum_{j=1}^{d/2} |\LO(x^j)-\LO(y^j)|$. Note $0 \leq \delta_t \leq n$ (since $f_{2j}(x) = n - f_{2j-1}(x)$ is uniquely determined by $f_{2j-1}(x)$), and we have created $y$ if $\delta_t=0$. Since the population never loses all solutions with the same Pareto optimal fitness vector (by Lemma~\ref{lem:Reference-Points}), $\delta_t$ cannot increase. Further, for all $1 \leq \ell \leq n$, define the random variable $X_\ell^v$ as the number of generations $t$ with $\delta_t=\ell$. Then the number of generations until there is a solution $y$ with $f(y)=v$ is at most $X^v=\sum_{\ell=1}^n X_\ell^v$. As in Phase~$1$, we estimate $X_\ell^v$ by the number of generations to create at least $\alpha:=\min\{\lfloor{\mu/((1+a)|S_d|)}\rfloor,n\}$ individuals $x$ with $f(x)=v$, denoting by the random variable $X_\ell^{\text{clone}}$, and then we estimate the number of generations $X_\ell^{\text{mut}}$ to decrease $\delta_t=\ell$ by flipping one specific bit in one of these $\min\{\lfloor{\mu/((1+a)|S_d|)}\rfloor,n\}$ individuals as follows. In a block $r$ with $f_{2r-1}(x) < f_{2r-1}(y)$ (and hence $f_{2r}(x)>f_{2r}(y)$) we may increase the number of leading ones and in a block $r$ with $f_{2r-1}(x)>f_{2r-1}(y)$ (and hence $f_{2r}(x) < f_{2r}(y)$) we may increase the number of trailing zeros. This decreases $\delta_t$, if all other blocks are not changed. Again, $X^v$ is stochastically dominated by $\sum_{\ell=0}^{n-1} (X_\ell^{\text{clone}}+X_\ell^{\text{mut}})$. For both random variables $X^{\text{clone}}:=\sum_{\ell=0}^{n-1} X_\ell^{\text{clone}}$ and $X^{\text{mut}}:=\sum_{\ell=0}^{n-1} X_\ell^{\text{mut}}$, we derive desired tail bounds separately. As in Phase~1, we obtain $\Pr(X^{\text{clone}} \geq 42n \ln(\alpha))  \leq e^{-3n}$. To decrease $\delta_t$ in one trial, it suffices to choose an individual $x$ with value $\delta_t$ as a parent and flip one specific bit in $x$ while not changing the other ones which happens with probability 
    $$1/\mu \cdot \min\{n,\lfloor{\mu/((1+a)|S_d|)}\rfloor\} \cdot 1/n \cdot (1-1/n)^{n-1} \geq \min\{1/(e\mu),1/(4en|S_d|)\}.$$
    Hence, $X^{\text{mut}}$ is also stochastically dominated by an independent sum of $n$ geometrically distributed random variables with success probability $p^*$ where $p^*=\mu/(8en|S_d|)$ if $\mu \leq 4n |S_d|$, and $p^*=1/7$ otherwise. As in Phase~1, we obtain $\Pr(X^{\text{mut}} \geq 13n/p^*) = \Pr(X^{\text{mut}} \geq n/p^* + 12n/p^*) = \Pr(X^{\text{mut}} \geq \expect{Y} + \lambda) \leq e^{-3n}$. In total, we obtain by a union bound that $\delta_t=0$ after $O(n \ln(\min\{\mu/|S_d|,n\}) + n^2 |S_d|/\mu)$ generations with probability at most $2e^{-3n}$. By a union bound over all possible $v$, and noting that the size of the Pareto front can be bounded by the number of distinct fitness vectors of \mLOTZ which is at most $2^n$, we obtain the following. The probability is at most $(2n/d+1)^{d/2} \cdot 2e^{-3n} \le 2^n \cdot 2e^{-3n} \le e^{-2n}$ for $n$ sufficiently large that the Pareto front is not completely covered after $O(n \ln(\min\{\mu/|S_d|,n\}) + n^2 |S_d|/\mu)$ generations, completing the proof. 
    \end{proof}
    
    The upper bound provided in Theorem~\ref{thm:Runtime-Analysis-NSGA-III-mLOTZ} improves upon the known bounds from~\citep{OprisNSGAIII,DoerrNearTight} for \nsgaIII, and from~\citep{ZhengLuiDoerrAAAI22} for NSGA-II, for the bi-objective \mLOTZ function, by a factor of $\Omega(\min\{\frac{\mu}{|S_d|}, \frac{n}{\ln(\min\{\frac{\mu}{|S_d|}, n\})}\})$. This factor is $\omega(1)$ for $\mu \in \omega(S_d) \cap 2^{o(n)}$, showing that our runtime bound is more tight for a large regime of objective values $d$ and population sizes $\mu$. Furthermore, \nsgaIII with a population size of $\mu = n+1$ can already optimize the bi-objective $2$-\LOTZ function in $O(n^2)$ generations with overwhelming probability. In contrast, there are only positive results for \nsga in case that the population size is at least $4n+4$ in case of $d=2$ (see~\citep{ZhengLuiDoerrAAAI22}). The reason is that \nsga uses a different mechanism for survival selection, the so called \emph{crowding distance}. We refer to~\citep{ZhengLuiDoerrAAAI22,ZhengD2023} for the details. Also, it may already fail when $d \ge 4$ for population sizes that are linear in the size of the Pareto front~\citep{InefficiencyLOTZ}.
     
\subsection{The Many-Objective OneMinMax Problem}

We now extend the results from~\citep{WiethegerD23} and~\citep{OprisNSGAIII} on the expected runtime of \nsgaIII on \mOMM for arbitrary even values of~$d$. There, $x$ is also divided in $d/2$ blocks, and we count both the numbers of ones and zeros in each block.

\begin{definition}[\citet{Zheng2023Inefficiency}]
Let $d$ be divisible by $2$ and let the problem size $n$ be a multiple of $d/2$. Then the $d$-objective function \mOMM is defined by
$\mOMM: \{0,1\}^n \to \mathbb{N}_0^d$ as 
\[
\mOMM(x) = (f_1(x), \ldots ,f_d(x))
\]
with 
\[
f_j(x)=
\begin{cases}
    \sum_{i=1}^{2n/d} x_{i+n(j-1)/d}, & \text{ if $k$ is odd,} \\
    \sum_{i=1}^{2n/d} (1-x_{i+n(j-2)/d}), & \text{ else,}
\end{cases}
\]
for all $x=(x_1, \ldots ,x_n) \in \{0,1\}^n$.
\end{definition}
Note that every search point is Pareto optimal in the case of $\mOMM$. Since for each block there are only $2n/d+1$ distinct fitness values, a Pareto optimal set has cardinality $(2n/d+1)^{d/2}$, which also coincides with the cardinality $|S_d|$ of a maximum set of mutually incomparable solutions $S_d$.

\begin{theorem}
\label{thm:Runtime-Analysis-NSGA-III-mOMM}
Consider \nsgaIII with or without stochastic population update optimizing $f:=\mOMM$ with $\varepsilon_{\text{nad}} \geq 2n/d$ and a set $\refer$ of reference points as defined above for $p \in \mathbb{N}$ with $p \geq 4n\sqrt{d}$, and a population size $\mu \geq (1+a)|S_d|$. For every initial population, the whole Pareto front of $\mOMM$ is found in $O(n\ln(\min\{\mu/|S_d|,n\}) + d|S_d| n \ln(n)/\mu)$ generations with probability at least $1-O(1/n^{3d})$ and in expectation. If additionally $d \in O(\sqrt{n}/\ln(n))$, the pareto dront of $\mOMM$ can be found in $O(n + d|S_d| n \ln(n)/\mu)$ generations with probability at least $1-O(1/n^{3d})$ and in expectation.
\end{theorem}

\begin{proof}
    Let $v$ a specific Pareto-optimal fitness vector satisfying $v_{2i-1} + v_{2i} = 2n/d$ for all $i \in [d/2]$. We upper bound the probability that a solution $y$ with $f(y) = (v_1, \ldots, v_d)$ has not been generated after $O(n \ln(\min\{\mu/((1+a)|S_d|), n\}) + d |S_d| n \ln(n)/\mu)$ generations, or after $O(n + d |S_d| n \ln(n)/\mu)$ generations if $d \in O(\sqrt{n}/\ln(n))$, by $O(1/n^{4d})$ from above. The final bound then follows by applying a union bound over all fitness vectors. In particular, we obtain that, with probability $(1+2n/d)^{d/2} \cdot O(n^{-4d}) = O(n^{3d}) = o(1)$ that there is still an uncovered fitness vector after $O(n \ln(\min\{\mu/((1+a)|S_d|), n\}) + d |S_d| n \ln(n)/\mu)$ generations. To bound the expectation, suppose that the Pareto front is not covered after this number of generations. In that case, we restart the argument and apply the same reasoning to a subsequent period of the same length, since the analyses below hold for any initialization of search points. The expected number of periods required is $1+o(1)$. 
    
    So define $\delta_t := \min_{x \in P_t} \sum_{j=1}^{d/2} \lvert f_{2j-1}(x) - f_{2j-1}(y) \rvert$, which represents the minimum total distance to $y$ taken over the odd-indexed objectives in the objective space. Note that $1 \leq \delta_t \leq n$ and that we have created an individual $y$ with $f(y)=v$ if $\delta_t=0$ (since $f_{2j}(x)=2n/d-f_{2j-1}(x)$ for all $x \in \{0,1\}^n$ and $j \in [d/2]$, and hence, each even objective is unique determined by the preceding odd one). Since the population never loses all solutions with the same Pareto optimal fitness vector (by Lemma~\ref{lem:Reference-Points}), $\delta_t$ cannot increase. As in the proof of Theorem~\ref{thm:Runtime-Analysis-NSGA-III-mLOTZ}, for all $\ell \in [n]$, define the random variable $X_\ell^v$ as the number of generations $t$ with $\delta_t=\ell$. Then the number of generations until there is a solution $y$ with $f(y)=v$ is at most $X^v=\sum_{\ell=1}^n X_\ell^v$. 
    
    First, consider the case $d \in O(\sqrt{n}/\ln(n))$. Fix a constant $\kappa>0$ such that $d \leq \kappa \sqrt{n}/\ln(n)$, or in other words, $\ln(|S_d|) = d/2 \ln(1+2n/d) \leq \kappa \sqrt{n}/(2\ln(n)) \ln(1+n) \leq \kappa \sqrt{n}$ and hence $|S_d| \leq e^{\kappa \sqrt{n}}$ for $n$ sufficiently large. We estimate $X_\ell^v$ by first considering the number of generations $X_\ell^{\text{clone}}$ required to create at least $\alpha_\ell:=\min\{\lceil{n/\ell}\rceil,\lfloor{\mu/{(1+a)|S_d|}}\rfloor\}$ individuals $x$ with $\sum_{i=1}^{d/2} |f_{2i-1}(x)-f_{2i-1}(y)|=\ell$ or to decrease $\delta_t=\ell$, and then, if $\delta_t$ has not been decreased, we estimate the number of generations $X_k^{\text{mut}}$ to decrease $\delta_t=\ell$ by mutation. The latter can be achieved by flipping one of $\ell$ specific bits in one of these at least $\alpha_\ell$ individuals, while leaving all other bits unchanged. Note that, by Lemma~\ref{lem:even-cover}(1), the number of such individuals cannot fall below $\alpha_\ell$.
    For both random variables $X^{\text{clone}}:=\sum_{\ell=1}^n X_\ell^{\text{clone}}$ and $X^{\text{mut}}:=\sum_{\ell=1}^n X_\ell^{\text{mut}}$ we derive some tail bounds separately. 
    
    We start with $X^{\text{clone}}$. For a given $\ell$, the strategy is first to generate $\nu := \min\{\lceil 288 \kappa \sqrt{n} \rceil, \lfloor \mu/((1+a)|S_d|) \rfloor\}$ individuals with distance $\delta_t = \ell$ by repeatedly replicating such individuals (using Lemma~\ref{lem:clone-general}(1)). Subsequently, we let generate the remaining individuals in larger rates considering Chernoff bounds (via Lemma~\ref{lem:clone-general}(2)). For $\gamma \in \{0, \ldots , \nu-1\} \setminus \{0\}$ denote by $X_{\ell,\gamma}$ the number of generations until there are $\gamma$ individuals $x$ with $\sum_{i=1}^{d/2} |f_{2i-1}(x)-f_{2i-1}(y)|=\ell$ or $g_t$ decreases, and when there are at least $\nu$ such individuals, by $X_\ell^*$ the number of generations until there are $\alpha_\ell$ individuals with $\sum_{j=1}^{d/2} |f_{2j-1}(x)-f_{2j-1}(y)|=\ell$ or $g_t$ decreases. The latter only needs to be considered if $\nu =\lceil{288 \kappa \sqrt{n}}\rceil \leq \lceil{n/\ell}\rceil-1$, which is only satisfied if $\ell \leq \lceil{\sqrt{n}/(288 \kappa)}\rceil$. 
    So $X^{\text{clone}}=\sum_{\ell=1}^n \sum_{\gamma=1}^{\min\{\alpha_\ell,\nu\}-1} X_{\ell,\gamma} + \sum_{\ell=1}^{\lceil{\sqrt{n}/(288 \kappa)}\rceil} X_\ell^*$.
    Note that $\sum_{\ell=1}^n \sum_{\gamma=1}^{\min\{\alpha_\ell,\nu\}-1} X_{\ell,\gamma}$ is stochastically dominated by $\sum_{\gamma=1}^{\nu-1} \sum_{\ell=1}^{\lfloor{n/\gamma}\rfloor} X_{\ell,\gamma}$ (since $\ell \leq \lfloor{n/\gamma}\rfloor$ if $\gamma \leq \lceil{n/\ell}\rceil-1$). 
    For $1 \leq \gamma \leq \min\{\alpha_\ell,\nu\}-1$ consider $Y_\gamma:=\sum_{\ell=1}^{\lfloor{n/\gamma}\rfloor} X_{\ell,\gamma}$. 
    By Lemma~\ref{lem:clone-general}(1) (with $\beta=\gamma+1$, $\delta=288 \kappa^2 n/\gamma$, and $m=\lfloor{n/\gamma}\rfloor$) we obtain $\Pr(Y_\gamma \geq 5760 (\kappa^2+2) n/\gamma^2) \leq e^{-4 \cdot 288 \kappa^2 n/\gamma} \leq e^{-4 \kappa \sqrt{n}}$
    and hence, by a union bound, 
    $\Pr(\sum_{\gamma=1}^{\nu-1}Y_\gamma \geq 5760 (\kappa^2+2) n \pi^2/6) = \nu e^{-4 \kappa \sqrt{n}} \leq e^{-3 \kappa \sqrt{n}}$ for $n$ sufficiently large. Here, we used $\sum_{\gamma=1}^{\nu-1} 1/\gamma^2 \leq \sum_{\gamma=1}^\infty 1/\gamma^2 = \pi^2/6$. To estimate $\sum_{\ell=1}^{\lceil{\sqrt{n}/(288\kappa)}\rceil} X_\ell^*$ 
    we use Lemma~\ref{lem:clone-general}(2) on $X^* := \sum_{\ell=1}^{\lceil{\sqrt{n}/(288\kappa)}\rceil} X_\ell^*$, $m=\lceil{\sqrt{n}/(288\kappa)}\rceil$, 
    and $\beta = \lceil{n/\ell}\rceil \leq n$ to obtain $\Pr(X^* \geq 2\sqrt{n} \ln(n)/(288 \kappa)) \leq \sqrt{n} \ln(n)e^{-288 \kappa \sqrt{n}/72} \leq e^{-3 \kappa \sqrt{n}}$ for $n$ sufficiently large. Hence, by a union bound, we obtain that 
    $$\Pr(X^{\text{clone}} \geq 5760 (\kappa^2+2) n\pi^2/6 + 2\sqrt{n} \ln(n)/(288 \kappa)) \leq e^{-2 \kappa \sqrt{n}}$$
    for $n$ sufficiently large.
    
    Now it remains to estimate $X^{\text{mut}}=\sum_{\ell=1}^n X_\ell^{\text{mut}}$ where $X_\ell^{\text{mut}}$ denotes the number of generations to decrease $\delta_t=\ell$ when there are at least 
    $\alpha_\ell$ individuals $x$ with $\sum_{i=1}^{d/2} |f_{2i-1}(x)-f_{2i-1}(y)|=\ell$. So in a trial, it suffices to choose an individual $x$ with this distance as a parent (which happens with probability at least $\min\{n/(\ell\mu), \lfloor{\mu/((1+a)|S_d|)}\rfloor/\mu\} \geq \min\{n/(\ell\mu), 1/(4|S_d|)\}$) and flip one of $\ell$ specific bits, while keeping the other bits unchanged (which happens with probability at least $\ell/n \cdot (1-1/n)^{n-1} \geq \ell/(en)$). Since a generation consists of $\mu$ trials, the probability for decreasing $\ell=\delta_t$ in one generation is at least 
    \[
    1-(1-s_\ell)^\mu \geq \frac{s_\ell \mu}{s_\ell \mu+1}=:p_\ell
    \]
    where $s_\ell=\ell \cdot \min\{n/(\ell \mu), 1/(4|S_d|)\}/(en) = \min\{1/(e\mu), \ell/(4en|S_d|)\}$.
    Note that $1/(e\mu) \geq \ell/(4en|S_d|)$ if $\ell \in [\eta]$ for $\eta:=\lfloor{4n|S_d|/\mu}\rfloor$. Also observe for $\ell \in [\eta]$ that
    $$p_\ell = \frac{\ell\mu/(4en|S_d|)}{1+\ell\mu/(4en|S_d|)} = \frac{1}{4en|S_d|/(\ell \mu)+1} \geq \frac{\ell \mu}{8en|S_d|},$$
    and for $\ell > \eta$ or $\ell = \eta =1$ observe that $p_\ell \geq 1/(2e+1) \geq 1/7$ since $\mu>2n|S_d|$ if $\eta =1$. So we see that $Z$ is stochastically dominated by two independent sums $A:=\sum_{\ell=1}^\eta A_\ell$ and $B:=\sum_{\ell=\eta+1}^n B_\ell$ of geometrically distributed random variables with success probabilities $c\ell/\eta$ for $c:=\eta \mu/(8e n |S_d|)$ and $1/7$ respectively, where $A$ is only considered if $\eta>1$. For $A$ and $B$ we compute tail bounds separately. At first consider $B$. We obtain by Lemma~\ref{lem:Doerr-dominance} for $\lambda \geq 0$
    \[
    \Pr(B \geq \expect{B} + \lambda) \leq \exp\left(-\frac{1}{4} \min\left\{\frac{\lambda^2}{49(n-\eta)}, \frac{\lambda}{7} \right\}\right) \leq \exp\left(-\frac{1}{4} \min\left\{\frac{\lambda^2}{49n}, \frac{\lambda}{7} \right\}\right).
    \]
    Note that $\expect{B} \leq 7n$. By using $\lambda=28n$ we obtain $\Pr(B \geq 35n) = e^{-n}$. Considering $A$, we obtain for any $\delta>0$ by Lemma~\ref{lem:Doerr-dominance2}
    \begin{align*}
    \Pr\left(A \geq \frac{(1+\delta) 8e |S_d| n\ln(\eta)}{\mu}\right) = \Pr\left(A \geq \frac{(1+\delta) \eta \ln(\eta)}{c}\right) \leq \eta^{-\delta},
    \end{align*}
    and hence, for $\delta=4d \ln(n)/\ln(\eta)$ we have for $n \geq 4$ (by using $\ln(\eta) \leq \ln(4n) \leq 2\ln(n)$)
    \begin{align*}
    \Pr&\left(A \geq \frac{(16+32d) e |S_d|n\ln(n)}{\mu}\right) = \Pr\left(A \geq \frac{(2+4d) \eta \ln(n)}{c}\right) = \Pr\left(A \geq \frac{2\eta \ln(n)}{c} + \frac{4d \eta \ln(n)}{c}\right)\\
    &\leq \Pr \left(A \geq \frac{\eta \ln(\eta)}{c} + \frac{\delta \eta \ln(\eta)}{c}\right) = \Pr \left(A \geq \frac{(1+\delta) \eta \ln(\eta)}{c}\right) \leq \eta^{-4d \ln(n)/\ln(\eta)} = e^{-4d \ln(n)}.
    \end{align*}
    So by a union bound over both $A$ and $B$ we obtain by noting that $X^{\text{mut}}$ is stochastically dominated by the sum $A+B$ 
    $$\Pr \left(X^{\text{mut}} \geq 35n + \frac{(16+32d) e |S_d| n\ln(n)}{\mu}\right) \leq e^{-n} + e^{-4d \ln(n)}.$$
    So after $35n + (16+32d) e |S_d| n\ln(n)/\mu + 5760 (\kappa^2+2) n\pi^2/6 + 2 \sqrt{n} \ln(n)/(288 \kappa) = O(n + d |S_d| n \ln(n)/\mu)$ generations, the probability that $R_t$ contains no solution $x$ with $f(x) = v$ is at most $e^{-n} + e^{-4d \ln(n)} + e^{-3\kappa \sqrt{n}} = O(n^{-4d})$ by a union bound over $X^\text{clone}$ and $X^\text{mut}$. 
    
    Now we consider the general case. Let $\alpha:=\min\{\lfloor{\mu/((1+a)|S_d|)}\rfloor,n\}$. We estimate $X_\ell^{\text{clone}}$ by the number of generations until $\alpha$ individuals $x$ with $\sum_{i=1}^{d/2} |f_{2i-1}(x)-f_{2i-1}(y)|=\ell$ are created or $g_t$ is decreased, and after this we estimate $X_\ell^{\text{mut}}$ by the number of generations to decrease $g_t=\ell$ by mutation. Let $X^{\text{clone}}=\sum_{\ell=1}^{\alpha-1}X_\ell^{\text{clone}}$ and $X^{\text{mut}}=\sum_{\ell=1}^{\alpha-1}X_\ell^{\text{mut}}$. For $X^{\text{clone}}$ we can derive similar tail bounds as in the proof of Theorem~\ref{thm:Runtime-Analysis-NSGA-III-mLOTZ} as $\Pr(X^{\text{clone}} \geq 42n \ln(\alpha))  \leq e^{-3n}$. The probability for decreasing $\ell=\delta_t$ in one generation is at least 
    \[
    1-(1-s_\ell)^\mu \geq \frac{s_\ell \mu}{s_\ell \mu+1}=:p_\ell
    \]
    where $s_\ell=\ell \cdot \min\{n/(\ell\mu), 1/(4|S_d|)\}/(en) = \min\{1/(e\mu), \ell/(4en|S_d|)\}$. The reason is that there are at least $\min\{\lceil{n/\ell}\rceil,\lfloor{\mu/((1+a)|S_d|)}\rfloor\}$ individuals with $\sum_{i=1}^{d/2} |f_{2i-1}(x)-f_{2i-1}(y)|=\ell$, as in the case above. Hence, we can also estimate $X^{\text{mut}}$ by 
    $$\Pr(X^{\text{mut}} \geq 35n + (16+32d)e |S_d| n\ln(n)/\mu) = e^{-n} + e^{-4d \ln(n)}.$$
    This proves the theorem.  
    \end{proof}

Hence, \nsgaIII with a population size of $n+1$ (coinciding with the size of the Pareto front) is also able to optimize the 2-\OMM benchmark in expected polynomial time. This cannot be achieved with the \nsga: In~\citep{ZhengLuiDoerrAAAI22} it is shown that the expected number of generations is $\exp(\Omega(n))$ until the population covers the whole Pareto front. 

\subsection{The Many-Objective CountingOnesCountingZeros Function}

In \mCOCZ, the search point $x$ is also divided into two halves and the second half is further divided into $d/2$ many blocks of equal size $n/d$. In each block we maximize both the number of ones and the number of zeros, i.e. the objectives are \emph{conflicting} there. In the first part we only maximize the number of ones which goes equally into each objective. 

\begin{definition}[\citet{Laumanns2004}]
Let $d$ be divisible by $2$ and let the problem size be a multiple of $d$. Then the $d$-objective function \mCOCZ is defined by
$\mCOCZ: \{0,1\}^n \to \mathbb{N}_0^d$ as 
\[
\mCOCZ(x) = (f_1(x), \ldots ,f_d(x))
\]
with 
\[
f_j(x)= \sum_{i=1}^{n/2}x_i+
\begin{cases}
    \sum_{i=1}^{n/d} x_{i+n/2+(j-1)n/(2d)}, & \text{ if $j$ is odd,} \\
    \sum_{i=1}^{n/d} \left(1-x_{i+n/2+(j-2)n/(2d)}\right), & \text{ otherwise,}
\end{cases}
\]
for all $x=(x_1, \ldots ,x_n) \in \{0,1\}^n$.
\end{definition}

Note that a search point $x$ is Pareto optimal with respect to \mCOCZ if and only if $x$ has $n/2$ many ones in the first half. Since for each block there are only $n/d+1$ distinct fitness values, a Pareto optimal set has cardinality $(n/d+1)^{d/2}$. Hence, not every search point is Pareto optimal with respect to \mCOCZ. However, in contrast to \mLOTZ we can make a precise statement about the cardinality of a maximum set of mutually incomparable solutions: It coincides with the size of the Pareto front. Note that the size of the Pareto front of \mCOCZ is smaller than that of \mOMM and \mLOTZ. Hence, a smaller population size for \nsgaIII suffices to optimize \mCOCZ accordingly. 

\begin{lemma}
\label{lem:Non-Dominated-Solutions-mCOTZ}
Let $\mathcal{F}_d$ be a Pareto optimal set of $f:=\mCOCZ$ consisting of mutually incomparable solutions. Then $\mathcal{F}_d$ is a set of mutually incomparable solutions with maximum cardinality for $f:=\mCOCZ$ with $\lvert{\mathcal{F}_d}\rvert=(n/d+1)^{d/2}$. 
\end{lemma}

\begin{proof}
    Let $S_d$ be a set of mutually incomparable solutions. We show that $\lvert{S_d}\rvert \leq (n/d+1)^{d/2}$. For a maximum set $M$ of solutions with distinct fitness vectors (i.e. $x,y \in M$ may be comparable, but $f(x) \neq f(y)$ if $x \neq y$) we have that $\lvert{M}\rvert = \lvert\{f(x) \mid x \in M\}\rvert = (n/2+1)(n/d+1)^{d/2}$ since there are $(n/2+1)(n/d+1)^{d/2}$ distinct fitness vectors: For a fixed number of ones in the first half of $x$ there are $(n/d+1)^{d/2}$ distinct solutions (which have all the same $f_1(x) + \ldots + f_d(x)$) and two search points $x_1,x_2$ with a distinct number of ones in the first half have also distinct fitness vectors since $f_1(x_1) + \ldots + f_d(x_1) \neq f_1(x_2) + \ldots + f_d(x_2)$.
    In the following we partition $M$ into $(n/d+1)^{d/2}$ many sets (of equal cardinality) such that every two search points $x,y$ in the same set are comparable. Let
    \[
    W:=\{w \in (\{0\} \cup [n/d])^d \mid w_{2i-1}+w_{2i}=n/d \text{ for all } i \in [d/2]\}
    \]
    and for $w \in W$ let $M_w=\{x \in M \mid f(x) = w+\ell \cdot \vecone \text{ for } \ell \in \{0\} \cup [n/2]\}$ where $\vecone:=(1, \ldots , 1) \in \mathbb{N}^d$, reflecting the possibility of having $\ell$ ones in the first half of the bit string that contribute to all objectives. Then $M= \bigcup_{w \in W} M_w$ and two solutions $x,y \in M_w$ are comparable: Let $x=w+\ell_1 \cdot \Vec{1}$ and $y=w+\ell_2 \cdot \Vec{1}$. Then $x \succeq y$ if $\ell_1 \leq \ell_2$ and $y \succeq x$ if $\ell_2 \leq \ell_1$. Hence, for every $w \in W$ there is at most one $x \in S_d$ with $x \in M_w$, which implies $\lvert{S_d}\rvert \leq \lvert{W}\rvert = (n/d+1)^{d/2}=\lvert{\mathcal{F}_d}\rvert$.
\end{proof}

Since $f_{\max} =n/2+n/d$ for $f:=\mCOCZ$ we obtain the following result for the runtime of \nsgaIII on $\mCOCZ$ which can be derived in a very similar way as the runtime of \nsgaIII on $\mOMM$ above, with the difference, that one has to consider also two phases: one to create a Pareto optimal solution, and a second one for covering the Pareto front. We also present results for both cases where stochastic population update are enabled and disabled.

\begin{theorem}
\label{thm:Runtime-Analysis-NSGA-III-mCOCZ}
Consider \nsgaIII with or without stochastic population update optimizing $f:=\mCOCZ$ with $\varepsilon_{\text{nad}} \geq n/2+n/d$ and a set $\refer$ of reference points as defined above for $p \in \mathbb{N}$ with $p \geq 2d^{3/2}(n/2+n/d)$ and a population size $\mu \geq (1+a)|S_d|$. For every initial population, the Pareto front of $\mCOCZ$ is found in $O(n\ln(\min\{\mu/|S_d|,n\}) + d|S_d| n \ln(n)/\mu)$ generations with probability at least $1-O(1/n^{3d})$ and in expectation. If additionally $d \in O(\sqrt{n}/\ln(n))$, the Pareto front of $\mCOCZ$ is found in $O(n + d|S_d| n \ln(n)/\mu)$ generations with probability at least $1-O(1/n^{3d})$ and in expectation.
\end{theorem}

\begin{proof}
    By Lemma~\ref{lem:fitnessvectors-non-dom-LOTZ} the condition on the population size $\mu$ in Lemma~\ref{lem:Reference-Points} is always met. Along with $f_{\max}=n/2+n/d$ and $p \geq 2d^{3/2}(n/2+n/d) = 2d^{3/2}f_{\max}$, Lemma~\ref{lem:Reference-Points} is applicable in every generation $t$. We divide the optimization procedure into two phases. We show that with probability at least $1-O(1/n^{3d})$ each phase is completed in $O(n\ln(\min\{\mu/|S_d|,n\}) + d|S_d| n \ln(n)/\mu)$ generations, or $O(n + d|S_d| n \ln(n)/\mu)$ generations if additionally $d \in O(\sqrt{n}/\ln(n))$, for both cases $a=0$ and $a=1$. To bound the expectation, we restart the argument in case of a failure. The expected number of periods required is $1+o(1)$. 
    
    \textbf{Phase $1$:} Create a Pareto optimal search point. 
    
    For each generation $t$ let $g_t:=\min_{x \in P_t} (n/2-\sum_{i=1}^{n/2} x_i)$ the minimum number of zeros in one individual $x$ in the first half of $x$. Note that $0 \leq g_t \leq n/2$ and that we have created a Pareto optimal search point $y$ if $g_t=0$. Since the population never loses non-dominated solutions (by Lemma~\ref{lem:Reference-Points}), $g_t$ cannot increase (as a solution $x$ with a maximum number of ones in the first half has a maximum sum in all objectives and is therefore non-dominated). As in the proof of Theorem~\ref{thm:Runtime-Analysis-NSGA-III-mOMM}, for all $\ell \in [n/2]$, define the random variable $X_\ell^{\text{climb}}$ as the number of generations $t$ with $g_t=\ell$. Then the number of generations until there is a solution $y$ with $f(y)=v$ is $X^{\text{climb}}=\sum_{\ell=1}^{n/2} X_\ell^{\text{climb}}$.
    
    At first we consider the case $d \in O(\sqrt{n}/\ln(n))$. Particularly, we fix a constant $\kappa>0$ with $d \leq \kappa \sqrt{n}/\ln(n)$ which also implies $\ln(|S_d|) = d/2 \ln(1+n/d) \leq \kappa \sqrt{n}$ for $n$ sufficiently large. Similar as in the proof of Theorem~\ref{thm:Runtime-Analysis-NSGA-III-mOMM} for the \mOMM case we estimate $X_\ell^{\text{climb}}$ by the number of generations $X_\ell^{\text{clone}}$ to create at least $\alpha_\ell:=\min\{\lceil{n/\ell}\rceil,\lfloor{\mu/((1+a)|S_d|)}\rfloor\}$ individuals $x$ with $g_t(x)=\ell$ or to decrease $g_t=\ell$, and then, if $g_t$ has not been decreased, we estimate the number of generations $X_\ell^{\text{mut}}$ to decrease $g_t=\ell$ by flipping one of $\ell$ specific bits in one of these $\alpha_\ell$ individuals (a one-bit in the first half), while not changing any other one. As in the proof of Theorem~\ref{thm:Runtime-Analysis-NSGA-III-mOMM}, we obtain that $\Pr(X^{\text{clone}} \geq 5760 (\kappa^2+2) n \pi^2/6 + 2 \sqrt{n} \ln(n)/(288 \kappa)) \leq e^{-2\kappa \sqrt{n}}$. The slight difference from the proof of Theorem~\ref{thm:Runtime-Analysis-NSGA-III-mOMM} is that we only need to consider the first half of the bit string instead of the entire string and hence, consider only the sum up to $n/2$. Now it remains to estimate $X^{\text{mut}}=\sum_{\ell=1}^{n/2} X_\ell^{\text{mut}}$ where $X_\ell^{\text{mut}}$ denotes the number of generations to decrease $g_t=\ell$ when there are at least $\min\{\lceil{n/\ell}\rceil,\lfloor{\mu/((1+a)|S_d|)}\rfloor\}$ individuals $x$ with $g_t(x)=\ell$. So in a trial, it suffices to choose an individual $x$ with this distance as a parent (which happens with probability at least $\min\{n/(\ell\mu), \lfloor{\mu/((1+a)|S_d|)}\rfloor/\mu\} \geq \min\{n/(\ell\mu), 1/(4|S_d|)\}$) and flip one of $\ell$ specific bits, while keeping the other bits unchanged (which happens with probability at least $\ell/n \cdot (1-1/n)^{n-1} \geq \ell/(en)$). Since a generation consists of $\mu$ trials, the probability for decreasing $\ell=\delta_t$ in one generation is at least 
    $$p_\ell = \frac{\ell \mu/(4en|S_d|)}{1+\ell \mu/(4en|S_d|)} = \frac{1}{4en|S_d|/(\ell \mu)+1} \geq \frac{\ell \mu}{8en|S_d|}.$$
    As in the proof of Theorem~\ref{thm:Runtime-Analysis-NSGA-III-mOMM}, for $\eta:=\lfloor{4n|S_d|/\mu}\rfloor$ we see that $X^{\text{mut}}$ is stochastically dominated by two independent sums $A:=\sum_{k=1}^\eta A_k$ and $B:=\sum_{k=\eta+1}^n B_k$ of geometrically distributed random variables with success probabilities $c\ell/\eta$ for $c:=\eta \mu/(8en |S_d|)$ and $1/7$ respectively where we only consider the first if $\eta>1$. 
    As in the proof of Theorem~\ref{thm:Runtime-Analysis-NSGA-III-mOMM}, we obtain
    $$\Pr \left(X^{\text{mut}} \geq 35n + \frac{(16+32d) e |S_d| n\ln(n)}{\mu}\right) \leq e^{-n} + e^{-4d \ln(n)}.$$
    So after $35n + (16+32d) e |S_d| n\ln(n)/\mu + 5760 (\kappa^2+2) n \pi^2/6 + 2 \sqrt{n} \ln(n)/(288 \kappa) = O(n + d |S_d| n \ln(n)/\mu)$ generations the probability that $R_t$ contains no Pareto optimal solution $x$ is at most $e^{-n} + e^{-4d \ln(n)} + e^{-2\kappa} = O(n^{-4d})$.
    
   Now consider the general case. Consider $\alpha:=\min\{\lfloor{\mu/((1+a)|S_d|)}\rfloor,n\}$, and bound $X_k^{\text{clone}}$ by the number of generations until $\alpha$ individuals $x$ with $g_t(x)=\ell$ are created, and finally estimate $X_\ell^{\text{mut}}$ by the number of generations to decrease $g_t=\ell$. Let $X^{\text{clone}}=\sum_{\ell=1}^{\alpha-1}X_\ell^{\text{clone}}$. For this sum $X^{\text{clone}}$ we can derive similar tail bounds as in the proof of Theorem~\ref{thm:Runtime-Analysis-NSGA-III-mLOTZ} as $\Pr(X^{\text{clone}} \geq 42n \ln(\alpha))  \leq e^{-3n}$. As before, we can also estimate $X^{\text{mut}}$ by 
   $$\Pr \left(Z \geq 35n + \frac{(16+32d) e |S_d| n\ln(n)}{\mu} \right) \leq e^{-n} + e^{-4d \ln(n)}.$$
   So after $35n + (16+32d) e |S_d| n\ln(n)/\mu+42 n \ln(\min\{\lfloor{\mu/((1+a)|S_d|)}\rfloor,n\}) = O(n \ln(\min\{\mu/|S_d|,n\})+d |S_d| n\ln(n)/\mu)$ generations the probability that $P_t$ contains no Pareto optimal solution $x$ is at most $e^{-n} + e^{-4d \ln(n)} = O(n^{-4d})$.

   \textbf{Phase $2$:} Cover the whole Pareto front.
   
   Denote by $\mathcal{F}_d$ a set of Pareto optimal solutions, and let $y \in \mathcal{F}_d$ with $f(y)=v$. As in the proof of Theorem~\ref{thm:Runtime-Analysis-NSGA-III-mOMM}, for each generation $t$ let $\delta_t:=\min_{x \in P_t \cap \mathcal{F}_d} \sum_{j=1}^{d/2} |f_{2j-1}(x)-f_{2j-1}(y)|$ the smallest possible sum of distances to $y$ with respect to the odd objectives. Note that $0 \leq \delta_t \leq n/2$ and that we have created an individual $y$ with $f(y)=v$ if $\delta_t=0$ (since $f_{2j}(x)=n/d-f_{2j-1}(x)$ for all $x \in \mathcal{F}_d$ and $j \in [d/2]$). Since the population never loses all solutions with the same Pareto optimal fitness vector (by Lemma~\ref{lem:Reference-Points}), $\delta_t$ cannot increase. As in the proof of Theorem~\ref{thm:Runtime-Analysis-NSGA-III-mLOTZ}, for all $\ell \in [n/2]$, define the random variable $X_\ell^v$ as the number of generations $t$ with $\delta_t=\ell$. Then the number of generations until there is a solution $y$ with $f(y)=v$ is at most $X^v:=\sum_{\ell=1}^n X_\ell^v$. 
   In the case $d \in O(\sqrt{n}/\ln(n))$, we can also estimate $X_\ell^v$ by the number of generations $X_\ell^{\text{clone}}$ to create at least $\alpha_\ell:=\min\{\lceil{n/\ell}\rceil,\lfloor{\mu/((1+a)|S_d|)}\rfloor\}$ individuals $x$ with $g_t(x)=\ell$ or to decrease $g_t=\ell$ and then, if $\delta_t$ has not been decreased, we estimate the number of generations $X_\ell^{\text{mut}}$ to decrease $g_t=\ell$ by flipping one of $\ell$ specific bits in one of these $\alpha_\ell$ individuals, while not changing any other one (since there are at least $\ell$ many bits contributing to $\delta_t$). As above, this leads to
   $$\Pr \left(X^v \geq 35n + \frac{(16+32d) e |S_d| n\ln(n)}{\mu} + \frac{5760(\kappa^2+2) n \pi^2}{6} + \frac{2\sqrt{n} \ln(n)}{288 \kappa} \right) \leq e^{-n} + e^{-4d \ln(n)} + e^{-2 \kappa \sqrt{n}}$$
   for the case $d \in O(\sqrt{n}/\ln(n))$ and $n$ sufficiently large, particularly if $d \leq \kappa \sqrt{n}/\ln(n)$ for the constant $\kappa>0$ from above. In the other case, we can also estimate $X_\ell^v$ by the number of generations $X_\ell$ to create at least $\alpha:=\min\{n,\lfloor{\mu/((1+a)|S_d|)}\rfloor\}$ individuals $x$ with $\delta_t(x)=\ell$, and then by the number of generations $Z_\ell$ to decrease $\delta_t=\ell$ by choosing such an individual as parent and flipping one of $\ell$ specific bits. As before, for $X:=\sum_{\ell=1}^n X_\ell$ and $Z:=\sum_{\ell=1}^n Z_\ell$ we obtain $\Pr(X \geq 42 n \ln(\min\{\lfloor{\mu/((1+a)|S_d|)}\rfloor,n\})) \leq e^{-3n}$ and $\Pr(Z \geq 35n + (16+32d) e |S_d| n\ln(n)/\mu) \leq e^{-n} + e^{-4d \ln(n)}$. By a union bound over both $X$ and $Y$ and all vectors $vv$ on the Pareto front, we obtain that after $O(n \ln(\min\{\mu/|S_d|,n\})+d |S_d| n\ln(n)/\mu)$ generations the whole Pareto front is not covered with probability at most $|S_d| \cdot O(1/n^{4d}) = O(1/n^{3d})$, completing the proof.
\end{proof}

\section{Runtime Analysis of NSGA-III on Multimodal Problems}

In this section, we study how \nsgaIII behaves on multimodal optimization problems, particularly problems that contain multiple local optima. Such problems are challenging for evolutionary algorithms, as the population may converge prematurely to suboptimal regions of the search space. Investigating how \nsgaIII handles these situations is therefore important for assessing its robustness and overall performance. In addition, we examine the effect of stochastic population update, where solutions for the next generation are not selected in a deterministic, greedy manner but instead uniformly at random. We will see that this modification can accelerate the optimization process, as it may help the algorithm escape local optima more easily and explore promising regions of the search space more effectively. As representative examples of multimodal problems, we consider \mOJZJ and \mRRMO.

\subsection{The Many-Objective OmeJumpZeroJump Problem}

In this section we define the $d$-$\OJZJfull_k$ function ($\mOJZJ_k$ for short), defined in~\citep{Zheng_Doerr_2024}, as a $d$-objective version of the bi-objective $\text{OJZJ}_k$ benchmark, and analyze the runtime of \nsgaIII on it. 

This function was studied in~\citep{Qu2022PPSN} to understand how NSGA-II can handle functions with local optima, i.e. changes of size $k \geq 2$ are necessary to cover the whole Pareto front. Fix $d \in \mathbb{N}$ divisible by two and let $n$ be divisible by $d/2$. For a bit string $x$ let $x:=(x^1, \ldots , x^{d/2})$ where all $x^j$ are of equal length $2n/d$. For odd $j \in [d]$ define $y^j:=x^{(j+1)/2}$ and $y^j:=x^{j/2}$ for even $j \in [d]$. For $2 \leq k \leq 2n/d$ the $d$-$\OJZJ_k(x)=\bigl(f_1(x), \ldots , f_d(x)\bigr)$ is defined as
\begin{align*}
  f_j(x)=\begin{cases}
           k+\ones{y^j},&\text{if $\ones{y^j} \leq \frac{2n}{d}-k$ or $y^j=1^{2n/d}$,} \\
           \frac{2n}{d}-\ones{y^j},&\text{else,}
         \end{cases}
\end{align*}
if $j \in \{1,\ldots , d\}$ is odd, and
\begin{align*}
  f_j(x)=\begin{cases}
           k+\zeros{y^j},&\text{if $\zeros{y^j} \leq \frac{2n}{d}-k$ or $y^j=0^{2n/d}$,} \\
           \frac{2n}{d}-\zeros{y^j},&\text{else,}
         \end{cases}
\end{align*}
if $j \in [d]$ is even. Note that for odd $j \in [d]$ the $(j+1)$-th objective is structurally identical to the $j$-th one with the roles of ones and zeros reversed. For every objective there are $2n/d+1$ different values and the maximum possible value is $f_{\max} = k+2n/d$. In~\citep{Zheng_Doerr_2024} it is shown that the Pareto front $F_d^*$ of $d$-$\OJZJ_k$ is $\{(a_1,2k+2n/d-a_1, \ldots , a_{d/2},2k+2n/d-a_{d/2}) \mid a_1, \ldots , a_{d/2} \in \{k,2k,2k+1, \ldots ,2n/d-1,2n/d,2n/d+k\}\}$ and has cardinality $(2n/d-2k+3)^{d/2}$ for $k \leq n/d$. Further, a maximum set of mutually incomparable solutions $S_d$ (its cardinality can be larger than $|F^*_d|$, compare with~\citep{Zheng_Doerr_2024}) fulfills $|F^*_d| \leq |S_d| \leq (2n/d+1)^{d/2}$ since $S_d$ does not contain two search points with the same fitness vector and hence $|S_d| \leq |f(\{0,1\}^n)|=(2n/d+1)^{d/2}$. In the bi-objective case, every non-Pareto optimal individual is dominated by a Pareto optimal one which is different in the many-objective case.

\begin{lemma}
\label{lem:Jump-d=2}
Suppose that $d=2$ and let $y$ be not Pareto optimal. Then every Pareto optimal $x$ strictly dominates $y$.
\end{lemma}

\begin{proof}
Let $x$ be Pareto optimal. Since $y$ is non-Pareto optimal, it fulfills $1 \leq \ones{y}<k$ or $n-k<\ones{y}\leq n-1$ (i.e. $n-k<\zeros{y}\leq n-1$ or $1 \leq \zeros{y} < k$). In the former case we have $f_1(y) = k+\ones{y} < k+\ones{x}= f_1(x)$, $f_2(y)=n-\zeros{y}<k \leq f_2(x)$ and in the latter $f_1(y) = n-\ones{y} < k \leq f_2(x)$ and $f_2(y) = k+\zeros{y} < k + \zeros{x}= f_2(x)$. In either case, $x$ dominates $y$. 
\end{proof}

We now establish improved upper bounds on the runtime without stochastic population update and compare them with the bounds obtained when stochastic update is enabled. For technical reasons, we first introduce the following. 

\begin{definition}
\label{def:Notation-Pareto-Front}
For a Pareto optimal fitness vector $v=f(x)$ we define the string $L(v) \in \{0,1,\bot\}^{d/2}$ as follows. For $j \in [d/2]$ let
\begin{itemize}
    \item $L(v)_j=1$ if $v_{2j-1}=2n/d+k$ (attained if $x^{j}=1^n$ since then $f_{2j-1}(x)=2n/d+k$),
    \item $L(v)_j=0$ if $v_{2j-1}=k$ (attained if $x^{j}=0^n$ since then $f_{2j-1}(x)=k$), and
    \item $L(v)_j=\bot$ if $2k \leq v_{2j-1} \leq 2n/d$ (attained if $k \leq \ones{x^{j}} \leq 2n/d-k$).
\end{itemize}
Note that for every Pareto optimal $x$ there is $w \in \{0,1,\bot\}^{d/2}$ with $w=L(f(x))$.
\end{definition}

For \nsgaIII with and without stochastic population update on \mOJZJ, we derive improved upper bounds on the runtime for gap sizes $2 \leq k \leq n/(3d)$ and any number $d$ of objectives. We present these bounds both in expectation and with probability $1 - 2e^{-2d} - o(1)$. This tail bound is weaker than those in the analyses above when $d = o(\log n)$, but it yields high-probability bounds if $d = \Omega(\log n)$. Nevertheless, it suffices for our purpose of comparing \nsgaIII with and without stochastic population update. In particular, the latter can yield an exponential speedup.

\begin{theorem}
\label{thm:Runtime-Analysis-NSGA-III-mOJZJ}
    Consider \nsgaIII with or without stochastic population update optimizing $f:=d$-$\OJZJ_k$ with $2 \leq k \leq n/(3d)$, $\varepsilon_{\text{nad}} \geq 2n/d$ and a set $\mathcal{R}_p$ of reference points as defined above for $p \in \mathbb{N}$ with $p \geq 4n\sqrt{d}$, and a population size $(1+a)|S_d| \leq \mu$. Then the number of generations until the whole Pareto front is covered is at most 
    $$O \left(n \ln\left(\min \left\{\frac{\mu}{|S_d|},d\right\}\right) + \frac{d^2 |S_d| n \ln(n)}{\mu} + \frac{d n^k |S_d|}{\mu} + d\ln \left(\min \left\{\frac{\mu}{|S_d|},n^k \right\} \right) \right)$$
    with probability at least $1-e^{-2d}-o(1)$ and in expectation if $a=0$. If $a=1$, then the number of generations is 
    $$O \left(n \ln\left(\min \left\{\frac{\mu}{|S_d|},d\right\}\right) + \frac{d^2 |S_d| n \ln(n)}{\mu} +\frac{kd(12en)^k}{k^k} + d\ln \left(\min \left\{\frac{\mu}{|S_d|},n^k \right\} \right) \right)$$
    until the whole Pareto front is covered with probability at least $1-2e^{-2d}-o(1)$ and in expectation.
\end{theorem}

\begin{proof}
  At first we prove the following lemma which states that with overwhelming probability there is a Pareto optimal individual after initialization.
  \begin{lemma}
  \label{lem:initialization}
   With probability $1-e^{-\Omega(n^2)}$ there is a Pareto optimal individual $x$ with $L(f(x)) = \{\bot\}^{d/2}$ after initialization, meaning that $k \leq \ones{x^j} \leq 2n/d-k$ for all $j \in [d/2]$.    
  \end{lemma}
  
  \begin{proof}
  For $d = 2$, the probability is at least $1-e^{-\Omega(n)}$ that $k \leq \ones{x} \leq n-k$. By applying a union bound on all $\mu$ individuals, the probability is at most $e^{-\Omega(\mu n)}$ that there exists no individual
  $x$ with $k \le \ones{x} \le n-k$ after initialization. Since $\mu = \Omega(n)$, this case is proven.
  
  If $d \geq 4$, we still can estimate the probability that $n/(3d) \leq \ones{x^j} \leq 5n/(3d)$ for a given block $j \in [d/2]$ (which implies that $k \leq \ones{x^j} \leq 2n/d-k$ since $k \leq n/(3d)$) by $1/2$ from below. Note also that $\mu \geq |F_d^*| = (2n/d-2k+3)^{d/2} \geq (4n/(3d)+1)^{d/2}$. Since all blocks initialize independently, the probability that all individuals $x$ are not Pareto optimal after initialization is at most $(1-(1/2)^{d/2})^\mu \leq e^{-\mu (1/2)^{d/2}} \leq e^{-(2n/(3d)+1/2)^{d/2}} \leq e^{-(n/6+1/2)^2} = e^{-\Omega(n^2)}$ where the last inequality holds by pluggin in $d=4$ since the function $g: \text{}]0,n/4[\text{} \to \mathbb{R}, x \mapsto  (2n/(3x)+1/2)^{x/2},$ is strictly monotone increasing by Lemma~\ref{lem:mon-property}, which finishes the proof of this lemma.
 \end{proof}
  
  Suppose that the event from Lemma~\ref{lem:initialization} happens. Then the probability is at least $n^{-n}$ to create any individual with mutation. Hence, the expected number of generations to create a Pareto optimal individual $x$ with $L(f(x))=\{\bot\}^{d/2}$ is $1+n^{-n}e^{-\Omega(n^2)} = 1+o(1)$. So suppose that there is a Pareto optimal $x \in P_t$ with $L(f(x)) = \{\bot\}^{d/2}$. We fix a fitness vector $v$ with a particular $L(v)$ value on the Pareto front and provide some tail bounds for the number of generations to cover $v$. We set $\lambda_t:=\max\{j \in \{0\} \cup [d/2] \mid \text{ there is $y \in P_t$ with $L(f(y))_i = L(v)_i$ for all $i \leq j$ and $L(f(y))_i = \bot$ \text{ for all $i>j$}}\}$ as the maximum possible number of blocks from the left set correctly in an individual $y \in P_t$, where all other blocks $i>\lambda_t$ still satisfy $k \leq \ones{y^i} \leq 2n/d-k$. Then such a $v$ is covered if $\lambda_t=d/2$. By Lemma~\ref{lem:Reference-Points} Pareto optimal individuals are not lost between generations and hence, $\lambda_t$ cannot decrease. Fix $y \in P_t$ with value $\lambda_t$ and let $w:=f(y)$. If $\lambda_t \in \{0, \ldots , d/2-1\}$ and $j =\lambda_t+1$ denote by $X_j$ the random variable until for every $k \leq i \leq 2n/d-k$ there is an individual $z$ with $\ones{z^j}=i$ and covering a Pareto optimal fitness vector $w$ with $L(w)=L(f(y))$. Then, when stochastic population update is disabled, let $Y_j$ denote the random variable counting the number of generations until $c_t(f(z')) \geq \min\{\lfloor \mu / |S_d| \rfloor, n^k\}$, where $z'$ is an individual satisfying $L(f(z')) = L(f(y))$, and the $j$-th block $(z')^j$ has Hamming distance $k$ from $0^{2n/d}$ if $x^j = 0^{2n/d}$, or from $1^{2n/d}$ if $x^j = 1^{2n/d}$. Otherwise, we set $Y_j=0$. We make this choice for $Y_j$ in case of $a=1$, as replicating such individuals $z$ has no influence on the process of crossing the fitness valley we are investigating, and jumping into that valley happens with probability $\Omega(k/n)$. Note that the latter is already significantly larger than crossing the valley in a single jump. Finally, let $Z_j$ the random variable until $\lambda_t$ is increased by finding a solution $z^*$ with $L(f(z^*))_j = L(v)_j$ and $L(f(z^*))_i = L(w)_i$ for $i \in [d/2] \setminus \{j\}$. Note that the time to obtain $\lambda_t=j$ is stochastically dominated by $X_j+Y_j+Z_j$, if already $\lambda_t=j-1$. Then the total time until a solution $x$ with $L(f(x)) = L(v)$ is found is stochastically dominated by $\sum_{j=1}^{d/2} (X_j+Y_j+Z_j)$. When considering $X_j$, one has already processed $\sum_{i=1}^{j-1} (X_i+Y_i+Z_i)$, which leads to an individual $z$ with $L(f(z))=L(v)_i$ for $i \in \{1, \ldots , j-1\}$, and $L(f(z)) = \{\bot\}$ otherwise. When considering $Y_j$ then one has already processed $X_j$, and when considering $Z_j$ one has already processed $X_j$ and $Y_j$. We have found a desired $x$ after $\sum_{i=1}^{d/2} (X_i+Y_i+Z_i)$ generations. In the following three lemmas, we provide some tail bounds on $X:=\sum_{i=1}^{d/2} X_i$, $Y:=\sum_{i=1}^{d/2} Y_i$ and $Z:=\sum_{i=1}^{d/2} Z_i$ separately.
  
  \begin{lemma}
  \label{lem:interior-covering}
  Let $c_1>0$ be a sufficiently large constant. Then
  $$\Pr (X \geq c_1 (n + n \ln (\min \{\mu/((1+a)|S_d|),d \}) + d^2 |S_d| n \ln(n)/\mu)) = O(n^{-4d}).$$
  Particularly, $\Pr(X \geq 2c_1 d^2 n \ln(n)) = O(n^{-4d})$.
  \end{lemma} 
  
  \begin{proof}
  Suppose that there is $y \in P_t$ with $L(f(y))_{i} = L(v)_{i}$ for all $i \in \{1, \ldots , j-1\}$, while $L(f(y))_{i} = \bot$ for all $i \geq j$. Let $w:=f(y)$. For a given vector $u$ on the Pareto front with $L(u) = L(w)$ and $2k \leq u_{2j-1} \leq 2n/d$ let $h_t:=\min\{f_{2j-1}(z) - u_{2j-1} \mid z \in P_t \text{ with } L(f(z)) = L(w)\}$ the minimum distance to $u_{2j-1}$ in the $(2j-1)$-th objective of individuals $z \in P_t$ with the same $L$-value as $y$. Note that $h_t$ cannot increase (by Lemma~\ref{lem:Reference-Points}) and that $0 \leq h_t \leq 2n/d-2k$. We created a desired individual $z$ with $L(f(z))=L(w)$ and $f_{2j-1}(z)=u_{2j-1}$ if $h_t=0$.
  For $j \in [d/2]$ and $\ell \in [2n/d]$ denote by $X_{j,\ell}$ the number of generations $t$ with $h_t=\ell$. We estimate $X_{j,\ell}$ by first creating at least $\alpha_\ell:=\min\{\lceil{n/\ell}\rceil, \lfloor{\mu/((1+a)|S_d|)}\rfloor\}$ individuals $z$ with $h_t(z)=\ell$, denoting by $X_{j,\ell}^{\text{clone}}$, and then by decreasing $h_t$ via mutation, denoting by $X_{j,\ell}^{\text{mut}}$. Then, $X$ is stochastically dominated by $\sum_{j=1}^{d/2} \sum_{\ell=1}^{2n/d}(X_{j,\ell}^{\text{clone}} + X_{j,\ell}^{\text{mut}})$. Let $X^{\text{clone}} := \sum_{j=1}^{d/2} \sum_{\ell=1}^{2n/d}X_{j,\ell}^{\text{clone}}$ and $X^{\text{mut}} := \sum_{j=1}^{d/2} \sum_{\ell=1}^{2n/d} X_{j,\ell}^{\text{mut}}$. 
  
  At first we consider the special case that there is a constant $\kappa \geq 1$ such that $d \leq \kappa \sqrt{n}/\ln(n)$. Here, we closely follow parts of the proof of Theorem~\ref{thm:Runtime-Analysis-NSGA-III-mOMM}. For a given $j$ (if $\lambda_t=j-1$) and $\ell$ the strategy is first to generate $\nu := \min\{\lceil 288 \kappa \sqrt{n} \rceil, \lfloor \mu/((1+a)|S_d|) \rfloor\}$ individuals with distance $h_t = \ell$ by repeatedly replicating such individuals (using Lemma~\ref{lem:clone-general}(1)). Subsequently, we let generate the remaining individuals in larger rates considering Chernoff bounds (via Lemma~\ref{lem:clone-general}(2)). For $\gamma \in \{0, \ldots , \nu-1\} \setminus \{0\}$ denote by $X^{\text{clone}}_{j,\ell,\gamma}$ the number of generations until there are $\gamma$ individuals $x$ with distance $h_t = \ell$ or $h_t$ increases, and when there are at least $\nu$ such individuals, by $X_{j,\ell}^*$ the number of generations until there are $\alpha_\ell$ individuals with $\sum_{i=1}^{d/2} |f_{2i-1}(x)-f_{2i-1}(y)|=\ell$ or $g_t$ increases. The latter only needs to be considered if $\nu =\lceil{288 \kappa \sqrt{n}}\rceil \leq \lceil{n/\ell}\rceil-1$, which is only satisfied if $\ell \leq \lceil{\sqrt{n}/(288 \kappa)}\rceil$. So $X^{\text{clone}}$ is stochastically dominated by $\sum_{j=1}^{d/2} \sum_{\ell=1}^{2n/d} \sum_{\gamma=1}^{\min\{\alpha_\ell,\nu\}-1} X_{j,\ell,\gamma}^{\text{clone}} + \sum_{j=1}^{d/2}\sum_{\ell=1}^{\lceil{\sqrt{n}/(288 \kappa)}\rceil} X_{j,\ell}^*$. Note that $\sum_{\ell=1}^{2n/d} \sum_{\gamma=1}^{\min\{\alpha_\ell,\nu\}-1} X_{j,\ell,\gamma}^{\text{clone}}$ is stochastically dominated by  $\sum_{\gamma=1}^{\nu-1} \sum_{\ell=1}^{\min\{\lfloor{n/\gamma}\rfloor,2n/d\}} X_{j,\ell,\gamma}^{\text{clone}}$ for all $j \in [d/2]$ (since $\ell \leq \lfloor{n/\gamma}\rfloor$ if $\gamma \leq \lceil{n/\ell}\rceil-1$).  
  For $\gamma \in \{0, \ldots , \min\{\alpha_\ell,\nu\}-1\} \setminus \{0\}$ and $j \in [d/2]$ consider $A_{j,\gamma}:=\sum_{\ell=1}^{\min\{\lfloor{n/\gamma}\rfloor,2n/d\}} X_{j,\ell,\gamma}^{\text{clone}}$. 
  By Lemma~\ref{lem:clone-general}(1) (with $\beta=\gamma$, $m=\min\{\lfloor{n/\gamma}\rfloor,2n/d\}$, and $\delta=288 \kappa^2m$)  we obtain $\Pr(A_{j,\gamma} \geq 5760 (\kappa^2+2) m/\gamma) \leq e^{-4 \cdot 288 \kappa^2 m}$ for $j \in [d/2]$. Let $\rho:=\min\{d/2,\lfloor{\mu/((1+a)|S_d|)}\rfloor\}$. For $\gamma \in [\rho]$ we have $m=\min\{\lfloor{n/\gamma}\rfloor,2n/d\} = 2n/d$. Hence, by a union bound,
  \begin{align*}
  \Pr\bigl(\sum_{j=1}^{d/2}\sum_{\gamma=1}^{\rho} A_{j,\gamma} \geq \sum_{\gamma=1}^{\rho} 5760 (\kappa^2+2) n/\gamma\bigr) \leq d^2/4 \cdot e^{-4 \cdot 288 \kappa^2 \cdot 2n/d} \leq e^{-4 \kappa \sqrt{n}}
  \end{align*}
  implying that
  \begin{align*}
  \Pr\bigl(\sum_{j=1}^{d/2}\sum_{\gamma=1}^{\rho} A_{j,\gamma} \geq 11520 (\kappa^2+2) \ln(\rho)n\bigr) \leq e^{-4 \kappa \sqrt{n}},
  \end{align*}
  since $\sum_{i=1}^{\rho} 1/i \leq 2 \ln(\rho)$. Further, by noting that $m=\min\{\lfloor{n/\gamma}\rfloor,2n/d\} = \lfloor{n/\gamma}\rfloor$ if $d/2 < \gamma \leq 2n/d$
  \begin{align*}
   \Pr\bigl(\sum_{j=1}^{d/2} \sum_{\gamma=d/2+1}^{2n/d} A_{j,\gamma} = 11520(\kappa^2+2)n\bigr)&= \Pr\bigl(\sum_{j=1}^{d/2} \sum_{\gamma=d/2+1}^{2n/d} A_{j,\gamma} \geq 5760 (\kappa^2+2) n d \cdot 2/d \bigr)\\
   &\leq \Pr\bigl(\sum_{j=1}^{d/2} \sum_{\gamma=d/2+1}^{2n/d} A_{j,\gamma} \geq 5760 (\kappa^2+2) n d \cdot \sum_{\gamma=d/2+1}^{2n/d} 1/\gamma^2\bigr)\\
   &\leq \Pr\bigl(\sum_{j=1}^{d/2} \sum_{\gamma=d/2+1}^{2n/d} A_{j,\gamma} \geq 5760 (\kappa^2+2) m d/2 \cdot \sum_{\gamma=d/2+1}^{2n/d} 1/\gamma\bigr)\\
   &\leq ne^{-4 \cdot 288 \kappa^2 \sqrt{n}}.
  \end{align*}
  Here, we used $\sum_{j=d/2+1}^{2n/d} 1/j^2 \leq \sum_{j=d/2+1}^{\infty} 1/j^2 \leq 2/d$ (compare with Lemma~\ref{lem:inequalities}). The inequality before the last one results from a union bound over all possible $j \in [d/2]$ and $\gamma \in \{d/2+1, \ldots , 2n/d\}$. To estimate the sum  $X_j^*:=\sum_{\ell=1}^{\lceil{\sqrt{n}/(288\kappa)}\rceil} X_{j,\ell}^*$ 
  for $j \in [d/2]$ we use Lemma~\ref{lem:clone-general}(2) on $\beta= \lceil{n/\ell}\rceil \leq n$ and $m= \lceil{\sqrt{n}/(288\kappa)}\rceil$ to obtain $\Pr(X_j^* \geq 2\sqrt{n} \ln(n)/(288 \kappa)) \leq \sqrt{n} \ln(n)e^{-288 \kappa \sqrt{n}/72} \leq e^{-3 \kappa \sqrt{n}}$ for $n$ sufficiently large. So we obtain by a union bound $\Pr(\sum_{j=1}^{d/2} X_j^* \geq \sqrt{n}d \ln(n)/(288 \kappa)) \leq d/2 \cdot e^{-3 \kappa \sqrt{n}}$ and hence, considering both cases for $m$ 
  \begin{align*}
  \Pr(X^{\text{clone}} &\geq 11520 (\kappa^2+2) \ln(\rho)n + 11520 (\kappa^2+2) n + \sqrt{n}d\ln(n)/(288\kappa)) \leq e^{-2\kappa\sqrt{n}} 
  \end{align*}
  for $n$ sufficiently large. 
  
  Now we estimate $X^{\text{mut}}$. Call an individual $x$ \emph{$\ell$-promising} if $L(f(x)) = L(f(y)) = L(w)$, and $f_{2i-1}(x) - u_{2i-1} = \ell$. To decrease $h_t=\ell$, it suffices to choose an $\ell$-promising individual $x \in P_t$ as a parent (which happens with probability at least $\min\{n/(\ell \mu), \lfloor{\mu/((1+a)|S_d|)}\rfloor/\mu\} \geq \min\{n/(\ell \mu), 1/(4|S_d|)\}$) and flip one out of $\ell$ specific bits, while keeping the other bits unchanged (which happens with probability at least $\ell/n \cdot (1-1/n)^{n-1} \geq \ell/(en)$). Since a generation consists of $\mu$ trials, the probability for decreasing $\ell=h_t$ in one generation is at least $1-(1-s_\ell)^\mu \geq \frac{s_\ell \mu}{s_\ell \mu+1}=:p_\ell$, where $s_\ell=\ell \cdot \min\{n/(\ell\mu), 1/(4|S_d|)\}/(en) = \min\{1/(e\mu), \ell/(4en|S_d|)\}$. Note that $1/(e\mu) \geq \ell/(4en|S_d|)$ if $\ell \in [\xi]$ for $\xi:=\min\{\lfloor{4n|S_d|/\mu}\rfloor,2n/d\}$ and otherwise $1/(e\mu) < \ell/(4en|S_d|)$. Also observe for $\ell \in [\xi]$ that
  $$p_\ell = \frac{\ell \mu/(4en|S_d|)}{1+\ell\mu/(4en|S_d|)} = \frac{1}{4en|S_d|/(\ell \mu)+1} \geq \frac{\ell \mu}{8en|S_d|},$$
  and for $\xi<\ell$ observe that $p_\ell \geq 1/(2e+1) \geq 1/7$. So we see that $X^{\text{mut}}$ is stochastically dominated by $dA/2$ and $dB/2$ where $A:=\sum_{\ell=1}^\xi A_\ell$ and $B:=\sum_{\ell=\xi+1}^{2n/d} B_\ell$ are two independent sums of geometrically distributed random variables with success probabilities $c\ell/\xi$ for $c:=\xi \mu/(8e n |S_d|)$ and $1/7$ respectively. The factor $d/2$ comes from the fact that we have to consider every block $j \in [d/2]$. For $A$ and $B$ we compute tail bounds separately. At first consider $dB/2$. Note that $\expect{dB/2} \leq 7n$. We obtain for $\lambda \geq 0$ by Lemma~\ref{lem:Doerr-dominance} 
    \[
    \Pr\left(\frac{dB}{2} \geq \expect{\frac{dB}{2}} + \lambda\right) \leq \exp\left(-\frac{1}{4} \min\left\{\frac{\lambda^2}{49(n-\xi)}, \frac{\lambda}{7} \right\}\right) \leq \exp\left(-\frac{1}{4} \min\left\{\frac{\lambda^2}{49n}, \frac{\lambda}{7} \right\}\right).
    \]
    By using $\lambda=28n$ we obtain $\Pr(dB/2 \geq 35n) = e^{-n}$. Considering $dA/2$, we can also assume that $\xi>1$, since if $\xi=1$ (implying $\ell=1$ and $\mu \geq 2n|S_d|$) we also see that $p_\ell \geq 1/7$. Hence, this case is already captured by the random variable $dB/2$. So we obtain for any $\delta>0$ by Lemma~\ref{lem:Doerr-dominance2}
    \begin{align*}
    \Pr \left(A \geq \frac{(1+\delta) 8en |S_d| \ln(\xi)}{\mu} \right) = \Pr \left(A \geq \frac{(1+\delta) \xi \ln(\xi)}{c} \right) \leq \xi^{-\delta}
    \end{align*}
    and hence, for $\delta=4d \ln(n)/\ln(\xi)$ we have (by using $\ln(\xi) \leq 2\ln(n)$)
    \begin{align*}
    \Pr&\left(A \geq \frac{(16+32d) e |S_d|n\ln(n)}{\mu}\right) = \Pr\left(A \geq \frac{(2+4d) \xi \ln(n)}{c}\right) = \Pr\left(A \geq \frac{2\xi \ln(n)}{c} + \frac{4d \xi \ln(n)}{c}\right)\\
    &\leq \Pr \left(A \geq \frac{\xi \ln(\xi)}{c} + \frac{\delta \xi \ln(\xi)}{c}\right) = \Pr \left(A \geq \frac{(1+\delta) \xi \ln(\xi)}{c}\right) \leq \xi^{-4d \ln(n)/\ln(\xi)} = e^{-4d \ln(n)}.
    \end{align*}
    By a union bound over both $dA/2$ and $dB/2$ we obtain 
    \begin{align*}
    \Pr &\left(X^{\text{mut}} \geq 35n + \frac{(8+16d) d e |S_d|n\ln(n)}{\mu} \right) \leq \Pr \left(\frac{dA}{2}+\frac{dB}{2} \geq 35n + \frac{(8+16d)d e |S_d|n\ln(n)}{\mu}\right) \\
    &\leq e^{-n} + e^{-4d \ln(n)}.
    \end{align*}
    Note that $d \sqrt{n} \ln(n) = O(n)$ due to $d \in O(\sqrt{n}/\ln(n))$. Hence, the result follows by a union bound over $X^{\text{mut}}$ and $X^{\text{clone}}$ by noting that $e^{-4d \ln(n)} = n^{-4d}$.
    

    In the general case, we even wait until $X_{j,\ell}^{\text{clone}} \geq \min\{2n/d,\lfloor{\mu/((1+a)|S_d|)}\rfloor\}=:\rho$ for all $j \in [d/2]$ and $\ell \in [2n/d]$, and then decrease $h_t$ by considering $X_{j,\ell}^{\text{mut}}$. We estimate $X^{\text{clone}}$ and $X^{\text{mut}}$ separately and start with $X^{\text{clone}}$. Similar as above, for $\gamma \in \{0, \ldots , \rho-1\} \setminus \{0\}$ denote by $X_{j,\ell,\gamma}^{\text{clone}}$ the random variable for the number of generations until there are exactly $\gamma$ distinct individuals $x$ with $g_t(x)=\ell$, or $h_t$ decreases. Then $X^{\text{clone}}$ is stochastically dominated by $\sum_{j=1}^{d/2}\sum_{\ell=1}^{2n/d}\sum_{\gamma=1}^{\rho-1}X_{j,\ell,\gamma}^{\text{clone}}$. By Lemma~\ref{lem:clone-general}(1), $\Pr(\sum_{j=1}^{d/2} \sum_{\ell=1}^{2n/d} X_{j,\ell,\gamma}^{\text{clone}} \geq 21n/\gamma) = e^{-4n}$ for every $\gamma \in [\rho-1]$ (by using $\beta=\gamma+1=i+1$, $\delta=0$ and $m=n$) and hence, by a union bound over $\gamma \in [\rho-1]$, we obtain 
    $$\Pr(X^{\text{clone}} \geq 42n \ln(\rho)) \leq \Pr(X^{\text{clone}} \geq \sum_{\gamma=1}^{\rho-1} 21n/\gamma) \leq \rho e^{-4n} \leq e^{-3n}$$ 
    (due to $\sum_{i=1}^{\rho-1} 21n/i \leq 42n \ln(\rho)$). Now we can estimate $X^{\text{mut}}$ in the same way as above. Since there are at least $\min\{n/(\ell\mu), \lfloor{\mu/((1+a)|S_d|)}\rfloor/\mu\} \geq \min\{n/(\ell \mu), 1/(4|S_d|)\}$) $\ell$-promising individuals, we can also say that $X^{\text{mut}}$ is stochastically dominated by $dA/2$ and $dB/2$ where $A:=\sum_{\ell=1}^\xi A_\ell$ and $B:=\sum_{\ell=\xi+1}^{2n/d} B_\ell$ are two independent sums of geometrically distributed random variables with success probabilities $c\ell/n$ for $c:=\mu/(4e |S_d|)$ and $1/7$ respectively. Hence, we also obtain 
    \begin{align*}
    \Pr &\left(X^{\text{mut}} \geq 35n + \frac{(8+16d) d e |S_d| n\ln(n)}{\mu} \right) \leq \Pr \left(\frac{dA}{2}+\frac{dB}{2} \geq 35n + \frac{(8+16d) d e |S_d| n\ln(n)}{\mu}\right) \\
    & \leq e^{-n} + e^{-4d \ln(n)},
    \end{align*}
    concluding the proof of Lemma~\ref{lem:interior-covering} by a union bound. Note that $\ln(d) = \Theta(\ln(n))$,  and 
    $$n + n \ln \left(\min \left\{\frac{\mu}{(1+a)|S_d|},d \right\} \right) + \frac{d^2 |S_d| n \ln(n)}{\mu} \in O(d^2 n \ln(n))$$
    which yields the second part of the theorem, particularly the estimation for the probability that $X$ exceeds $2c_1 d^2 n \ln(n)$ for a sufficiently large constant $c_1>0$.
    \end{proof}

    Now we are ready to estimate $Y = \sum_{j=1}^{d/2} Y_j$. Recall that $Y_j$ counts the number of generations until $c_t(f(z)) \geq \min\{\lfloor \mu / |S_d| \rfloor, n^k\}$. Note that we only need to consider $Y_j$ if the $(2j-1)$-th objective value of the target vector $v$ is either $k$ or $2n/d + k$. Further, $Y_j$ only contributes to the total runtime when stochastic population update is disabled.
    
\begin{lemma}
  \label{lem:interior-stack}
   For a sufficiently large constant $c_2>0$ we have
   \begin{align*}
   \Pr(Y &\geq c_2(\ln^2(\min\{\lfloor{\mu/|S_d|}\rfloor,n\}) + d\ln(\min\{\mu/|S_d|,n\}) + d\ln(\min\{\lfloor{\mu/|S_d|}\rfloor,n^k\})/2)) = e^{-2d} + o(1).
   \end{align*}
\end{lemma}
  
  \begin{proof}
  Let $\rho:=\min\{\mu/|S_d|, 144n\}$. For $\gamma \in \{0, \ldots , \rho-1\} \setminus \{0\}$ and $j \in [d/2]$, let $Y_{j,\gamma}$ denote the random variable counting the number of generations until there are exactly $\gamma$ individuals $x$ whose Hamming distance is $k$ from $0^{2n/d}$ if $v_{2j-1} = k$, or from $1^{2n/d}$ if $v_{2j-1} = 2n/d+k$. Further, for $j \in [d/2]$ denote by $Y_j^*$ the random variable until $\min\{\mu/|S_d|, n^k\}$ such individuals are created, when there are at least $\rho$ many. Then $Y=\sum_{\gamma=1}^{\rho-1} \sum_{j=1}^{d/2} Y_{j,\gamma} + \sum_{j=1}^{d/2} Y_j^*$. For $\gamma \in \{0, \ldots , \rho-1\} \setminus \{0\}$ we have 
  \begin{align*}
  \Pr\bigl(\sum_{j=1}^{d/2} Y_{j,\gamma} &\geq 21(d + \ln(\min\{n,\mu/|S_d|\}))/\gamma \bigr) \leq e^{-4d - 4\ln(\min\{n,\mu/|S_d|\})}
  \end{align*}
  by Lemma~\ref{lem:clone-general}(1)  (by using $\beta=\gamma+1$, $\delta=d/2 + \ln(\min\{n,\mu/|S_d|\})$ and $m=d/2$). With a union bound we obtain 
  \begin{align*}
  \Pr \bigl(\sum_{\gamma=1}^{\rho-1}\sum_{j=1}^{d/2} Y_{j,\gamma} &\geq 42\ln(\rho)(d+\ln(\min\{\mu/|S_d|,n\}))\bigr) \\
  &\leq e^{-2d - 2\ln(\min\{\mu/|S_d|,n\})} \leq e^{-2d}. 
  \end{align*}
  By Lemma~\ref{lem:clone-general}(2) (on $m=d/2$ and $\beta = \min\{\mu/|S_d|,n^k\}$) we have $\Pr(\sum_{j=1}^{d/2} Y_j^* \geq d \ln(\min\{\mu/|S_d|,n^k\})/2) \leq d/2 \ln(\beta) e^{-2n} \leq e^{-n}$, and the proof of this lemma is finished by noting that $Y_j=0$ if stochastic population update is enabled. 
\end{proof}

Finally, we estimate $Z:=\sum_{j=1}^{d/2} Z_j$. Recall that $Z_j$ is the time to increase $\lambda_{j-1}$ after processing $X_j + Y_j$, particularly to create $1^{2n/d}$ in block $j$ if $v_{2j-1}=2n/d+k$ or $0^{2n/d}$ in block $j$ if $v_{2j-1}=k$.
  
\begin{lemma}
  \label{lem:jump}
   For $a=0$ we have 
   $$\Pr(Z \geq d \lceil{119en^k |S_d|/\mu \rceil}) \leq e^{-4d}$$
   and for $a=1$
   $$\Pr(Z \geq 17kd(12en)^k/k^k) \leq e^{-4d}.$$
\end{lemma}
  
\begin{proof}
Suppose that $a=0$. By symmetry, we can assume that $v_j = k+2n/d$. Hence, we must evolve a corresponding search point $z$ with $z^j = 1^{2n/d}$ and $L(f(z))_i=L(w)_i$ for $i \neq j$. To estimate $Z_j$ for $j \in [d/2]$, we can assume that there are $\min\{\lceil{n^k}\rceil,\lfloor{\mu/|S_d|}\rfloor\}$ individuals $x$ with Hamming distance $k$ to an individual $y$ with $L(f(y))_i=L(v)_i$ for $i \leq j$ and $L(f(y))_i = \bot$ for $i>j$ otherwise since we already processed $Y_j$. So in a trial, it suffices to choose an individual $x$ with this distance as a parent (which happens with probability at least $\min\{n^k/\mu, \lfloor{\mu/|S_d|}\rfloor/\mu\} \geq \min\{n^k/\mu, 1/(2|S_d|)\}$) and flip $k$ specific bits, while keeping the other bits unchanged (which happens with probability at least $n^k \cdot (1-1/n)^{n-1} \geq n^k/e$). Since a generation consists of $\mu$ trials, the probability for creating such a $y$ in one generation is at least 
\[
1-(1-s_k)^\mu \geq \frac{s_k \mu}{s_k \mu+1}=:p_k
\]
where $s_k=\min\{n^k/\mu, 1/(2|S_d|)\}/(en^k) = \min\{1/(e\mu), 1/(2en^k|S_d|)\}$. Also observe for $\mu \leq 2n^k |S_d|$ that
$$p_k = \frac{\mu/(2en^k|S_d|)}{1+\mu/(2en^k|S_d|))} = \frac{1}{2en^k|S_d|/\mu+1} \geq \frac{\mu}{4en^k|S_d|},$$
and for $\mu > 2n^k |S_d|$ observe that $p_k \geq 1/(2e+1) \geq 1/7$. So we see that $\sum_{i=1}^{d/2} Z_i$ is stochastically dominated by an independent sum $A:=\sum_{k=1}^{d/2} A_k$ of geometrically distributed random variables with success probability $1/7$ if $\mu > 2n^k |S_d|$ and with success probability $\frac{\mu}{4en^k|S_d|}$ otherwise. For both cases, we determine tail bounds separately by using 
\[
\Pr(A \geq \expect{A} + \lambda) \leq \exp\left(-\frac{1}{4} \min\left\{\frac{2 \lambda^2 p_k^2}{d}, \lambda p_k \right\}\right)
\]
for $\lambda>0$ (compare with Lemma~\ref{lem:Doerr-dominance}). Note that $\expect{A} = d/(2p_k)$. If $\mu > 2n^k |S_d|$ we have $\Pr(A \geq 119d) \leq e^{-4d}$ by using $\lambda=112d$ and noting that $\expect{A} \leq 7d$ and if $\mu \leq 2n^k |S_d|$ we have $\Pr(A \geq 66edn^k |S_d|/\mu) \leq e^{-4d}$ by using $\lambda=64edn^k|S_d|/\mu$ and noting that $\expect{A} \leq 2edn^k|S_d|/\mu$.
By considering both situations simultaneously, we obtain 
$$\Pr(Z \geq d\lceil{119en^k |S_d|/\mu}\rceil) \leq \Pr(A \geq d \lceil{119en^k |S_d|/\mu \rceil}) = e^{-4d},$$
and the case $a=0$ follows. 

Now consider the case $a=1$. By symmetry we can assume that $v_j = k+2n/d$. We consider a sequence of $k$ successive generations and for $\ell \in [k]$ we call the $\ell$-th generation \emph{successful} if an individual $z$ is created with $2n/d-k+\ell$ ones in block $j$ while $z^i=x^i$ for $i \in [d/2] \setminus \{j\}$ and $z$ is not removed. Suppose that the $(\ell-1)$-th generation is successful. Then, generation $\ell$ is successful if in one trial one chooses a parent $p$ with $\ones{p^j} = 2n/d-k+\ell-1$ and $z^i=x^i$ for $i \in [d/2] \setminus \{j\}$, flips a zero bit in the $j$-th block while keeping the remaining bits unchanged, and finally keeps the new created individual. With probability at least 
\[
1-\left(1-\frac{k-\ell+1}{e \mu n}\right)^{\mu} \geq \frac{\frac{k-\ell+1}{en}}{1+\frac{k-\ell+1}{en}} \geq \frac{k-\ell+1}{4n}
\]
a desired individual $z$ is created in generation $\ell$. Note that $z$ is not Pareto optimal for $\ell<k$. However, $\mu-\lceil{3\mu/2}\rceil$ individuals chosen uniformly at random survive (see Line~\ref{line:update} and Line~\ref{line:survival} in Algorithm~\ref{alg:nsga-iii} for $a=1$) and hence even a non Pareto optimal individual remains with probability at least $(\mu - \lceil{3 \mu/2}\rceil)/\mu = \lfloor{\mu/2}\rfloor/\mu \geq (\mu/2-1)/\mu = 1/2- 1/\mu$. For $n$ sufficiently large, we have $1/2- 1/\mu \geq 1/3$ and hence with probability at least $(k-\ell+1)/(12n)$, generation $\ell$ is successful. If all $k$ generations are successful, the desired $y$ is created, which happens with probability at least 
\begin{align*}
\prod_{\ell=1}^k \frac{k-\ell+1}{12n} = \frac{k!}{(12n)^k} \geq \frac{k^k}{(12en)^k}
\end{align*}
(since $k! \geq (k/e)^k$ by Stirling's approximation). Hence, $Z_j$ is stochastically dominated by $k \cdot Z_j^*$ where  $Z_j^*$ is a geometrically distributed random variable with success probability $p_k:=k^k/(12en)^k$. Note that $\expect{Z^*} = 1/p_k = d(12en)^k/(2k^k)$. Hence, we obtain for any $\lambda>0$ 
\begin{align*}
\Pr\Bigl(Z \geq \frac{kd(12en)^k}{2k^k} + k\lambda \Bigr) &\leq \Pr\Bigl(k \cdot Z^* \geq \frac{kd(12en)^k}{2k^k} + k\lambda \Bigr) = \Pr\Bigl(Z^* \geq \frac{d(12en)^k}{2k^k} + \lambda \Bigr)\\
&= \Pr(Z^* \geq \expect{Z^*} + \lambda) \leq \exp\left(-\frac{1}{4} \min\left\{\frac{2\lambda^2p_k^2}{d}, \lambda p_k \right\}\right)
\end{align*}
and hence, for $\lambda = 16d/p_k$ we obtain that $\Pr(Z^* \geq 17kd(12en)^k/k^k) = e^{-4d}$, providing the case $a=1$ and proving the lemma. 
\end{proof}

Now, we can apply a union bound together with the three lemmas above to obtain, for the case where stochastic population update is disabled and for a sufficiently large constant $c > 0$, that
\begin{align*}
\Pr\bigl(X+Y+Z &\geq c(n \ln(\min\{\mu/|S_d|,d\}) + d^2 |S_d| n \ln(n)/\mu + d\ln(\min\{\mu/|S_d|,n^k\})+ d n^k |S_d|/\mu)\bigr)\\
&\leq 2e^{-4d} + o(1).
\end{align*}
We used that the linear summand in $n$ in the estimation for $X$ in Lemma~\ref{lem:interior-covering} is absorbed by $d^2|S_d| n \ln(n)/\mu + n \ln(\min\{\mu/((1+a)|S_d|)\})$, and $\ln^2(\min\{\lfloor{\mu/|S_d|}\rfloor,n\}) \in O(n)$. In a similar way, we obtain for the case that stochastic population update is enabled,
\begin{align*}
\Pr\bigl(X+Y+Z &\geq c(n \ln(\min\{\mu/|S_d|,d\}) + d^2 |S_d| n \ln(n)/\mu + kd(12en)^k/k^k)\bigr)\\
&\leq 2e^{-4d} + o(1).
\end{align*}

By a union bound over all possible values $s$ from $\{0,1,\bot\}^{d/2}$, we see for all $s \in \{0,1,\bot\}^{d/2}$ that there is a Pareto optimum $x \in P_t$ with $L(f(x)) = s$ in 
\begin{equation}
\label{eq:1}
O(n \ln(\min\{\mu/|S_d|,d\}) + d^2 |S_d| n \ln(n)/\mu + d\ln(\min\{\mu/|S_d|,n^k\})+ d n^k |S_d|/\mu)
\end{equation}
generations if $a=0$ and 
\begin{equation}
\label{eq:2}
O(n \ln(\min\{\mu/|S_d|,d\}) + d^2 |S_d| n \ln(n)/\mu + kd(12en)^k/k^k)
\end{equation}
generations if $a=1$ with probability at least $1-e^{-2d}$ for $n$ sufficiently large. Suppose that this happens. Now it remains to cover the Pareto front which can be done in a similar way as in the \OMM case: One fixes a Pareto optimal fitness vector $v$ and defines $\delta_t:=\min\{\sum_{i=1}^{d/2} |f_{2i-1}(x)-v_{2i-1}| \mid x \in P_t \text{ with } L(f(x)) = L(v)\}$. One can decrease $\delta_t$ by choosing an individual $y \in P_t$ with $\sum_{i=1}^{d/2} |f_{2i-1}(y)-v_{2i-1}| = \delta_t$ as a parent, and flip one of $\delta_t$ specific bits. Hence, in the same way as in the proof of Theorem~\ref{thm:Runtime-Analysis-NSGA-III-mOMM}, covering the Pareto front happens in $O(n\ln(\min\{\mu/|S_d|,n\}) + d|S_d| n \ln(n)/\mu)$ generations, or $O(n + d|S_d| n \ln(n)/\mu)$ generations if also $d \in O(\sqrt{n}/\ln(n))$, with probability at least $1-O(1/n^{3d})$. This provides the result of the theorem by noting that when a failure occurs, then all the arguments from above can be repeated (outgoing from a successful initialization). In expectation, there are $O(1)$ repetitions, since a failure happens with probability at most $2e^{-2d}+o(1)$. Note that $n + d|S_d| n \ln(n)/\mu$ for $d \in O(\sqrt{n}/\ln(n))$, and $n\ln(\min\{\mu/|S_d|,n\}) + d|S_d| n \ln(n)/\mu$ can be both absorbed by the summands occuring in~\ref{eq:1} and~\ref{eq:2}, providing the proof of the whole theorem.  
\end{proof} 

\subsection{A Many-Objective Real-Royal-Road Problem}

A different example of a many-objective multimodal function, the many-objective Real-Royal-Road function (\mRRMO for short), was introduced in~\citep{Opris2025}, specifically designed to demonstrate that applying crossover can be highly beneficial, potentially leading to an exponential speedup in runtime. It was shown that \nsgaIII with mutation alone requires $n^{\Omega(n/d)}/\mu$ generations in expectation to find a single Pareto-optimal solution, whereas enabling one-point crossover allows the algorithm to explore the entire Pareto front in a polynomial number of generations. However, it is important to note that the number of possible offspring generated by one-point crossover on two fixed parents is quite limited. Specifically, it grows linearly with the length of the bit string, even though the search space has size $2^n$. Consequently, the characteristics of the offspring depend heavily on the positions of genes in the parents, which can result in a very sparse exploration of the search space and may even prevent escaping from local optima. In fact, if the Pareto front of the Real-Royal-Road function were slightly different, \nsgaIII with one-point crossover would likely fail as well. In contrast, stochastic population update maintains a much more diverse population, as even low-fitness individuals can survive. This promotes a better balance between exploration and exploitation. In particular, these weaker solutions can continue exploring the search space via standard bit mutation and traverse large fitness valleys.

\begin{definition}[Definition~12 in~\citep{OPRIS2026}]
\label{def:RRRMO-one-point}
Regarding the substring $x^j$ for $j \in [d/2]$, let 
\begin{itemize}
\item $B:=\{y \in \{0,1\}^{2n/d} \mid \ones{y} = 6n/(5d), \LZ(y)+\TZ(y)=4n/(5d)\}$, and 
\item $A:=\{y \in \{0,1\}^{2n/d} \mid \ones{y} = 8n/(5d), \LZ(y)+\TZ(y)=2n/(5d)\}$.
\end{itemize}
Regarding the whole bit string $x$\newedit{,} let
\begin{itemize}
\item $L:=\{x \in \{0,1\}^n \mid 0 \leq \ones{x^j} \leq 6n/(5d) \text{ for all } j \in [d/2], \ones{x^i} < 6n/(5d) \text{ for an } i \in [d/2]\},$
\item $M:=\{x \in \{0,1\}^n \mid \ones{x^j} = 6n/(5d) \text{ for all } j \in [d/2] \text{ and } x^i \notin B \text{ for an } i \in [d/2]\},$
\item $N:=\{x \in \{0,1\}^n  \mid x^j \in A \cup B \text{ for all } j \in [d/2]\}$.
\end{itemize}

Then, the function class \mRRMO: $\{0,1\}^n \to \mathbb{N}_0^d,$ is defined as 
\[
d\text{-\RRRMO}(x) = (f_1(x), f_2(x), \ldots ,f_d(x))
\]
with 
\[
f_j(x) = g_j(x):= \begin{cases}
	\ones{x^{1+(j-1)/2}} \text{ if $j$ is odd,} \\
	\ones{x^{1+(j-2)/2}} \text{ if $j$ is even,}
\end{cases}
\]
if $x \in L$,
\[
f_j(x) = h_j(x) := g_j(x) + \begin{cases}
	\LZ(x^{1+(j-1)/2}) \text{ if $j$ is odd, } \\
	\TZ(x^{1+(j-2)/2}) \text{ if $j$ is even, }
\end{cases}
\]
if $x \in M$, 
\[
f_j(x) = 4n\lvert{K(x)}\rvert/(5d)+ h_j(x)
\]
if $x \in N$\newedit{,} where $K(x):=\{j \in [d/2] \mid x^j \in A\}$, and $f_j(x)=0$ otherwise. 
\end{definition}

In the $d$-objective Real-Royal-Road function, the bit string is divided into $d/2$ blocks, each of length $2n/d$, and the maximum in each objective is $2n/5+2n/d$. By Lemma~14 in~\cite{OPRIS2026} the maximum number of mutually incomparable solutions is at most $(4n/(5d)+1)^{d-1}$. Algorithms that initialize their population uniformly at random typically start with search points $x$ where for all $j \in [d/2]$ we have $0 < \ones{x^j} \leq 3/5 \cdot (2n/d) = 6n/(5d)$, meaning that $x \in L \cup M \cup N$. Then a fitness signal encourages increasing the number of ones in each $x^j$ to $6n/(5d)$, leading to search points $x \in M \cup N$. Next, these ones are collected within a cumulative block by increasing the sum of leading and trailing zeros in each block $j$, thereby achieving for all $j \in [d/2]$ that $x^j \in B$, or $x \in N$. Finally, when $x \in N$, a strong fitness signal is assigned \emph{equally} to each objective based on $\lvert{K(x)}\rvert$, which is the number of blocks $j \in [d/2]$ where $x^j \in A$. We are aiming for creating a search point $x$ with maximum $K(x)$ (i.e. $K(x) = [d/2]$). We will show how stochastic population update can help to increase the cardinality of $K(x)$ by starting with an individual $x$ with $x^j \in B$ and then repeatedly mutating $x$ in order to create an individual $y$ with $y^j \in A$. The Pareto set is reached when $K(x)=[d/2]$, and can then easily be covered through mutation. Theorem~\ref{thm:Runtime-Analysis-NSGA-III-mRRMOUpdate} below gives a result on the runtime for \nsgaIII with stochastic population update on $d$\text{-\RRRMO}. For the proof, we focus solely on the situation of creating a search point $y$ with $y^j \in A$ from a search point $x$ with $x^j \in B$ by repeatedly mutating $x$, and take the remaining results from~\citep{OPRIS2026} for exploring such a search point $x$ and covering the Pareto front as given. Furthermore, we state the runtime only in expectation, as this is sufficient for comparison with the significant worse runtime of at least $(1-o(1))n^{2n/(5d)-1}/\mu$ generations for the mutation only case for finding one Pareto optimal point, provided that the number of objectives $d$ is not too large. This poor runtime can be easily observed: with probability $1-o(1)$, all solutions $x \in P_0$ are initialized with at most $6n/(5d)$ ones in each block. Consequently, in each trial, one must flip $2n/(5d)$ bits within a single block to obtain one of $2n/(5d)$ possible blocks coinciding with a block of a Pareto optimal point. For further details, see Theorem~20 in~\citep{OPRIS2026}.  

\begin{theorem}
\label{thm:Runtime-Analysis-NSGA-III-mRRMOUpdate}
Let $d \in \mathbb{N}$ be divisible by $2$ and $n$ be divisible by $5d/2$. Then the algorithm \nsgaIII (Algorithm~\ref{alg:nsga-iii}) with $\varepsilon_{\text{nad}} \geq 2n/5+2n/d$, a set 
$\refer$ of reference points for $p \in \mathbb{N}$ with $p \geq 2d^{3/2}(2n/5+2n/d)$, and a population size $\mu \geq (4n/(5d)+1)^{d-1}$, finds the Pareto set of $f:=d\text{-\RRRMO}$ in expected $O(n^3 + n (12n)^{2n/(5d)} / (2n/(5d))!)$ generations which is  $O(n (12n)^{2n/(5d)} / (2n/(5d))!)$ for $d<2n/5$.
\end{theorem}

Before proving the theorem, we will take the following lemmas as given. The first lemma corresponds to Phases~1 to~4 in the proof of Theorem~15 in~\citep{OPRIS2026} and excels the exploration of a search point in $\mathcal{Q} := \{x \in \{0,1\}^n \mid x^j \in A \cup B \text{ for all } j \in [d/2]\}$. The second lemma addresses the coverage of the entire Pareto front and corresponds to Phase~$(d/2 + 5)$ in that proof. Note that these phases were originally formulated for a crossover probability $p_c \in (0,1)$. However, their execution relies solely on mutation steps, and the proof can be straightforwardly adapted to the case where the crossover probability is zero.

\begin{lemma}
\label{lem:Runtime-Analysis-NSGA-III-mRRMOFact1}
Assume the same conditions as in Theorem~\ref{thm:Runtime-Analysis-NSGA-III-mRRMOUpdate}. Then, the expected number of generations until \nsgaIII finds an individual $x \in \mathcal{Q}$ is $O(n^3)$.
\end{lemma}

\begin{lemma}
\label{lem:Runtime-Analysis-NSGA-III-mRRMOFact2}
Assume the same conditions as in Theorem~\ref{thm:Runtime-Analysis-NSGA-III-mRRMOUpdate} and suppose that there is already a Pareto optimal individual $x$, particularly $x \in \mathcal{Q}$ with $K(x)=[d/2]$. Then, the expected number of generations until \nsgaIII covers the whole Pareto front is $O(n^3)$.
\end{lemma}

\begin{proof}[Proof of Theorem~\ref{thm:Runtime-Analysis-NSGA-III-mRRMOUpdate}] 
With Lemma~\ref{lem:Runtime-Analysis-NSGA-III-mRRMOFact1}, the expected number of generations until \nsgaIII finds an individual $x \in \mathcal{Q}$ is $O(n^3)$. Suppose that there is such an $x$. Now we estimate the expected number of generations until an individual $x'$ is created with $K(x')>K(x)$. To this end, we fix a block $j \in [d/2]$ with $x^j \in B$ and consider $2n/(5d)$ successive generations, where in each generation we increase the number of ones by exactly one by flipping a single bit, while leaving all other bits in $x$ unchanged. Notably, all ones in block $j$ should always lie within a consecutive sub-block of length $8n/(5d)$. Then, after $2n/(5d)$ such generations, we have created such a desired $x'$. 
More formally, we call the $\ell$-th generation \emph{successful} if an individual $z$ is created with $6n/(5d)+ \ell$ ones in block $j$, which are all part of a sub-block of $8n/(5d)$ bits. Suppose that the $(\ell-1)$-th generation is successful. Then, generation $\ell$ is successful if, in a single trial, a parent $p$ is selected with $\ones{p^j} = 6n/(5d)+\ell-1$ while $z^i=x^i$ for $i \in [d/2] \setminus \{j\}$, a zero bit in the $j$-th block belonging to a sub-block of $8n/(5d)$ bits containing $6n/(5d) + \ell-1$ one bits is flipped while keeping the remaining bits unchanged. Finally, the new created individual is kept in the population. With probability at least 
\[
1-\left(1-\frac{2n/(5d)-\ell+1}{\mu e n}\right)^{\mu} \geq \frac{\frac{2n/(5d)-\ell+1}{en}}{1+\frac{2n/(5d)-\ell+1}{en}} \geq \frac{2n/(5d)-\ell+1}{4n}
\]
a desired individual $z$ is created in generation $\ell$. Note that $z$ is not Pareto optimal for $\ell<2n/(5d)$. However, $\mu-\lceil{3\mu/2}\rceil$ individuals chosen uniformly at random survive (see Line~\ref{line:update} and Line~\ref{line:survival} in Algorithm~\ref{alg:nsga-iii} for $a=1$) and hence even a non Pareto optimal individual remains with probability at least $(\mu - \lceil{3 \mu/2}\rceil)/\mu \geq 1/3$ for $n$ sufficiently large and hence with probability at least $(\frac{2n}{5d}-\ell+1)/(12n)$, generation $\ell$ is successful. If all $2n/(5d)$ generations are successful, the desired $y$ is created, which happens with probability at least 
$$\prod_{\ell=1}^{2n/(5d)} \frac{2n/(5d)-\ell+1}{12n} = \frac{(2n/(5d))!}{(12n)^{2n/(5d)}}.$$
Hence, the expected number of generations is $2n/(5d) \cdot (12n)^{2n/(5d)} / (2n/(5d))!$  until such a desired $y$ is created, since all generations $\ell \in [2n/(5d)]$ must be successful. Accounting for all $j \in [d/2]$, we see that in expected $n/5 \cdot (12n)^{2n/(5d)} / (2n/(5d))!$ generations a Pareto optimal individual is created. Also, considering the time until one creates an individual $x \in \mathcal{Q}$ the first time, and covering the Pareto front if there is $x \in \mathcal{Q}$ with $x^j \in A$ for all $j \in [d/2]$, the theorem is proven.  
\end{proof}

The speedup achieved by \nsgaIII with stochastic population update is $\Omega((2n/(5d))!/(n \mu 12^{2n/(5d)}))$ for $d<2n/5$ which is $n^{\Omega(n)}$ if $d \leq c \ln(n)$ for a sufficiently small constant $c > 0$, and $\mu = 2^{O(n)}$. Hence, stochastic population update provides an exponential speedup in the runtime for a broad parameter setting. 

\section{Lower Runtime Bounds for \nsgaIII on the OneJumpZeroJump Problem}

To complement the results on \mOJZJ from the previous section, we establish sharp lower bounds on the runtime of \nsgaIII without stochastic population update for the bi-objective and four-objective cases $d \in \{2,4\}$. These bounds are tight across a wide range of parameter settings. A central tool in our analysis is applying Lemma~\ref{lem:sparsity} to the \mOJZJ benchmark. This allows us to identify a sufficiently large subset $V$ of the Pareto front and guarantees an even distribution of solutions across $V$, as formalized in the following lemma.

\begin{lemma}{\label{lem:Upper-Bound-Cover-Number-Jump}}
\label{lem:sparsity-2}
    Consider \nsgaIII without stochastic population update optimizing $f:=d$-$\OJZJ_k$ with $2 \leq k \leq n/(2d)$, $\varepsilon_{\text{nad}} \geq 2n/d$ and a set $\mathcal{R}_p$ of reference points as defined above for $p \in \mathbb{N}$ with $p \geq 4n\sqrt{d}$, and a population size $(2n/d+1)^{d-1} \leq \mu \in 2^{O(n)}$. Suppose that all $P_0,P_1, \ldots $ ever seen by \nsgaIII consist only of Pareto optimal individuals. Then, with probability at least $1-e^{-\Omega(n)}$ the cover number of each Pareto optimal fitness vector is at most $\lceil{2^{d/2+1}\mu/|S_d|}\rceil$ after $O(d^2n)$ generations.
\end{lemma}

\begin{proof}
    Recall that $F_d^*$ denotes the Pareto front of $d$-$\OJZJ_k$. First, we determine a set $V$ of Pareto optimal vectors with cardinality at least $|F_d^*|/2 = (2n/d-2k+3)^{d/2}/2 \geq (n/d+3)^{d/2}/2 \geq (2n/d+6)^{d/2}/2^{d/2+1} \geq |S_d|/2^{d/2+1}$ and show that it will be covered in $O(n)$ generations with probability at least $1-e^{-\Omega(n)}$. Let $\beta:=\sqrt[d/2]{2}$ and let 
    $$V:=\Bigl\{v \in (\mathbb{N}_0)^d \mid v_i \in \Bigl[\frac{\frac{2n(\beta-1)}{d}+2k-3}{2\beta}-1+k,\frac{\frac{2n(\beta+1)}{d}-2k+3}{2\beta}+1+k \Bigl] \text{ for all }i \in [d] \Bigl\}.$$ 
    We obtain for $n$ sufficiently large owing to $k \leq n/(2d)$
    \begin{align*}
    B_2:&=\frac{\frac{2n(\beta+1)}{d}-2k+3}{2\beta} = \frac{n}{d}+\frac{1}{\beta}\left(\frac{n}{d} - k+\frac{3}{2}\right)\\
    &= \frac{2n}{d} - k +\left(1-\frac{1}{\beta}\right)\left(k-\frac{n}{d}\right) + \frac{3}{2\beta} < \frac{2n}{d} - k.
    \end{align*}
    Note that for 
    $$B_1:=\frac{\frac{2n(\beta-1)}{d}+2k-3}{2\beta}$$ 
    we have $B_1+B_2 = 2n/d$ and therefore $B_1 = 2n/d-B_2 \geq k$. 
    Hence, every $v \in V$ is Pareto optimal. Further
    \begin{align*}
    |V| & \geq \left(\frac{\frac{2n(\beta+1)}{d}-2k+3}{2\beta} - \frac{\frac{2n(\beta-1)}{d}+2k-3}{2\beta}\right)^{d/2}\\
    & \geq \left(\frac{\frac{2n}{d}-2k+3}{\beta}\right)^{d/2} = \frac{\left(\frac{2n}{d}-2k+3\right)^{d/2}}{2} = \frac{|F_d^*|}{2}
    \end{align*}
    and if an individual $x$ has fitness $f(x) \in V$ then $\ones{x}^j,\zeros{x}^j \in \{B_1-1, \ldots , B_2+1\}$ for all $j \in [d/2]$. Now we consider two phases. Phase~1 ends if the whole $V$ is covered and Phase~2 if the cover number of every $v\in |F_d^*|$ is bounded by $\lceil{2^{d/2+1} \mu/|S_d|}\rceil$ from above. We show that each phase is finished in $O(n)$ generations with probability at least $1-e^{-\Omega(n)}$. This concludes the proof by a union bound over both phases.\\ 
    \textbf{Phase 1:} Cover the whole $V$.\\
    By a classical Chernoff bound the probability is at least $1-e^{-\Omega(n^2)}$ that there is an individual $x$ initialized with $f_j(x) \in \{B_1, \ldots , B_2\}$ for all $j \in [d/2]$, i.e. $f(x) \in V$ (compare with Lemma~\ref{lem:initialization}). Suppose that this happens 
    and fix a covered $w \in V$. Let $v \in V$. We show with probability at least $1-e^{-\Omega(d^2n)}$ the vector $v$ is covered after $(\frac{2\beta ed}{\beta-1} + 1)n = O(d^2n)$ generations (where this asymptotical notion follows since one can write $\beta-1 = 2^{2/d}-1 = e^{2\ln(2)/d}-1 = \sum_{j=1}^\infty \frac{(2 \ln(2)/d)^j}{j !} = \frac{2 \ln(2)}{d} \sum_{j=0}^\infty \frac{(2 \ln(2)/d)^j}{(j + 1)!} = \Theta(1/d)$ due to $1 \leq \sum_{j=0}^\infty \frac{(2 \ln(2)/d)^j}{(j + 1)!} \leq e^{2 \ln(2)/d} \leq e^{\ln(2)} = 2$). Let $W_t:=\{x\in P_t \mid f(x) \in V\}$ and $e_t:=\min_{x \in W_t} \sum_{i=1}^{d/2}|f_{2i-1}(x)-v_{2i-1}|$. Note that $e_t=0$ if $v$ is covered, and $0 \leq e_t \leq d(B_2-B_1+2)/2$. By Lemma~\ref{lem:Reference-Points}, $e_t$ cannot increase, but it can be decreased in one trial by choosing $x \in P_t$ with $\sum_{j=1}^{d/2}|f_j(x)-v_j| = e_t$ as parent (prob. at least $1/\mu$) and flipping a one bit (zero bit) in block $i$ to zero (one) if $f_{2i-1}(x)-v_{2i-1}>0$ ($f_{2i-1}(x)-v_{2i-1}<0$) which happens with probability at least $B_1/n \cdot (1-1/n)^{n-1} \geq B_1/(en) \geq \frac{\beta-1}{\beta ed} \in \Omega(1/d^2)$. In one generation, this happens with probability at least 
    $$1-\left(1-\frac{\beta-1}{\beta ed \mu}\right)^{\mu} \geq \frac{(\beta-1)/(\beta ed)}{1+(\beta-1)/(\beta ed)} \geq \frac{\beta-1}{2 \beta e d} =: p.$$
    Let $\ell:=\lfloor{d(B_1+B_2+2)/2}\rfloor$. For $j \in [\ell]$ define the random variable $X_j$ as the number of generations with $j=e_t$. Then $X:=\sum_{j=1}^\ell X_j$ is stochastically dominated by the sum $Y:=\sum_{j=1}^\ell Y_j$ of geometrically distributed independent random variables $Y_j$ with success probability $p=\Omega(1/d^2)$. Note that $\expect{Y} \leq n/p$ (since $\ell \leq n$) and we obtain by Lemma~\ref{lem:Doerr-dominance} for $s:= \sum_{j=1}^{\ell} 1/p^2 = O(d^4n)$, and $\lambda \geq 0$
    \[
    \Pr(Y \geq \expect{Y} + \lambda) \leq \exp\left(-\frac{1}{4} \min\left\{\frac{\lambda^2}{s}, \lambda p \right\}\right).
    \]
    For $\lambda = 4n/p$ we obtain $\Pr(X \geq 5n/p) \leq \Pr(Y \geq n/p+4n/p) \leq e^{-n}$. By a union bound over all $v \in V$ we obtain that every $v \in V$ is covered in $5n/p = O(d^2n)$ generations with probability at least $1-e^{-\Omega(n)}$. \\
    \textbf{Phase 2:} The cover number of every $v \in |F_d^*|$ is at most $\lceil{\mu/|V|}\rceil \leq \lceil{2^{d/2+1}\mu/|S_d|}\rceil$.\\
    Now we can apply Lemma~\ref{lem:sparsity}(2), and obtain that the cover number of every $v \in |F_d^*|$ is at most $\lceil{\mu/|V|}\rceil \leq \lceil{2^{d/2+1}\mu/|S_d|}\rceil$ with probability at least $1-e^{-\Omega(n)}$ after $O(n)$ generations.
    
    Combining both phases yields the result by a union bound. 
    \end{proof}

Observe that when $d>2$, it may happen that non-Pareto optimal search points with rank one are created and survive. For example, consider a population located on the Pareto front that does not contain $0^{2n/d}$ or $1^{2n/d}$ in any block $j \in [d/2]$. If $\mu-1$ individuals are cloned and one non-Pareto-optimal individual $y$ is generated such that $y^i = 1^{2n/d}$ for some block $i \in [d/2]$, while $1 \le \ones{y^j} < k$ holds for another block $j$, then $y$ may be first-ranked and hence, has a chance to remain in the population. Therefore, Lemma~\ref{lem:even-cover}(3) cannot be applied from scratch, complicating the analysis. Hence, we present only the cases for $d \in \{2,4\}$ and leave the cases for more than four objectives for future work. The reason is that much deeper insights about the population dynamics of \nsgaIII are needed here, particularly how solutions below the Pareto front behave, since the dynamics between the single objectives become significantly more complex for $d>4$. But at least for the case $d=4$, no solution $y$ with $1 \leq \ones{y^j} \leq k-1$ or $n/2-k+1 \leq \ones{y^j} \leq n-1$ for both $j \in \{1,2\}$ ever exist in the population with high probability. Hence, all solutions in the population have Hamming distance at least $k$ from $1^n$. We will show below that this fact also hinders the creation of solutions of the form $1^{n/2}y^1$ or $y^21^{n/2}$ for $y^1,y^2 \in \{0,1\}^{n/2}$ with $1 \leq \ones{y^j} \leq k-1$ or $n/2-k+1 \leq \ones{y^j} \leq n-1$. Then at least $k$ specific bits still need to be flipped at once to obtain $1^n$. All these considerations finally provide a tight runtime bound for the case $d=4$ and $k \geq 4$ for a large parameter regime of the population size $\mu$. We begin with the case $d=2$, where all individuals ever seen in $P_t$ are Pareto optimal after a successful initialization.
    
\begin{theorem}
\label{thm:lower-bound-2}
    Consider \nsgaIII without stochastic population update optimizing $f:=d$-$\OJZJ_k$ with $d=2$, $2 \leq k \leq n/(2d)$, $\varepsilon_{\text{nad}} \geq 2n/d$ and a set $\mathcal{R}_p$ of reference points as defined above for $p \in \mathbb{N}$ with $p \geq 4n\sqrt{d}$, and a population size $(2n/d+1)^{d-1} \leq \mu \in O(n^{k-1}) \cap O(\text{poly}(n))$. Then the expected number of generations required for covering the entire Pareto front is at least $\Omega(n^{k+1}/\mu)$. 
\end{theorem}

\begin{proof}
    By a classical Chernoff bound, with probability $1-e^{-\Omega(n)}$, every individual $x$ fulfills $\ones{x} \in \{k,\ldots , n-k\}$ after initialization. Suppose that this happens. Then there are only Pareto optimal individuals in $P_0$ and all future populations $P_1,P_2, \ldots$ since a non-Pareto optimal individual is dominated by every Pareto optimal one (see Lemma~\ref{lem:Jump-d=2}). Since there are $\mu$ Pareto optimal individuals in $R_t$, such a $y$ never survives (due to the non-dominated sorting procedure). 
    Let $n$ be sufficiently large such that $\mu \leq c n^{k-1}$ for a constant $c>0$ and the event from Lemma~\ref{lem:sparsity-2} applies on at most $\delta n$ generations for a further constant $\delta>0$, meaning that with probability at least $1-e^{-\Omega(n)}$ the cover number of every $v \in |F_2^*|$ is at most $\lceil{4 \mu/|S_2|}\rceil$ after at most $O(n)$ generations. Suppose that this happens, and further, assume that $1/n+5n/|F_2^*| \leq \alpha$ for a suitable constant $\alpha>0$ (due to $|F_2^*| = 2n/d-2k+3 \in \Theta(n)$). We consider a phase of at most $\delta n$ generations (i.e. $\mu \delta n$ trials) and show that no individual $y \in \{0^n,1^n\}$ is generated within $\mu \delta n$ trials with at least constant probability. Since in one trial one has to flip at least $k$ specific one bits or $k$ specific zero bits to create $y$, the probability that this happens in $\mu \delta n$ trials is at most $1-(1-2/n^k)^{\mu \delta n} \leq 1-(1-2/n^k)^{c \delta n^k} \leq 1-(1/16)^{c \delta} =:b$ where we used $(1-2/n)^n \geq \left((1-1/n)^n\right)^2 \geq 1/16$. Note that $b<1$ is a constant. Hence, by a union bound we obtain with probability at least $1-b-e^{-\Omega(n)} = \Omega(1)$ that there is a generation $t$ where every $v\in |F_2^*|$ has cover number at most $\lceil{4\mu/|F_2^*|}\rceil$, and every individual $x \in P_t$ fulfills $x \notin \{1^n,0^n\}$. Suppose that this happens. We estimate the expected time to create $x=1^n$ (which is the only Pareto optimal search point with fitness $f_1(x)=2n/d+k$) from below. 
    Note that for a given $k \leq \ell \leq 2k-1$ there are at most $\lceil{4\mu/|F_2^*|}\rceil \leq 5\mu/|F_2^*|$ different individuals $y \in P_t$ with $\ell$ zeros. Note also it requires to flip at least $2k$ specific bits to create $1^n$ from a $y$ with at least $2k$ zeros. Since $k \geq 2$, all these considerations yield a probability of at most
    \begin{align*}
    & \left(\frac{1}{n}\right)^{2k} + \sum_{j=0}^{k-1} \frac{5}{|F_2^*|} \cdot \left(\frac{1}{n}\right)^{j+k} \cdot \left(1-\frac{1}{n}\right)^{n-j-k} \\
    & \leq \left(\frac{1}{n}\right)^{2k} + \frac{5}{|F_2^*|n^k} \leq \left(\frac{1}{n}\right)^{k+1}\left(\frac{1}{n}+\frac{5n}{|F_2^*|}\right) \leq \frac{\alpha}{n^{k+1}}
    \end{align*}
    to create $1^n$ in one trial. Hence, in one generation we obtain by a union bound over $\mu$ trials that $1^n$ is created with probability at most $\alpha \mu (1/n)^{k+1}$. Hence, the expected number of generations to create $1^n$ is at least $(1-b-e^{-\Omega(n)})n^{k+1}/(\alpha \mu) = \Omega(n^{k+1}/\mu)$
    which proves the theorem.
\end{proof}

When combining Theorems~\ref{thm:Runtime-Analysis-NSGA-III-mOJZJ} and ~\ref{thm:lower-bound-2} we obtain for $|S_2|=|F_2^*| \leq \mu \in O(n^{k-1}) \cap O(\text{poly}(n))$, $d=2$ and $2 \leq k \leq n/4$ the tight runtime of $\Theta(n^{k+1}/\mu)$ for \nsgaIII on $2$-\OJZJ in terms of generations. Our result corresponds to the lower bound of $\Omega(\mu n^k)$ for $\mu = o(n^2)$ established in~\citep{DoerrQu2023a} for the \nsga in the case $d = 2$, while using a population size that is linear in $n$. For larger population sizes, our bound becomes $o(\mu n^k)$, and hence \nsgaIII outperforms \nsga. Finally, we also consider the case $d = 4$, for which we derive similar tight runtime bounds. We omit the cases $k = 2$ and $k = 3$, however, as they would further complicate the analysis, since they require a more intricate treatment of the underlying population dynamics.

\begin{theorem}
\label{thm:lower-bound-4}
    Consider \nsgaIII without stochastic population update optimizing $f:=d$-$\OJZJ_k$ with $d=4$, $4 \leq k \leq n/(2d)$, $\varepsilon_{\text{nad}} \geq 2n/d$ and a set $\mathcal{R}_p$ of reference points as defined above for $p \in \mathbb{N}$ with $p \geq 4n\sqrt{d}$, and a population size $(2n/d+1)^{d-1} \leq \mu \in O(n^{k-2}) \cap O(\text{poly}(n))$. Then the expected number of generations required for covering the whole Pareto front is at least $\Omega(n^{k+2}/\mu)$. 
\end{theorem}

\begin{proof}
By a classical Chernoff bound, with probability $1-e^{-\Omega(n)}=1-o(1)$, each individual $x$ satisfies $\ones{x^j} \in \{k,\ldots , n/2-k\}$ for both $j \in \{1,2\}$ after initialization. Suppose that this happens. Then in all future populations $P_1,P_2, \ldots$ there are only individuals $x$ with $\ones{x^j} \in \{0,k, \ldots , n/2-k,n/2\}$ for at least one $j \in \{1,2\}$, since other individuals are dominated by every $y$ with $\ones{y^j} \in \{k \ldots , n/2-k\}$ for both $j \in \{1,2\}$. 

\begin{lemma}
	\label{lem:not-jumping}
	Let $D_t:=\{s \in \{0,1,\bot\}^2 \mid \text{ there is }x \in P_t \text{ with } L(f(x))=s\}$ where $F_4^*$ denotes the Pareto front of $4$-$\OJZJ$ and assume that $0^2,1^2,01,10 \notin D_t$. With probability $1-o(1)$, after $cn \ln(n)$ generations for some constant $c>0$, every $x \in P_t$ is Pareto optimal and for every $v \in F_4^*$ with $L(v) \in D_t$ we have $1 \leq c_t(v) \leq \lfloor{8 \mu/|S_d|}\rfloor$, or $D_t$ enlarges. If $D_t$ enlarges this way, then with probability at least $1/4-o(1)$, still $0^2,1^2,01,10 \notin D_t$.
\end{lemma}

\begin{proof}
	Note that the probability to flip $k$ specific bits in a block is $O(1/n^k)$ which happens then with probability $\mu d n \ln(n) = o(1)$ in $d n \ln(n)$ generations for every constant $d>0$, since $d n \ln(n)$ generations consist of at most $\mu d n \ln(n) = o(n^k)$ trials. So we can assume that we never create $1^{n/2}$ or $0^{n/2}$ in one block by flipping $k$ specific bits. Note that $P_0$ is Pareto optimal, where $L(f(P_0)) = \{\bot^2\}$.
	
	At first we look at the case when $D_t$ enlarges. To create an individual $x$ with $L(f(x)) \in \{1\bot,\bot1,0\bot,\bot0\}$ the first time in an iteration $t$ (where there may be non-Pareto optimal individuals), one must mutate an individual $z \in P_t$ that satisfies $|\ones{z^j}-\ones{x^j}| < k$ and $\ones{z^{3-j}} \in \{k,\ldots,n/2-k\}$ for a block $j \in \{1,2\}$. Suppose that $x^j \in \{0^{n/2},1^{n/2}\}$ is created. To achieve also $\ones{x^{3-j}} \notin \{k,\ldots,n/2-k\}$, it is necessary to flip at least $\zeros{z}-k+1>0$ zero bits or at least $\ones{z}-k+1>0$ one bits in block $3-j$. This probability can be estimated by $1-(1-1/n)^n \leq 3/4$ from above, since it requires to flip any bit. Hence, under the condition that a search point $x$ with $\ones{x^j} \in \{0,n/2\}$ is created in one trial, the probability that $\ones{x^{3-j}} \in \{0, \ldots , k-1, n/2-k+1, \ldots , n/2\}=:M$ is at most $3/4$. If this does not happen, $x$ dominates every individual $z^*$ with $\ones{(z^*)^j}=\ones{x^j}$ and $(z^*)^{3-j} \in M \setminus \{0,n/2\}$, and such a $z^*$ can never survive, also not in future generations. Hence, with probability at least $1/4-o(1)$, $D_t$ enlarges, while not having any search point of the form $1^{n/2}z_1, 0^{n/2}z_2, z_31^{n/2}, z_40^{n/2} \in P_t$ where $\ones{z_i} \in M \setminus \{0,n/2\}$ for all $i \in [4]$. Hence, $D_t$ can be only enlarged again by choosing such a $z \in P_t$ above as a parent (since the choice of the other parent would require to flip $k$ specific bits in one block which we excluded), but then still with probability $1/4-o(1)$ we have $1^{n/2}z_1, 0^{n/2}z_2, z_31^{n/2}, z_40^{n/2} \notin P_t$ when $D_t$ enlarges, where $\ones{z_i} \in M$ for all $i \in [4]$. Note that this process may be repeated until $D_t = \{\bot^2,0\bot,1\bot,\bot0,\bot1\}$. 
	
	Now, when assuming that $D_t$ does not enlarge during this period, Lemma~\ref{lem:interior-covering} (for fixed $t$) implies that all Pareto-optimal vectors $v$ with $L(v)=s$ for $s \in D_t$ are covered within $c n \ln n$ generations for a suitable constant $c>0$, with probability $1 - o(1)$. Moreover, within the same time, $\mu$ Pareto-optimal individuals are created and $1 \leq c_t(v) \leq \lfloor{8\mu/|S_d|}\rfloor$ (compare with the proof of Lemma~\ref{lem:sparsity}(1), where any Pareto-optimal individuals needs to be replicated at most $\mu$ times, and compare also with Lemma~\ref{lem:sparsity-2}). This proves the lemma. 
\end{proof}

Note that, in the next $n^{k+2}$ trials, the probability to flip $2k$ specific bits is at most $o(1)$ since $k \geq 4$. So suppose that this does not happen. 
By Lemma~\ref{lem:not-jumping}, after $O(n \ln(n))$ generations, with probability at least $1/4-o(1)$ either $D_t$ is enlarged, and still $0^2,1^2,01,10 \notin D_t$, or for every Pareto optimal $v$ with $L(v) \in D_t$, there exists an individual $x \in P_t$ with $f(x) = v$, and all individuals in $P_t$ are Pareto optimal. In the first case, we obtain an individual of the form $1^{n/2} z_1$, $z_2 1^{n/2}$, $0^{n/2} z_3$ or $z_4 0^{n/2}$ where $z_1, \ldots , z_4 \in \{0,1\}^{n/2}$ with $\ones{z_j} \in \{k, \ldots , 2n/d-k\}$ before $1^n$ or $0^n$.
But this applies also in the second case, since when generating an individual $x$ with $x^j \in \{0^{n/2},1^{n/2}\}$ the first time by flipping $k$ bits in block $j$, changing no bits in block $3-j$ happens with probability at least $1/4$. Suppose that this happens. By symmetry, suppose without loss of generality that $1^{n/2}z_1 \in P_t$ for a $z_1 \in \{0,1\}^{n/2}$ with $k \leq \ones{z_1} \leq n/2-k$, but there is no $y \in P_t$ with $y=1^{n/2}z^*$ with $\ones{z^*} \notin \{k, \ldots , n/2-k\}$. Then $D_t$ enlarges again, or by Lemmas~\ref{lem:not-jumping},~\ref{lem:sparsity}(2) and~\ref{lem:sparsity-2}, after $O(n \ln(n))$ generations the following happens with probability $1-o(1)$: for every $v \in F_4^*$ with $L(v) \in D_t$ there exists an individual $x \in P_t$ with $f(x) =v$, and the cover number of each such $v$ lies between $\lfloor{\mu/|S_4|}\rfloor$, and $\lceil{8\mu/|S_4|}\rceil$. Moreover, all individuals are Pareto optimal. In case that $D_t$ enlarges, still $1^2,10,01,0^2 \notin D_t$ with probability $1/4-o(1)$, and we repeat these arguments. Hence, with probability at least $1/4^q - o(1)$, where $q \in \{0\} \cup [4]$ denotes the number of enlargements of $D_t$, this situation occurs, and still $1^2, 10, 01, 0^2 \notin D_t$.

Suppose that this happens. Now we estimate the probability of creating an individual $x$ of the form $z^* 1^{n/2}$ before $1^n$ for $\ones{z^*} \in \{k, \ldots , n/2-k\}$. If such an individual is already in $P_t$, then there is nothing to show. So suppose not. To create such an individual, one has to flip $2k$ specific bits if one chooses a $y$ as parent with $\ones{y^1} \neq 1^{n/2}$ which we excluded. So suppose one chooses an individual $y$ with $\ones{y^2} = \ell \in \{k, \ldots , n/2-k\}$ as parent and creates $1^{n/2}$ in the second block. The probability is at most $\lceil{8\mu/|S_4|}\rceil$ to choose a $y$ with $y^1 = 1^{n/2}$ and $\ones{y^2}=\ell$. The probability is at least $(2n/d-2k) \lfloor{\mu/|S_4|}\rfloor \geq (2n/d-n/d) \lfloor{\mu/|S_4|}\rfloor \geq n/4 \cdot \lfloor{\mu/|S_4|}\rfloor$ to choose a $y$ with $y^1 \neq 1^{n/2}$ and $\ones{y^2}=\ell$ and finally, the probability is at least $1/4$ not to change the first block. Hence, the probability to create $z^*1^{n/2}$ for a $z^* \in \{0,1\}^{n/2}$ with $\ones{z^*} \in \{k, \ldots , n/2-k\}$ before $1^n$ or a non-Pareto optimal individual in that way is at least 
$$\frac{1}{4} \cdot \frac{n/4 \cdot \lfloor{\mu/|S_4|}\rfloor}{n/4 \cdot \lfloor{\mu/|S_4|}\rfloor + \lceil{8\mu/|S_4|}\rceil} = \frac{1}{4}-\frac{\lceil{8\mu/|S_4|}\rceil/4}{n/4 \cdot \lfloor{\mu/|S_4|}\rfloor + \lceil{8\mu/|S_4|}\rceil} \geq \frac{1}{4}-\frac{1}{n/16 + 4} = 1/4- o(1).
$$
due to $\lfloor{a}\rfloor/\lceil{8a}\rceil \geq 1/16$ for $a \geq 1$. Thus, by the law of total probability on all possible $\ell \in \{k, \ldots , n/2-k\}$, the probability to create $z^*1^{n/2}$ for a $z^* \in \{0,1\}^{n/2}$ with $\ones{z^*} \in \{k, \ldots , n/2-k\}$ before $1^n$ or a non-Pareto optimal individual in that way is at least $1/4-o(1)$. Suppose that this happens. Now the situation is that $1\bot \in D_t$, $\bot 1 \in D_t$, but $1^2 \notin D_t$, and further, every individual $x$ with $x^j = 1^{n/2}$ for a $j \in \{1,2\}$ is Pareto optimal in $P_t$ and all future iterations. This situation occurs with probability at least $1/256 - o(1)$. Call a generation \emph{bad} if there is a non-Pareto optimal individual, or the cover number of a vector on the Pareto front is larger than $\lceil{8 \mu/|F_4^*|}\rceil$, or there is a $v \in F_4^*$ with $L(v) \in D_t$ which is still uncovered. Then, there are at most $O(n \ln(n))$ bad generations with probability at least $1/256-o(1)$ (since $D_t$ can enlarge only at most four times this way) and if this happens, $1^n$ is not generated within a bad generation with probability at least $1-o(1)$, since this requires to flip $k$ bits. So we consider generations which are not bad. Now, to create $1^n$, one has to either flip $2k$ specific bits, or choose a parent $x$ with $x=1^{n/2}z^*$ or $x=z^*1^{n/2}$ for a $z^* \in \{0,1\}^{n/2}$ with $k \leq \ones{z^*} \leq n/2-k$, and flip $\zeros{z^*}$ specific bits in $x$. For $\ell \in \{k, \ldots , n/2-k\}$ there are at most $2\lceil{8\mu/|F_4^*|}\rceil \leq 17\mu/|F_4^*|$ such different individuals $x \in P_t$ with $\ell$ zeros. Note also it requires to flip at least $2k$ specific bits to create $1^n$ from a $y$ with $H(1^n,y) \geq 2k$. Since $k \geq 4$, all these considerations yield a probability of at most
\begin{align*}
& \left(\frac{1}{n}\right)^{2k} + \sum_{j=0}^{k-1} \frac{17}{|F_4^*|} \cdot \left(\frac{1}{n}\right)^{j+k} \cdot \left(1-\frac{1}{n}\right)^{n-j-k} \\
& \leq \left(\frac{1}{n}\right)^{2k} + \frac{17}{|F_4^*|n^k} \leq \left(\frac{1}{n}\right)^{k+1}\left(\frac{1}{n}+\frac{17n}{|F_4^*|}\right) \leq \frac{2}{n^{k+2}}
\end{align*}
to create $1^n$ in one trial for $n$ sufficiently large. Hence, in one generation we obtain by a union bound over $\mu$ trials that $1^n$ is created with probability at most $2 \mu (1/n)^{k+2} = o(1)$. Hence, the expected number of generations to create $1^n$ is at least $(1/256 - o(1))n^{k+2}/(2\mu) = \Omega(n^{k+2}/\mu)$ which proves the theorem. 
\end{proof}

When combining Theorems~\ref{thm:Runtime-Analysis-NSGA-III-mOJZJ} and ~\ref{thm:lower-bound-4} we obtain for $|S_2|=|F_2^*| \leq \mu \in O(n^{k-2}) \cap O(\text{poly}(n))$, $d=4$ and $4 \leq k \leq n/4$ the tight runtime of $\Theta(n^{k+1}/\mu)$ for \nsgaIII on $4$-\OJZJ in terms of generations.

\section{Conclusions}
In this paper we provided new insights in the population dynamics of \nsgaIII, demonstrating that it rapidly spreads its solutions evenly across the entire search space. Our methods were developed with a level of generality that makes them applicable to a broad range of problems. Using these techniques, we derived upper runtime bounds for \nsgaIII on classical benchmark functions such as \mOMM, \mLOTZ, and \mCOCZ, as well as on the multimodal benchmark \mOJZJ, for any number $d$ of objectives, where we also considered the well-spread of solutions. In all cases and for certain regimes for the population size that are asymptotically larger than the maximum number of mutually incomparable solutions, the derived upper bounds are tighter than those obtained for NSGA-II in~\citep{ZhengD2023}. For \mOMM and \mOJZJ, \nsgaIII even outperforms NSGA-II. This indicates that \nsgaIII is relatively robust for with respect to the choice of the population size, which may be highly beneficial in situations where the underlying problem is not known well.

Additionally, we showed that stochastic population update, where solutions for survival are not always selected deterministically, can lead to an exponential speedup in runtime on \mOJZJ and even on another multimodal function \mRRMO. To complement this analysis, we also established even tight runtime bounds in the former case when the number of objectives is two or four for certain regimes of gap size $k$ and the population size $\mu$. 

We hope that the techniques developed in this work represent a significant step forward in understanding the behavior of \nsgaIII, while also shedding light on its strengths and limitations. To deepen this understanding, future work could focus on deriving rigorous lower bounds for \nsgaIII on additional classical benchmark functions such as \mOMM, \mCOCZ, and \mLOTZ, as well as on analyzing more complex combinatorial optimization problems, for example the many-objective minimum spanning tree problem or many-objective subset selection problems. Furthermore, we hope that the deeper theoretical understanding of the dynamics of \nsgaIII across this variety of problems will provide useful insights for practitioners, for instance by facilitating the development of refined algorithmic variants aimed at improving performance in real-world scenarios where many local optima occur.

\section*{Acknowledgements}
This work benefited from fruitful discussions at Dagstuhl Seminar 24271 ''Theory of Randomized Optimization Heuristics''. We also thank Benjamin Doerr for insightful discussions regarding our analyses. 

\textbf{Declaration of generative AI and AI-assisted technologies in the writing process}

During the preparation of this work the author used Chatgpt 4.0 in order to improve language and readability. After using this tool, the author reviewed and edited the content as needed and takes his full responsibility for the content of the publication.

\bibliographystyle{abbrvnat}
\bibliography{journal}
\end{document}